\documentclass{article}
\usepackage{microtype}
\usepackage{graphicx}
\usepackage{subfigure}
\usepackage{booktabs} %
\usepackage[ruled, vlined]{algorithm2e} %
\usepackage{hyperref}

\usepackage[accepted]{icml2025}

\usepackage{amsmath}
\usepackage{amssymb}
\usepackage{mathtools}
\usepackage{amsthm}

    \def\colorful{1}
    
    \ifnum\colorful=1

    \newcommand{\snote}[1]{\footnote{{\bf [Sushrut: {#1}\bf ] }}}
    \else

    \newcommand{\snote}[1]{}
    \fi
    
\newcommand{\alb}{\overline{\alpha}} %

\newcommand{\stept}{t} %
\newcommand{\Tstep}{T}

\newcommand{\E}{\mathbb{E}}

\newcommand{\B}{\mathbf}
\newcommand{\Var}{\mathrm{Var}}

\def\R{\mathbb R}
\def\C{\mathbb C}

\newcommand{\cL}{\mathcal{L}}

\newcommand{\cN}{\mathcal{N}}

\newcommand{\bA}{\mathbf{A}}
\newcommand{\bB}{\mathbf{B}}
\newcommand{\bC}{\mathbf{C}}

\newcommand{\bF}{\mathbf{F}}

\newcommand{\bI}{\mathbf{I}}

\newcommand{\bR}{\mathbf{R}}

\newcommand{\bx}{\mathbf{x}}

\newcommand{\by}{\mathbf{y}}
\newcommand{\bz}{\mathbf{z}}

\newcommand{\bs}{\mathbf{s}}

\newcommand{\bv}{\mathbf{v}}

\usepackage{graphicx} %
\usepackage{amsmath,amssymb}
\usepackage{xcolor}
\usepackage[utf8]{inputenc}
\usepackage{subcaption} %
\usepackage{booktabs}
\usepackage{multirow}
\usepackage{animate}  
\usepackage{subcaption}  
\usepackage{graphicx}  

\newcommand{\var}{\mathrm{Var}}
\newcommand{\im}{\text{Im}}
\newcommand{\re}{\text{Re}}
\newcommand{\snr}{\mathsf{SNR}}
\newcommand{\Cov}{\mathsf{Cov}}
\newcommand{\beps}{\boldsymbol{\epsilon}}
\usepackage[capitalize,noabbrev]{cleveref}
\crefname{appendix}{App.}{Apps.} %
\crefname{section}{\S}{\S} %
\crefname{figure}{Fig.}{Figs.}

\theoremstyle{plain}

\newtheorem{proposition}{Proposition}
\newtheorem{lemma}{Lemma}

\newtheorem{fact}{Fact}
\theoremstyle{definition}
\newtheorem{definition}{Definition}

\theoremstyle{remark}

\newcommand{\ff}[1]{\textbf{\textcolor{blue}{}}}   %

\usepackage[textsize=tiny]{todonotes}
\usepackage{wrapfig}
\usepackage{enumitem}
\usepackage{rotating} %

\icmltitlerunning{A Fourier Space Perspective on Diffusion Models}

\begin{document}

\twocolumn[
\icmltitle{A Fourier Space Perspective on Diffusion Models}

\icmlsetsymbol{equal}{*}

\begin{icmlauthorlist}
\icmlauthor{Fabian Falck}{msr}
\icmlauthor{Teodora Pandeva}{msr}
\icmlauthor{Kiarash Zahirnia}{msr}
\icmlauthor{Rachel Lawrence}{msr} \\
\icmlauthor{Richard Turner}{msr}
\icmlauthor{Edward Meeds}{msr}
\icmlauthor{Javier Zazo}{msr}
\icmlauthor{Sushrut Karmalkar}{msr}
\end{icmlauthorlist}

\icmlaffiliation{msr}{Microsoft Research}

\icmlcorrespondingauthor{Fabian Falck}{fabian.falck@microsoft.com}
\icmlcorrespondingauthor{Sushrut Karmalkar}{skarmalkar@microsoft.com}

\icmlkeywords{Machine Learning, ICML}

\vskip 0.3in
]

\begin{abstract}

Diffusion models are state-of-the-art generative models on data modalities such as images, audio, proteins and materials. These modalities share the property of exponentially decaying variance and magnitude in the Fourier domain. Under the standard Denoising Diffusion Probabilistic Models (DDPM) forward process of additive white noise, this property results in high-frequency components being corrupted faster and earlier in terms of their Signal-to-Noise Ratio (SNR) than low-frequency ones. The reverse process then generates low-frequency information before high-frequency details. In this work, we study the inductive bias of the forward process of diffusion models in Fourier space. We theoretically analyse and empirically demonstrate that the faster noising of high-frequency components in DDPM results in violations of the normality assumption in the reverse process. Our experiments show that this leads to degraded generation quality of high-frequency components. We then study an alternate forward process in Fourier space which corrupts all frequencies at the same rate, removing the typical frequency hierarchy during generation, and demonstrate marked performance improvements on datasets where high frequencies are primary, while performing on par with DDPM on standard imaging benchmarks.

\end{abstract}

\printAffiliationsAndNotice{} 

\section{Introduction}
\label{sec:Introduction}

\begin{figure*}[t]
\centering
\begin{minipage}[b]{0.38\textwidth}
\begin{minipage}[b]{0.48\textwidth}
    \centering
    \includegraphics[width=1\textwidth]{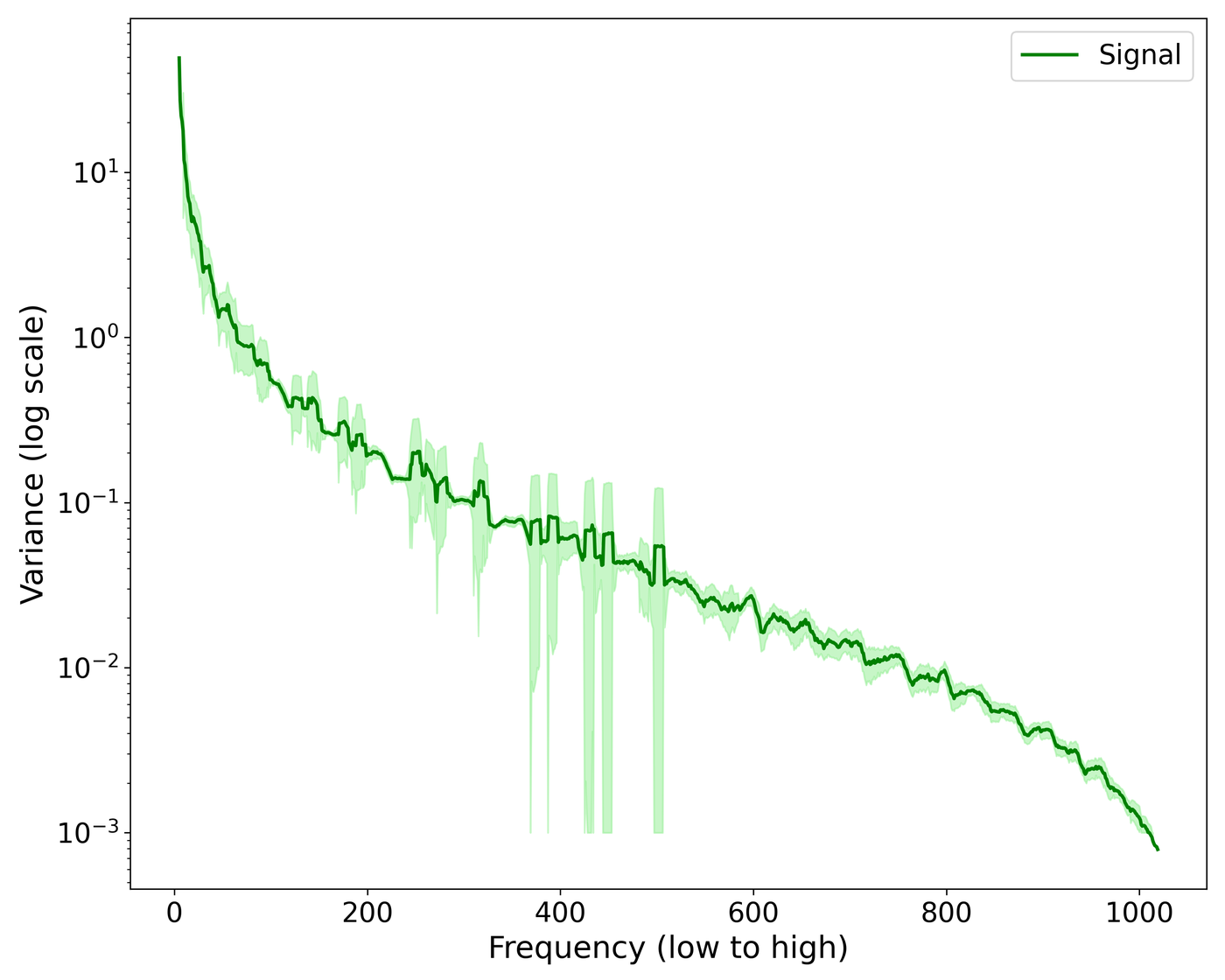}
\end{minipage}
\hfill
\begin{minipage}[b]{0.48\textwidth}
    \centering
    \includegraphics[width=1\textwidth]{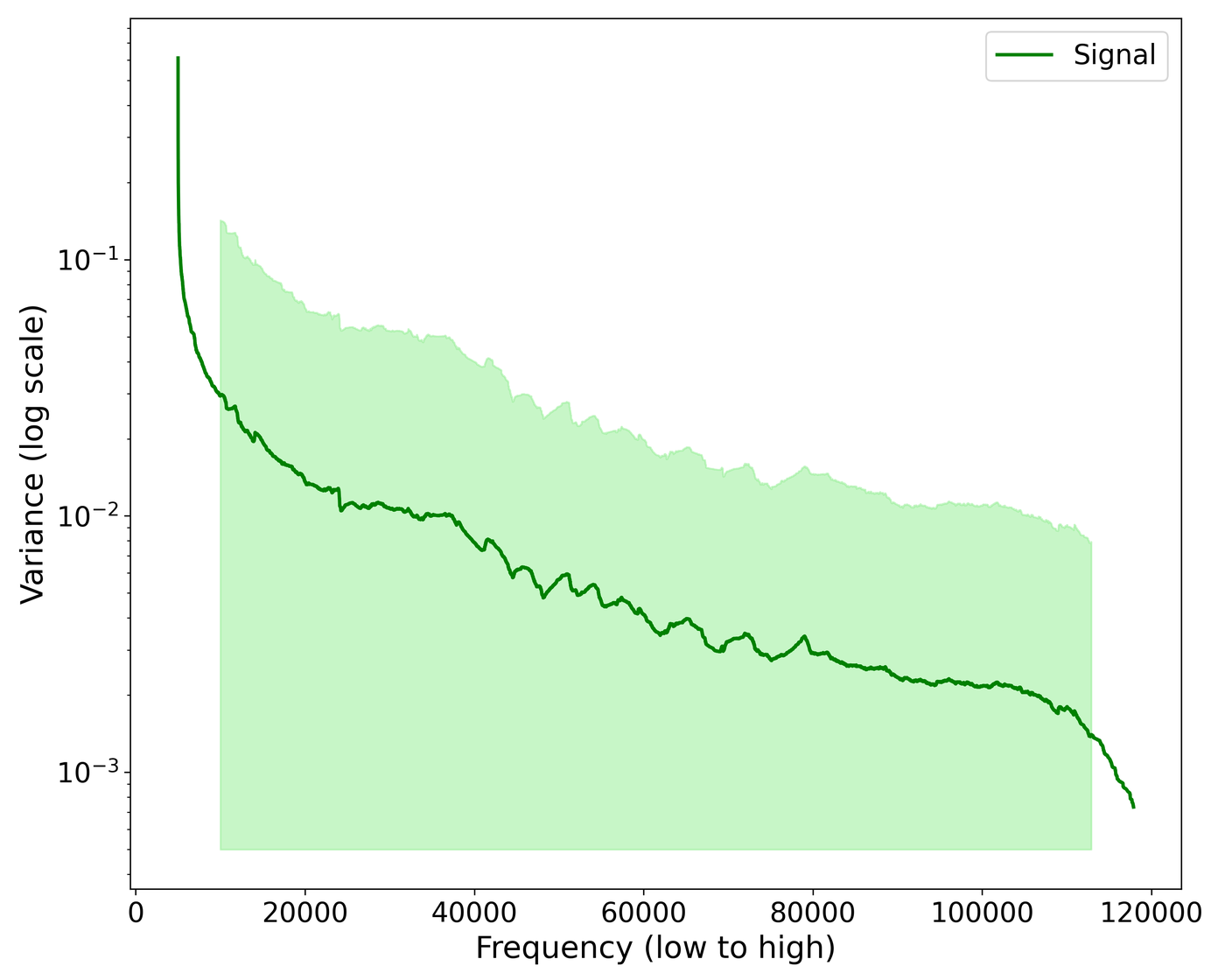}
\end{minipage}
\hfill
\begin{minipage}[b]{0.48\textwidth}
    \centering
    \includegraphics[width=1\textwidth]{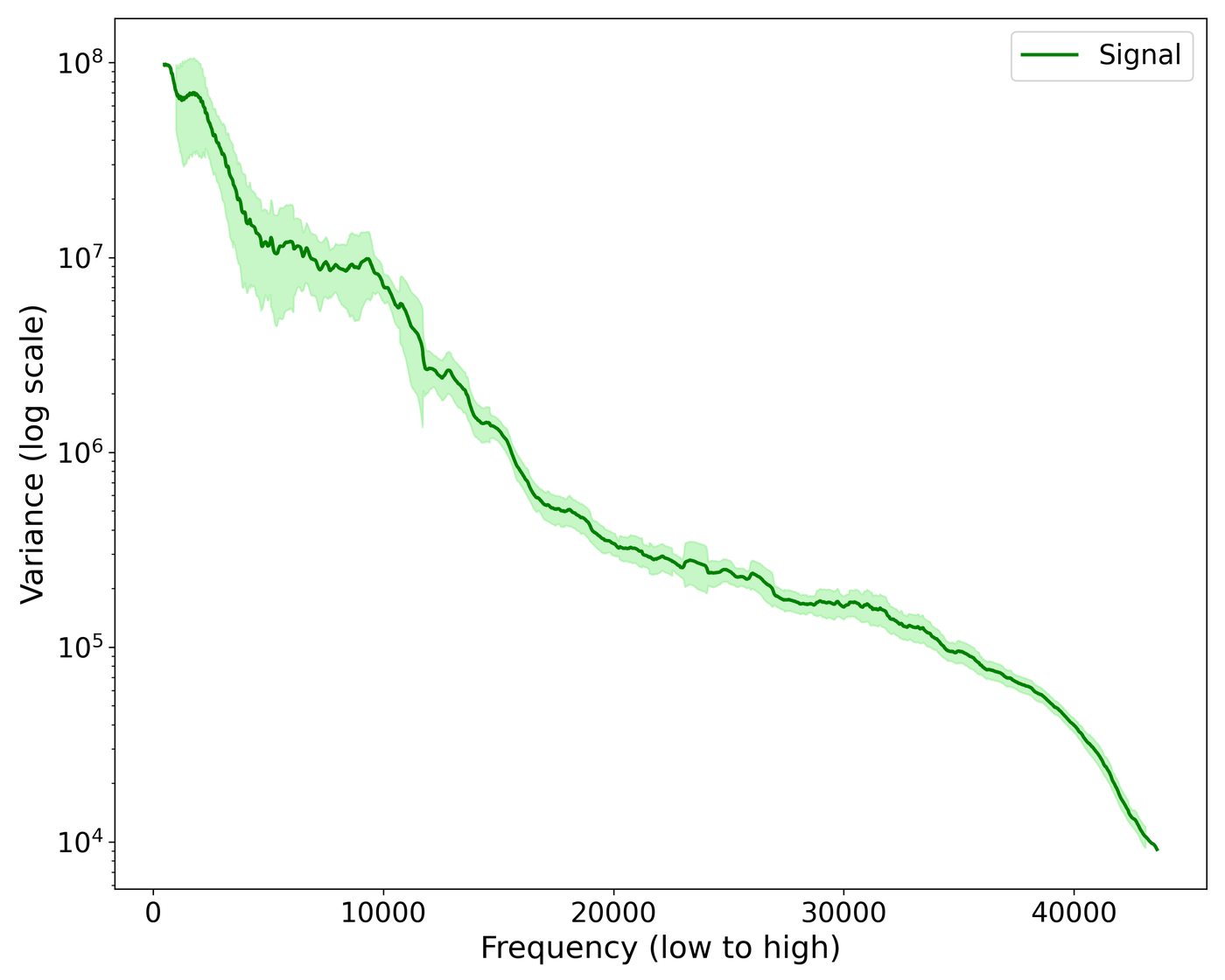}
\end{minipage}
\hfill
\begin{minipage}[b]{0.48\textwidth}
    \centering
    \includegraphics[width=1\textwidth]{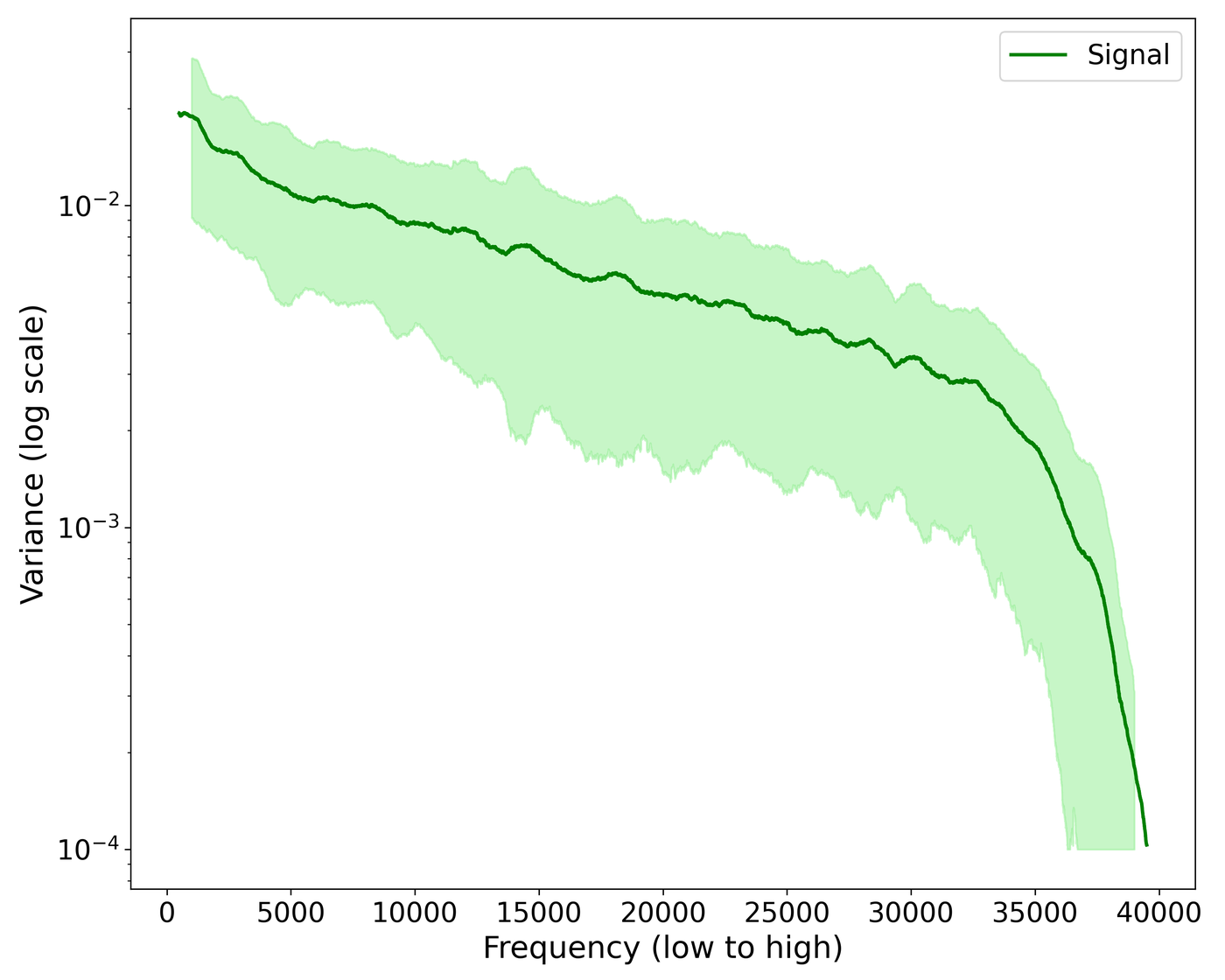}
\end{minipage}
\end{minipage}
\animategraphics[loop, autoplay, width=.3\textwidth]{2}{Figures/fig1_modality-images_cifar10_p_False/fig-}{01}{21}  
\animategraphics[loop, autoplay, width=.3\textwidth]{2}{Figures/fig1_modality-images_cifar10_K_False/fig-}{01}{21}
\caption{
[Left] \textit{The Fourier power law} observed in (top-left) images \cite{krizhevsky2009learning}, (top-right) videos \cite{kay2017kinetics}, (bottom-left) audio \cite{gtanz1999music}, and (bottom-right) Cryo-EM derived protein density maps \cite{wwpdb2024emdb}.  %
[Center] A DDPM forward process on these modalities noises high-frequency components substantially faster (SNR changes more per time increment), and earlier than low-frequency components. %
[Right] The alternate EqualSNR forward process noises all frequencies at the same rate, disrupting DDPM's generation hierarchy.
The GIFs are best viewed in Adobe Reader.
}
\label{fig:fig1}
\end{figure*}

Diffusion models are the state-of-the-art generative model on data modalities such as images \cite{rombach2022high,baldridge2024imagen}, videos \cite{videoworldsimulators2024,ho2022imagen,blattmann2023align}, proteins \cite{watson2023novo,lewis2024scalable}, and materials \cite{zeni2023mattergen}.  
They excel on a wide range of tasks and applications on these modalities, such as generating high-resolution images and videos given a text prompt \cite{rombach2022high}, sampling the distribution of conformational states of proteins \cite{lewis2024scalable}, or generating novel materials under property constraints \cite{zeni2023mattergen}.
\textit{Why do diffusion models work so well on these modalities}, possibly even surpassing the performance and efficiency of  autoregressive models?   %

In this work, we study this question via %
the forward (or noising) process of diffusion models in Fourier space. %
The forward process of standard diffusion models such as DDPM \cite{song2020denoising} corrupts data by progressively adding white Gaussian noise until all information is destroyed.\footnote{White noise has equal variance across all frequency components.}   %
Diffusion models then learn a denoiser which reverses this forward process by starting from Gaussian noise and iteratively refining it to approximate the original data \cite{song2020score,sohl2015deep}.

While one might assume that the forward process destroys all information in data %
uniformly, this is not the case.    %
In fact, for the modalities mentioned above, a DDPM forward process does \textit{not} treat all frequency components equally. 
These modalities have in common that they exhibit a \textit{power law} in their Fourier representation: 
low-frequency components have orders of magnitudes higher variances (and magnitudes) than high-frequency components (see \Cref{fig:fig1} [left], \cite{van1996modelling}).
This data property has two important implications:  %
The DDPM forward process noises high-frequency components both substantially earlier \cite{rissanen2022generative}, and faster than low-frequency components (see \Cref{fig:fig1} [centre]), which we will discuss in \Cref{sec:ddpm} and theoretically characterise with the \textit{Signal-to-Noise Ratio (SNR)} (see \Cref{defn:snr}).   
Intuitively speaking, the forward process in DDPM corrupts  
the high frequency information---the fine details such as edges---in fewer timesteps
than low-frequency features, such as the larger structures and overall colour of an image
(see \Cref{fig:fwd-bwd-high-low} for an illustration).
What is the inductive bias of this forward process on the learned reverse process?

In theory (i.e. with unlimited resources and arbitrarily accurate estimates of the scores of the noised distributions), diffusion models can learn to express any continuous distribution (see for e.g. Theorem 2 in \cite{chen2023sampling}).
In practice, however, diffusion models are constrained, for instance by a limited number of intermediate steps (discretisation) and the expressiveness of the neural network (score estimation), resulting in approximation errors.
Since DDPM applies noise more aggressively to high-frequency components, corrupting them in fewer steps, we hypothesise that their approximation error is larger, resulting in a lower generation quality.   %
As a result, DDPM prioritises low-frequency components within its resource constraints.   %

The forward process also imposes a hierarchy of the frequencies during generation: 
as the generative process learns to reverse the forward process (which in DDPM and on the modalities of interest noises high frequencies before low frequencies), the reverse process generates low frequencies first, and generates high frequencies conditional on low frequencies (see \Cref{fig:fwd-bwd-high-low} for an illustration).
Previous work has observed this phenomenon of a soft-conditioning or ``approximately autoregressive'' generation in DDPM diffusion models \cite{dieleman2024spectral,gerdes2024gud}, drawing an important parallel with Large Language Models (LLMs).
However, in spite of this `hidden structure' in diffusion models, we lack understanding of the degree to which it is essential.  %
Can diffusion models also generate all frequencies at the same rate, i.e. without any hierarchy (see \Cref{fig:fig1} [right]), and how does this perform?
Can we also generate high-frequency details first, then low frequencies conditional on them (see \Cref{fig:flipped-snr-gif} in \Cref{app:Additional experimental details and results})? 

In this work, we study the forward process of diffusion models on data exhibiting the Fourier power law and its effect on the learned reverse process in Fourier space.
We theoretically and empirically analyse the impact of noising rate and hierarchy on the assumptions of the reverse process, and in turn on generation quality in diffusion models.
We organise our \textit{contributions} as follows: %
\begin{itemize}[itemsep=7pt, topsep=2pt, parsep=2pt, partopsep=2pt, leftmargin=7pt]  %
\item In \Cref{sec:ddpm}, we formally describe the DDPM forward process and define the SNR. We observe that this forward process results in a non-uniform rate of noising across frequencies: specifically, high-frequency components are noised faster (in the sense that the SNR descreases more per time increment) and earlier (in the sense that the SNR of high-frequency components is smaller than the SNR of low-frequency components at all timesteps). 
\item  In \Cref{sec:technical_explanation}, we theoretically explain and empirically demonstrate that faster noising as measured by SNR of high-frequency components violates the Gaussian assumption in the backward process, leading to worse generation quality for these frequencies.
We then propose a framework for %
forward processes based on frequency-specific SNR, including a training and sampling algorithm, a loss based on the Evidence Lower Bound (ELBO), and a mutual calibration procedure. In particular, we explore a variance-preserving forward process which noises all frequencies at the same rate~(EqualSNR), disrupting the usual generation hierarchy in DDPM. 
\item In \Cref{sec:experiments}, we demonstrate that EqualSNR performs on-par with DDPM on standard image benchmarks. Additionally, it improves the generation quality of high frequencies, and consequently the performance of diffusion models on datasets where high frequencies are dominant. 
In particular, we show that for DDPM, the high-frequency components of generated images can be distinguished `by eye' from the high-frequency components of the original data. In fact, a simple logistic regression classifier trained on two summary statistics of high-frequency components can discriminate DDPM-generated images from real images, while the same classifier performs poorly for EqualSNR.
This has important implications for generative modelling tasks where high-frequency details are of core interest, such as in astronomy or medical imaging, and for the generation of more realistic DeepFake data.
\end{itemize}

\begin{figure*}[!ht]
    \centering
    \begin{minipage}[t]{0.49\textwidth}
        \hspace{-0.5em}
        \includegraphics[width=1.03\textwidth]{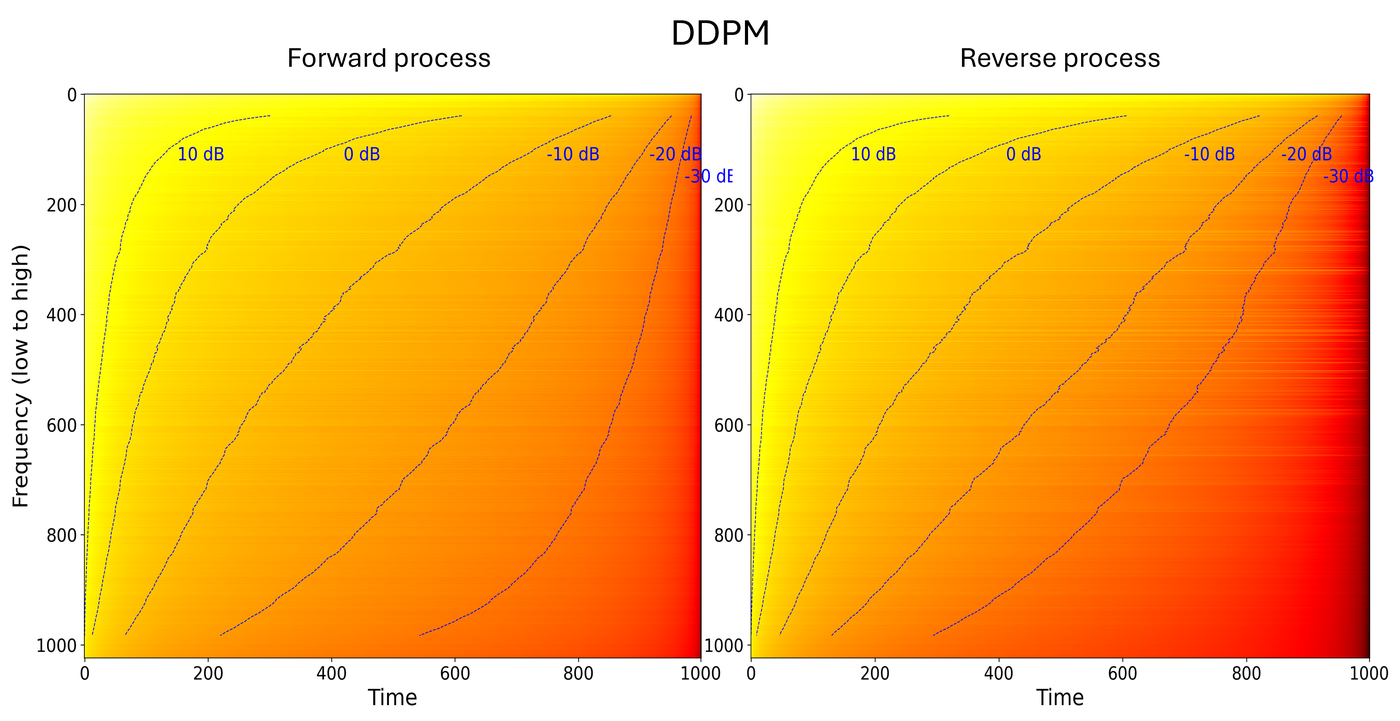}
        \caption*{DDPM: High frequencies are corrupted substantially faster (SNR changes more per time increment) than low frequencies. 
        }
        \label{fig:ddpm}
    \end{minipage}%
    \hfill
    \begin{minipage}[t]{0.49\textwidth}
        \centering
        \includegraphics[width=1.1\textwidth]{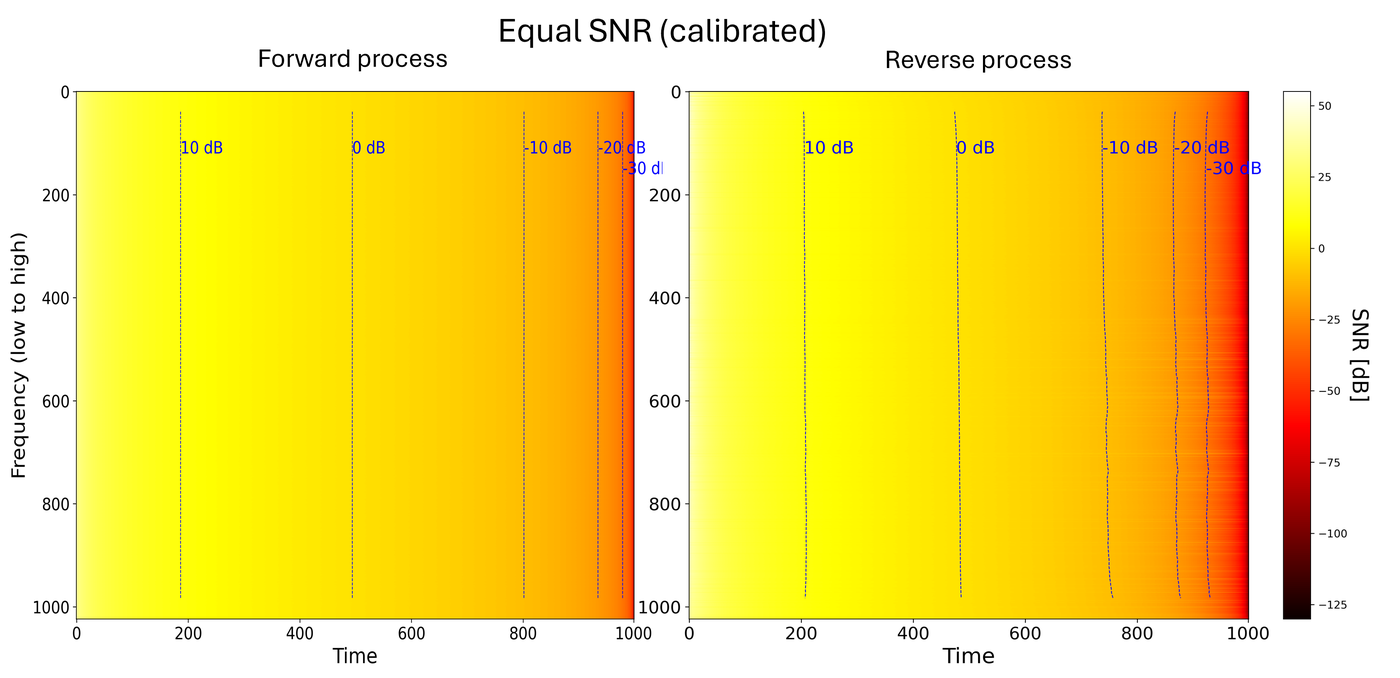}
        \caption*{EqualSNR: All frequencies are corrupted equally fast, achieving the same SNR of all frequencies at each timestep.}
        \label{fig:equal-snr}
    \end{minipage}%
    \caption{
         \textit{Comparing the SNR (dB scale) for DDPM and the alternate EqualSNR in Fourier space.} 
        In the forward process, SNR is computed as a Monte Carlo estimate of Eq. \eqref{eq:snr_definition} on CIFAR10  (referring to \Cref{app:Additional experimental details and results} for details on the reverse process).
        }
        
    \label{fig:snr_heatmaps}
\end{figure*}

\section{Background: a spectral analysis of DDPM}
\label{sec:ddpm}

In this section we recall the DDPM forward process and analyse how it affects the frequencies of a signal, showing that high-frequencies are corrupted faster, and substantially earlier than low-frequency ones. 
This effect can be quantified in terms of the Signal-to-Noise Ratio (SNR).
All proofs, further theoretical results, and an extended exposition of our theoretical framework are deferred to \Cref{app:Extended theoretical framework and proofs}.

\paragraph{The DDPM forward process. }

Given a data point $\bx_0 \in \mathbb{R}^d$ drawn from the data distribution whose density is given by $p(\bx_0)$, the DDPM forward process \cite{ho2020denoising} generates a sequence of latent variables $\{\bx_t\}_{t=0}^T$ which satisfy the following transitions:
\begin{equation}
q(\bx_t | \bx_{t-1}) = \mathcal{N}\left( \bx_t; \sqrt{1 - \alpha_t} \bx_{t-1}, \alpha_t \mathbf{I} \right),
\label{eq:ddpm_fwd_iterate}
\end{equation}
where $\alpha_t$ controls the amount of (white) noise added at each timestep $t$, which is equal for all data dimensions. 
The marginal distribution of $\bx_t$ given $\bx_0$ is then obtained in closed form: 
\begin{equation}
    q(\bx_t | \bx_0) = \mathcal{N}\left( \bx_t; \sqrt{\overline{\alpha}_t} \bx_0, (1 - \overline{\alpha}_t) \mathbf{I} \right),
    \label{eq:ddpm_yt}
\end{equation}
with $\overline{\alpha}_t = \prod_{s=1}^t (1 - \alpha_s)$. This can be reparameterised as
\begin{equation}
\bx_t = \underbrace{\sqrt{\overline{\alpha}_t} \bx_0 }_{\text{signal}} + \underbrace{\sqrt{1 - \overline{\alpha}_t} \boldsymbol{\epsilon}}_{\text{noise}}, \quad \boldsymbol{\epsilon} \sim \mathcal{N}(0, \mathbf{I}),
\label{eq:ddpm_pixel_reparam}
\end{equation}
where the first term on the right-hand side is the (scaled) signal, the second term is the (scaled) noise.

\paragraph{DDPM corrupts high-frequencies faster and earlier.} 

In this paper, we focus on data modalities where diffusion models achieve state-of-the-art performance, such as images, video, proteins, and materials.   %
These modalities share the property that---when viewed in their Fourier representation---their signal variance (and magnitude) decays with frequency by orders of magnitude (see Fig. \ref{fig:fig1} [left]).
We call this data property the \textit{Fourier power law} henceforth, which has for the example of images been previously noted in several works \cite{van1996modelling,hyvarinen2009natural,rissanen2022generative}. 
This property has two implications that we study in this work: fine details (high frequencies) are corrupted 1) faster, and 2) before larger structures (low frequencies) during the forward process \cite{kingma2024understanding}.

To see this, we can view the DDPM forward process equivalently under a change of basis to the Fourier space, a basis which has been of interest in the diffusion community \cite{dieleman2024spectral,gerdes2024gud} (see \Cref{sec:Related work} for a discussion). 
This is accomplished by applying the Fourier transform $\mathbf{F}$ to the latent variables $\mathbf{x}_t$ as:  %
\begin{align}
    \by_t := \bF \bx_t &= \bF(\sqrt{\overline{\alpha}_t} \bx_0) + \bF(\sqrt{1 - \overline{\alpha}_t} \beps) \nonumber\\
    &= \underbrace{\sqrt{\overline{\alpha}_t} \bF \bx_0}_{\text{signal } \mathbf{s}} + \underbrace{\sqrt{1 - \overline{\alpha}_t} \bF \beps}_{\text{noise } \mathbf{n}}  \label{eq:ddpm_fourier}
\end{align}
by linearity, and 
$\mathbf{n} \sim \mathcal{CN}(\mathbf{0},\bF (1 - \overline{\alpha}_t) \mathbf{I} \bF^\dagger) 
= \mathcal{CN}(\mathbf{0}, (1 - \overline{\alpha}_t) \mathbf{I})$, 
where $\bF^\dagger$ is the adjoint of $\bF$ and noting that $\by_t$ is complex-valued (see \Cref{app:Extended theoretical framework and proofs} for details on this calculation).   %
Since $\mathbf{F}$ is invertible, there is a one-to-one correspondence between a DDPM forward process in Euclidean (or pixel) space in \Cref{eq:ddpm_pixel_reparam} and in Fourier space in \Cref{eq:ddpm_fourier}, rendering these equivalent, alternative viewpoints.

\paragraph{Signal-to-Noise Ratio. }
We can quantify the corruption of the signal with the \textit{Signal-to-Noise Ratio (SNR)}.

\begin{definition}[Signal-to-Noise Ratio]
\label{defn:snr}
    Let $s$ and $n$ be two random variables with realisations in $\C$, 
    and $f(s, n) = s + n$ be a measurement process. 
    The Signal-to-Noise Ratio (SNR) is defined as
    \begin{equation}
        \mathsf{SNR}(f) = \frac{\mathrm{Var}[s]}{\mathrm{Var}[n]},
        \label{eq:snr_definition}
    \end{equation}
    where $\var(s) := \var(\re(s)) + \var(\im(s))$ (and likewise for $n$).
    For $d$-dimensional random vectors $\mathbf{s}, \boldsymbol{n}$, 
    we abuse notation and define $\mathsf{SNR}(\mathbf{f})$ entry-wise: 
    $\mathsf{SNR}(\mathbf{f})_i = \mathsf{SNR}(f_i)$.
\end{definition}
We can now compute the SNR of $(\mathbf{y}_t)_i$ in \Cref{eq:ddpm_fourier}, i.e. the SNR of frequency $i$ at timestep $t$ as  %
\begin{equation}
    s_t^{\text{DDPM}}(i) 
    := \mathsf{SNR}((\mathbf{y}_t)_i) 
    = \frac{\overline{\alpha}_t \mathbf{C}_{i}}{1 - \overline{\alpha}_t},
    \label{eq:snr_freq}
\end{equation}
where $\mathbf{C}_{i} := \var((\by_0)_i)$ represents the signal variance of frequency $i$ which decays rapidly with frequency for our data modalities of interest.  %
\Cref{eq:snr_freq} lets us formalise the two implications of the Fourier power law property: 
The SNR of low frequencies is orders of magnitudes higher than the SNR of high frequencies at all timesteps under a DDPM forward process (see red line in Fig. \ref{fig:fig1} [middle] for an animation) \cite{kingma2024understanding, dieleman2024spectral}.
Informally speaking, a DDPM forward process does \textit{not} `noise all frequencies equally'.
Relative to the signal, the white noise of the DDPM forward process, which is equal for all frequencies, corrupts high-frequency information faster, i.e. in fewer timesteps.   %
The SNR of high-frequency components changes more rapidly per time increment compared to low-frequency components. 

It further imposes a hierarchy onto the forward process where high frequencies attain a low SNR much earlier than low frequencies. 
As the generative process of diffusion models reverses the forward process, low-frequencies are generated earlier than high-frequencies, which can be viewed as being generated conditional on the former (see \Cref{fig:fwd-bwd-high-low} for an illustration).
In this work, we analyse the effect of the faster and earlier noising of high frequency information in DDPM on the learned reverse process and its generative performance.
In preview of our experimental results, we will demonstrate that high-frequencies generated with DDPM are of poor quality, and show that an alternate forward process (EqualSNR) alleviates this issue.
This has important implications for applications where high-frequency details are the key modelling objective, such as astronomy or medical imaging, and DeepFake technology.

\section{Fast noising disrupts high-frequency generation: alternate forward processes}  %
\label{sec:technical_explanation}

\begin{figure*}[ht]
    \centering
    \includegraphics[width=.7\linewidth]{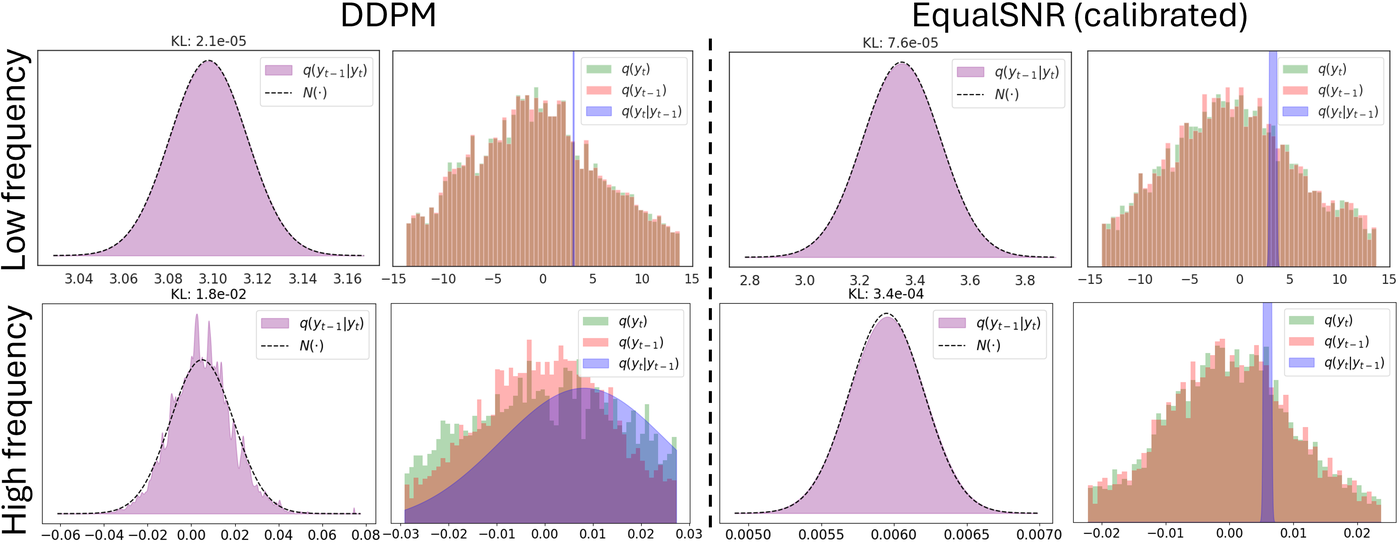}
    \caption{Fast noising of high frequencies leads to violations of normality in the DDPM reverse process. We plot Monte Carlo estimates of $q(\by_t) = \mathbb{E}_{\by_0 \sim q(\by_0)} q(\by_t | \by_0)$ (and similarly for $q(\by_{t-1})$) as histograms. %
    }
    \label{fig:gauss_assumption}
\end{figure*}

In this section, we theoretically and empirically investigate the consequence of our observation in \Cref{sec:ddpm} that high frequencies are noised faster than low frequencies in DDPM. 
Specifically, we will show that under limited resources, the Gaussian assumption of the reverse process is violated for high-frequency components. 
We then propose an alternate forward process in Fourier space and corresponding training and sampling algorithm which alleviates the issue.
In \Cref{sec:experiments}, we will experimentally investigate the effect of this violation on generation quality.  %

\subsection{Faster noising of high frequencies violates the Gaussian assumption in their reverse process. }
\label{sec:Aggressive noising of high frequencies violates the Gaussian assumption in their reverse process}

The stochastic reverse process of DDPM, which we consider first, iteratively denoises a (Fourier-transformed) iterate $\by_t$ by sampling from the learned reverse process distribution $p_\theta(\by_{t-1}|\by_t)$ which approximates the intractable distribution $q(\by_{t-1} \mid \by_t)$. %
A key assumption of the DDPM reverse process is that the distribution $q(\by_{t-1} | \by_t)$ is Gaussian, leading to the design choice that $p_\theta(\by_{t-1}|\by_t)$ is Gaussian.
In the limit of the total number of timesteps $T$ tending to infinity, $q(\by_{t-1} | \by_t)$ is known to converge to a Gaussian \cite{Feller1954DiffusionPI}.   %
However, in the practical setting of having a finite number of discretisation steps, this assumption only holds if the Gaussian noise added in the forward process $q(\by_{t} | \by_{t-1})$ is small enough relative to the signal variance.  %
Under a DDPM forward process on data modalities with the Fourier power law, which adds noise of equal variance to all frequencies but has orders of magnitudes smaller signal variance in the high frequencies, i.e. a faster noising of the high-frequency components (see \Cref{sec:ddpm}), this may lead to significant violations of the Gaussian assumption in high-frequency components.

To see this formally, we first apply Bayes rule:
\begin{equation}
\label{eq:fwd_bayes}
q(\by_{t-1} | \by_t) = \frac{q(\by_t | \by_{t-1}) q(\by_{t-1})}{q(\by_t)}.
\end{equation}
We note that if the Gaussian distribution $q(\by_t | \by_{t-1})$ has a large variance relative to $q(\by_{t-1})$ and $q(\by_t)$, which is the case for high frequencies in DDPM (see \Cref{sec:ddpm}), fluctuations in the quantity $\frac{q(\by_{t-1})}{q(\by_t)}$ are more apparent in the distribution $q(\by_{t-1}|\by_t)$. %
We formalise this intuition in \Cref{thm:backward_step}. We state the informal version here, with the formal version shown in \Cref{app:counterexample}.

\begin{proposition}[(Informal) Counterexample to normality of $q(\by_{t-1} | \by_{t})$]
\label{thm:backward_step}
There is a choice of sufficiently small positive constants $\delta, \tau$ such that the following holds. Let
\(D_{0} = \tfrac{1}{2}\,\mathcal{N}(-1,\,\delta^2) + \tfrac{1}{2}\,\mathcal{N}(1,\,\delta^2)\); 
and \(\mathbf{x}_{t-1} \sim D_{0}\) and let \(\boldsymbol{\varepsilon}\sim \mathcal{N}(0,4)\).
Then for the forward update, 
\(\mathbf{x}_t = \mathbf{x}_{t-1} + \boldsymbol{\varepsilon}\).
The corresponding reverse distribution 
\(q(\mathbf{x}_{t-1}| \mathbf{x}_t)\)
is at least a constant away in total variation distance from any Gaussian. \end{proposition}

\Cref{thm:backward_step} provides a counterexample to the assumption that the reverse process can be arbitrarily well-approximated by a Gaussian. 
It shows that starting with a mixture of two sufficiently separated Gaussians, if we add sufficiently large variance noise, then the distribution $q(\by_{t-1}|\by_t)$ is a constant away from any Gaussian (in fact, it also looks like a mixture of two Gaussians).
These deviations lead to errors in the reverse process which can accmulate across time \cite{li2023error}.
For instance, in the case of CIFAR10 data, this happens for high frequencies since the variance of the noise added to the high frequency components is much higher relative to the variance of the data.  %
For low frequency components on the other hand, since the variance of the data is large, this phenomenon does not occur.

In \Cref{fig:gauss_assumption} [left] we verify our theoretical analysis empirically for DDPM on CIFAR10. 
We estimate $q(\by_{t-1}|\by_t)$ (purple) via Bayes rule through \Cref{eq:fwd_bayes} for a [top] low frequency and [bottom] high frequency, focusing on the real component, one image channel and time step $t=1$ (see \Cref{app:Additional experimental details and results} Figs. \ref{fig:gauss_app_start} to \ref{fig:gauss_app_end} for further plots).
We observe that in DDPM, relative to the variance of the noisy signal $q(\by_t)$ and $q(\by_{t-1})$, the Gaussian $q(\by_t|\by_{t-1})$ has low variance for the low frequency components (as can be seen by the sharp peak), and a significantly larger variance for high frequencies.  
This results in violations of the Gaussian assumption for the posterior $q(\by_{t-1}|\by)$ of high frequencies, while low frequencies approximately follow a Gaussian distribution.
We can quantify this in terms of the KL-divergence, observing that it is orders of magnitudes higher for high-frequencies ($1.8 \times 10^{-2}$) than for low-frequencies ($3.1 \times 10^{-4}$).
Note that while the violations are small, resulting errors may accumulate across timesteps in the reverse process, adding up to a significant distortion in the final sample \cite{li2023error}.

\subsection{Alternate forward processes}

In the following and throughout our experiments in \Cref{sec:experiments} we study an alternate forward process which we call \textit{EqualSNR} where the SNR of all frequencies is the same at every timestep.  %
In \Cref{app:Extended theoretical framework and proofs} we also define \textit{FlippedSNR},  a forward process where the SNR of the $i^{\text{th}}$ frequency is the same as the SNR of the $(d-i)^{\text{th}}$, and which hence inverts the frequency hierarchy of DDPM during generation.
EqualSNR investigates two questions:
first, it noises all frequencies at the same rate. 
Recalling our analysis of the aggresiveness of noising in \Cref{sec:Aggressive noising of high frequencies violates the Gaussian assumption in their reverse process}, in ESNR we enforce the ratio of the variance of $q(\by_t|\by_{t-1})$ and the variance of $\frac{q(\by_t)}{q(\by_{t-1})}$ to be \textit{equal} for all frequencies at each timestep.
We can here compute the Monte Carlo estimate of $q(\by_t)$ and $q(\by_{t-1})$ via the push-forward in \Cref{eq:pushforward_fourier}.
Fluctuations in $\frac{q(\by_t)}{q(\by_{t-1})}$ hence affect all frequencies equally, and the Gaussian assumption of the reverse process in high-frequency components is no longer violated with similar distribution distances (see \Cref{fig:gauss_assumption} [right]; KL for low frequencies: $1.0 \times 10^{-4}$, compared to high frequencies: $1.3 \times 10^{-4}$), overcoming this issue of DDPM. %
Second, both EqualSNR and FlippedSNR address the question whether low-to-high frequency generation is essential in diffusion models: 
EqualSNR, which we will focus on in \Cref{sec:experiments}, generates samples without any hierarchy among the frequencies, while FlippedSNR inverts the hierarchy of DDPM, generating frequencies from high to low.
We now formalise this.

\begin{definition}[EqualSNR process in Fourier space]  
\label{prop:noise_variance}  
Let $\bC_i = \Var[(\by_0)_i]$ represent the coordinate-wise variance in Fourier space, and let $\boldsymbol{\epsilon} \sim \mathcal{C}\mathcal{N}(0, \Sigma)$. Suppose the forward process in Fourier space is given by 
$\by_t = \sqrt{\overline \alpha_t} \by_0 + \sqrt{1 - \overline \alpha_t} \boldsymbol{\epsilon}_\Sigma,$  
with the SNR at timestep $t$ and frequency $i$ defined as  
$s_t(i) = \frac{\overline \alpha_t \bC_i}{(1 - \overline \alpha_t) \Sigma_{ii}}$.  
This implies:  
\begin{enumerate}[noitemsep, topsep=0pt, parsep=0pt, partopsep=0pt, leftmargin=10pt]  
    \item \textbf{DDPM}: The forward process for DDPM has SNR  
    $s_t^{\text{DDPM}}(i) = \frac{\overline \alpha_t \bC_i}{(1 - \overline \alpha_t)}$.  

    \item \textbf{Equal SNR}: The forward process has equal SNR across all coordinates if and only if  
    $\Sigma_{ii} = c \bC_i$, where $c$ is a universal constant. The process is coordinate-wise `variance-preserving' (as in \cite{song2020score}) when $c = 1$.  
    \footnote{If we require that $\Sigma = c \Cov(\by_0)$, the equal SNR property holds across \emph{all bases} (see \Cref{sec:snr-high-d}).}  
\end{enumerate}  
\end{definition}

\begin{figure*}[ht!]
    \centering
    \includegraphics[width=.9\linewidth]{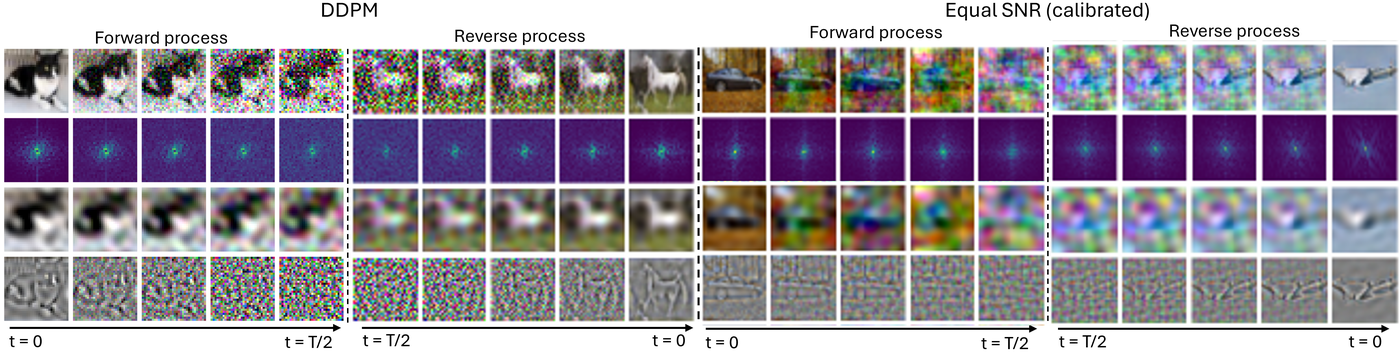}
    \caption{\textit{The forward process controls when frequencies are generated in the reverse process. }We visualise the forward  and backward process of [left] DDPM and [right] EqualSNR in pixel space (rows 1,3,4) and Fourier space (magnitudes; row 2).
    Rows 3 and 4 are low- and high-pass filters of the (noisy) image.
    DDPM noises high-frequency components first and hence generates them last, while EqualSNR noises and generates all frequencies at the same time.
    }
    \label{fig:fwd-bwd-high-low}
\end{figure*}

Note that we define the forward processes through their frequency-specific SNR $s_t(i)$ at time $t$, not the mixing coefficients (e.g. $\bar{\alpha}_t$), illustrated in \Cref{fig:snr_heatmaps} (together with a proxy-SNR for the reverse process, see \Cref{app:Additional experimental details and results} for details).
This is advantageous as it allows us to straight-forwardly apply such a forward process to high-resolution data without modification (in contrast to related work \cite{hoogeboom2023simple} where, in order to preserve the per-timestep SNR, they have to explicitly account for an adjustment to the mixing weights when generating high-resolution images).  
In the following, we summarise key framework components to train diffusion model in Fourier space with these alternate forward process, referring to \Cref{app:Extended theoretical framework and proofs} for details.

\begin{algorithm}
    \fontsize{10pt}{12pt}\selectfont  %
    \caption{Training algorithm (Fourier space noise)}   %
    \label{algo:training}
    \KwIn{
    Data samples $S := \{\B{x}_0^{(i)}\}_{i=1}^N$, 
    number of diffusion steps $\Tstep$, 
    noise schedule $\{\alb_\stept\}_{\stept=1}^\Tstep$,
    neural network $f_\theta$,
    number of training iterations $M$.}
    $\mathbf{C}:= \text{Diag}(\text{Cov}(\by))$\;
    \For{$M$ training iterations}{
        Sample $\B{x}_0 \sim S,  \by_0 = \B{F}\bx_0, \stept \sim \text{Uniform}(1, \ldots, T)$, for every $j \in [d/2]$, $(\epsilon_C)_j = \overline{(\epsilon_C)_{d-j}}$ and $(\epsilon_C)_j\sim \frac{\bC_j^{1/2}}{\sqrt 2} (\epsilon + i \epsilon')$ where $\epsilon, \epsilon' \sim \mathcal N(0, 1)$.\\
        \noindent  Fourier forward process:  $\mathbf{y}_t = \sqrt{\overline{\alpha}_t} \mathbf{y}_0 +
      \sqrt{1 - \overline{\alpha}_t}~\boldsymbol{\epsilon}_C 
    .$ \\ %
        \noindent Predict sample: $\hat{\B{y}}_0 = (\B{F} \circ f_\theta)(\B{F}^{-1}(\B{y}_t), t)$. \\
        \label{line:input_in_pixelspace}
        \noindent Compute loss $\mathcal{L}_\stept = \|\B{C}^{-1/2} (\B{y}_0 - \hat{\B{y}}_0)\|^2$. \;\\ %
        \noindent Update $\theta$ using gradient descent on $\mathcal{L}_t$.
    }
\end{algorithm}

\paragraph{Training algorithm and ELBO. }
\Cref{algo:training} allows to train diffusion models with the alternate forward processes in Fourier space (in the variant predicting the clean sample).
A key difference to standard DDPM training is that the noise and the difference in the loss $\mathcal{L}_t$ are scaled by the signal variances $\bC^{1/2}$ in Fourier space. %
\Cref{thm:loss_is_ELBO_complex} in \Cref{app:Extended theoretical framework and proofs} proves that $\mathcal{L}_t$ is an ELBO. 
To ensure a fair comparison with standard diffusion models, we train the neural network $f_\theta$ (typically a U-Net \cite{ronneberger2015u} in Euclidean/pixel space to maintain its inductive bias.

\paragraph{Sampling algorithm.}

\Cref{algo:sampling} adapts the Denoising Diffusion Implicit Models (DDIM) sampling algorithm \cite{song2020denoising} to Fourier space.
We note that while our analysis of the normality assumption in \Cref{sec:Aggressive noising of high frequencies violates the Gaussian assumption in their reverse process} was for the stochastic reverse process of DDPM, even though 
 DDIM is deterministic, a similar property holds here, which results in poor generation quality for faster noised frequencies.  

\begin{algorithm}
     \fontsize{10pt}{12pt}\selectfont  %
    \caption{Sampling algorithm (DDIM in Fourier space)}     %
    \label{algo:sampling}
    \KwIn{
    Input noise $\by_T := \boldsymbol{\epsilon}_C$ where for every $j \in [d/2]$, $(\epsilon_C)_j = \overline{(\epsilon_C)_{d-j}}$and $(\epsilon_C)_j\sim \frac{C_{jj}^{1/2}}{\sqrt 2} (\epsilon + i \epsilon')$ where $\epsilon, \epsilon' \sim \mathcal N(0, 1)$,
    neural network $f_\theta$,
    noise schedule $\{\overline \alpha_t\}_{t=1}^T$,
    sampling steps $T$.} 
    \For{$t$ from $T$ to $1$}{
        \noindent Predict sample: $\hat{\by}_0^{(t)} =  (\B{F} \circ f_\theta)(\B{F}^{-1}(\B{y}_t), t)$.\;\\
            $\by_{t-1} = \sqrt{\bar{\alpha}_{t-1}}\hat{\by}_0^{(t)} + 
            \frac{\sqrt{1-\bar{\alpha}_{t-1}}}{\sqrt{1-\bar{\alpha}_t}} (\by_t - \sqrt{\bar{\alpha}_t}\hat{\by}_0^{(t)})$. \;
    }
    \KwRet{$\B{F}^{-1}\by_0^{(1)}$}.
\end{algorithm}

\paragraph{Calibration. }
To compare performance across different forward processes, we calibrate them by ensuring that the average SNR across frequencies at any given timestep is the same. 
This means that the average amount of information destroyed across frequencies at timestep $t$ is the same across these processes. 
We refer to \Cref{subsec:calibration} for how to realise the calibration with an appropriate choice of the mixing coefficients $\bar{\alpha}_t$.

\begin{table*}[t]
    \centering
    \caption{
    EqualSNR performs on par with DDPM.%
    We measure performance using Clean-FID ($\downarrow$) \cite{parmar2021cleanfid}.
    }
    \label{tab:results_all-schedules}
\scriptsize
    \begin{tabular}{cc|cccc|cccc|cccc}
        \toprule
        & &\multicolumn{4}{c|}{\textbf{CIFAR10 (32 $\times$ 32)}} & \multicolumn{4}{c|}{\textbf{CelebA (64 $\times$ 64)}} & \multicolumn{4}{c}{\textbf{LSUN Church (128 $\times$ 128)}} \\
        & \hfill $T$ & 50 & 100 & 200 & 1000& 50 & 100 & 200 & 1000& 50 & 100 & 200 & 1000  \\
        \midrule
        & DDPM schedule & 18.63 & 18.01 &  17.68 & 17.7  & 10.10 &  8.72 & 8.30 & 8.62  &29.36 & 25.36  & 24.03 & 23.22    \\    %
        & EqualSNR (calibrated) schedule & 16.00 & 15.91 & 15.76 & 15.73& 9.45 & 8.79 & 8.62 & 8.56 & 19.42 & 19.75  & 19.90 & 19.80  \\  %
        \midrule
        & DDPM (calibrated) schedule &  16.64 & 14.69 & 14.07 & 13.85  & 12.65 & 7.88 & 6.54  & 6.59 & 40.31 &26.4  & 22.05 & 20.09 \\
           & EqualSNR schedule & 15.44 & 14.56 & 14.13 & 13.63   & 12.99 & 11.64 & 10.96  & 10.37 & 27.13 & 25.68  & 24.81 & 24.05  \\ %
        \bottomrule
    \end{tabular}
\end{table*}

\section{Experiments}
\label{sec:experiments}

In this section, we experimentally study the effect of the rate of noising in DDPM and EqualSNR on generation quality in Fourier space.  

\paragraph{Experiment setting.} 
We use three natural imaging datasets, a modality which exhibits the power law property in their Fourier representation, of different resolution: 
CIFAR10 ($32 \times 32$), CelebA ($64 \times 64$), and LSUN Church ($128 \times 128$). 
Furthermore, we use synthetic, high-frequency datasets which are described in the corresponding sections. 
We use a standard U-Net architecture throughout our experiments \cite{ronneberger2015u}.   %
\Cref{app:Additional experimental details and results}  provides additional experimental details and results. %

\paragraph{The reverse process learns to mirror its forward.}  %
To provide intuition, we begin by analysing the effect a forward (noising) process has on the reverse (generation) process in Fourier space. 
In \Cref{fig:fwd-bwd-high-low} we illustrate the forward and backward process for a single image for DDPM and EqualSNR schedule in pixel and Fourier space between $t=0$ and $t=T/2$ (see  \Cref{app:Additional experimental details and results} Figs. \ref{fig:fwd_ddpm_app} to \ref{fig:bwd_flippedsnr_app} for the full time interval, further examples and other datasets).
In DDPM, the low- and high-pass filtered images (rows 3 and 4; details on the calculation in \Cref{app:Additional experimental details and results}) show that the high frequency information is lost early on during the forward process. 
This enforces the reverse process to generate low frequencies first, and complete high frequency information only at the very end of the reverse process. 
In contrast, in EqualSNR the low- and high-pass images are corrupted at the same rate, which enforces the reverse process to learn to generate all frequencies simultaneously. 
 \Cref{fig:snr_heatmaps} (reverse process) and  \Cref{fig:freq_var_app} further illustrates this mirroring of the forward and reverse process computing the trajectory of variances of each frequency component at different timesteps (see \Cref{app:Additional experimental details and results} for details).

\paragraph{Faster noising degrades high-frequency generation quality in DDPM, EqualSNR overcomes this.} 
In Sections \ref{sec:ddpm} and \ref{sec:technical_explanation} we theoretically and empirically identified that high-frequency components are noised faster than low-frequency components, and showed that this leads to violations of the Gaussian assumption in their reverse process. 
We here investigate the effect of such (accumulating) violations on generation performance. 
We build on the Fourier-based Deepfake detection framework proposed by \cite{dzanic2020fourier}. %
Intuitively, this framework trains a classifier on a subset of the frequencies to discriminate generated and real data. 
We use classification performance as a (frequency-specific) proxy for generation quality. 
All experiments were performed with $T=1000$ sampling steps.

\begin{figure}[ht]
    \centering
    \includegraphics[width=\columnwidth]{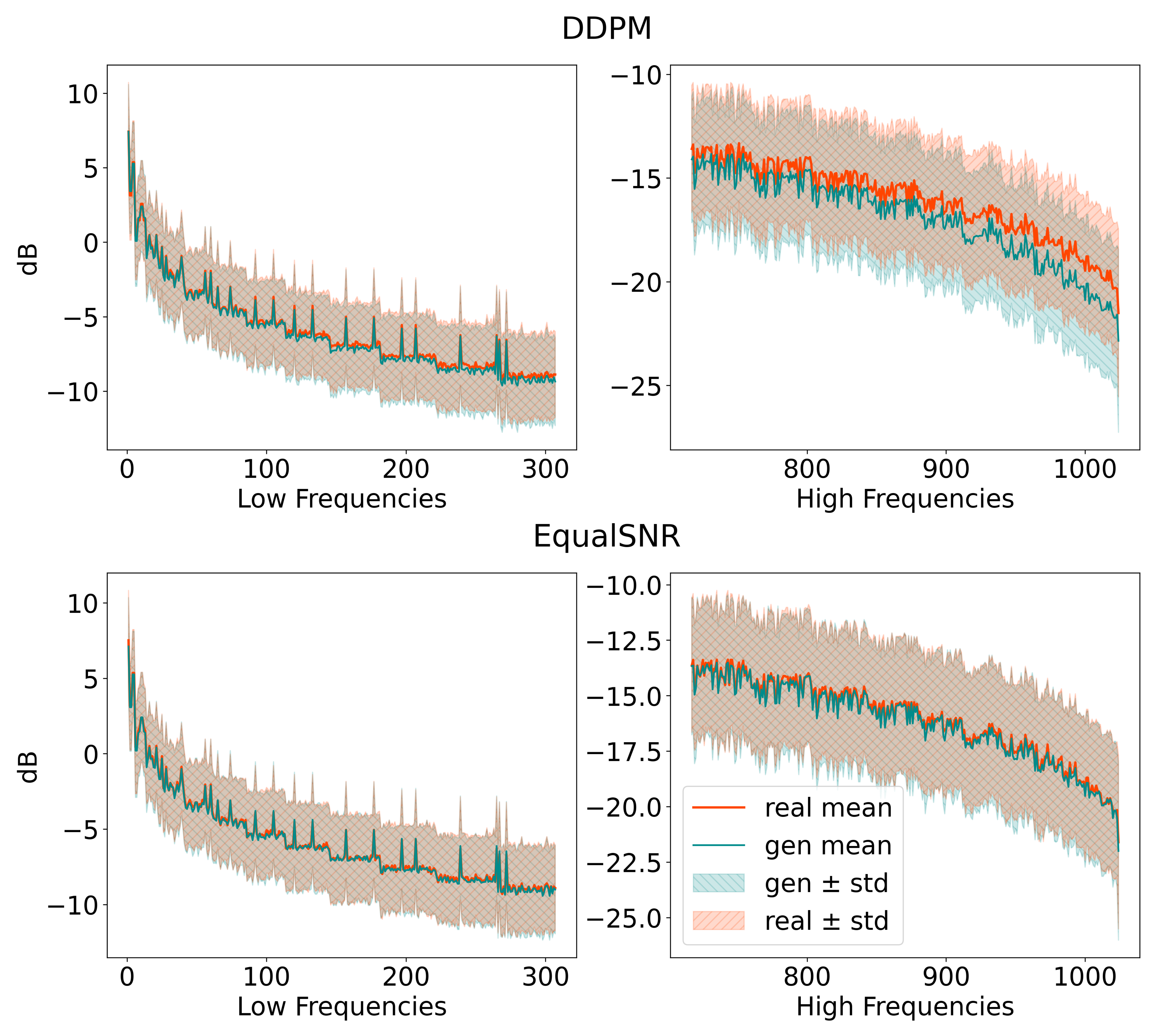}
    \caption{
    \textit{EqualSNR is superior to DDPM in high-frequency generation quality. }
    We plot the spectral magnitude profile (in decibels) for low and high frequencies, comparing data generated with [top] DDPM and [bottom] EqualSNR generated (blue) and real (red) data.
    }
    \label{fig:magnitutes-real-vs-generated}
\end{figure}

\Cref{fig:magnitutes-real-vs-generated} compares the spectral profile of magnitudes for CIFAR10 training data, and samples generated with a [top] DDPM and [bottom] EqualSNR diffusion model, respectively.  %
We plot the mean magnitude and standard deviation for the [left] lowest and [right] highest frequencies. 
Our analysis reveals a systematic failure of DDPM to accurately capture the high-frequency statistics of CIFAR10:
while the low-frequency magnitudes of generated samples perfectly match those of the data distribution, high-frequency magnitudes can be qualitatively distinguished, even `by eye'.
In contrast, EqualSNR demonstrates superior generation quality for the high frequencies with matching magnitudes for both low and high frequencies.

\begin{table}[ht]
\small
\centering
\caption{
\textit{High-frequency components generated with DDPM are easy to discriminate from real data, those generated with EqualSNR are not.}
We report classifier accuracy averaged over 100 runs and corresponding  \% (out of 100) of true positives (TP)  at significance levels $0.05$ and $0.01$. %
}
\label{tab:classifier-cifar10-high}
\begin{tabular}{c l c c c }
\toprule
\cmidrule(lr){2-4} %
& Freq. band & Mean Acc. &   \%TP at $0.05$ & \%TP at $0.01$  \\ %
\midrule
\multirow{3}{*}{\begin{turn}{90} DDPM \end{turn}}   & 5\% & 0.624 & 99\% & 99\% \\
& 15\% & 0.643 & 100\% & 99\%  \\
& 25\% & 0.654 & 100\% & 100\%  \\ \midrule
\multirow{3}{*}{\begin{turn}{90} \scriptsize EqualSNR \end{turn}} & 5\% & 0.516  & 13\% &5\% \\
& 15\% & 0.521  & 16\% &1\% \\
& 25\% & 0.518  & 10\% & 5\% \\
\bottomrule
\end{tabular}
\end{table}

In \Cref{tab:classifier-cifar10-high} we measure this mismatch of the approximate and data distribution for high-frequency components quantitatively and on a per-sample basis.
We isolate high frequency bands (top 5\%,  15\%,  25\%) and fit a  regression model with two parameters to the spectral magnitudes of each image following \citep{dzanic2020fourier}.   %
We then train a logistic regression classifier receiving these two regression parameters per image as input which discriminates real and generated images, repeated for 100 times. %
We intentionally choose a simple regression model and classifier for the benefit of the interpretability of the analysis, noting that more complicated classifiers which in particular take into account all frequencies may achieve substantially better performance.
To assess the validity of our classifier performance results and determine whether they are statistically significant (rather than occurring by random chance), we perform the following test:
we divide the data into 100 independent partitions and perform 100 hypothesis tests on the classifier's performance. 
For each test, our measure of success is the fraction of correctly rejected hypotheses (true positives) (see \Cref{app:Additional experimental details and results} for details on the statistical test).

For standard DDPM, we correctly reject the null hypothesis (distributions of real and generated images are the same) almost 100\% of the time for all high-frequency bands.
In contrast, for EqualSNR, we reject it at a much lower rate (10-13\%) at a significance level of 0.05, and 1-5\% at a significance level of 0.01.
Note when repeating the same experiment with low-frequency bands, DDPM and EqualSNR perform similarly (see \Cref{tab:classifier-cifar10-low} in \Cref{app:Additional experimental details and results}).
This underlines the advantageous high-frequency generation performance of EqualSNR compared to DDPM. 
We will exploit this property on data where high-frequency information is the key modelling objective next. 

\paragraph{When high-frequency information matters: a synthetic study.}  

\begin{figure}[ht]
    \centering
    \includegraphics[width=\linewidth]{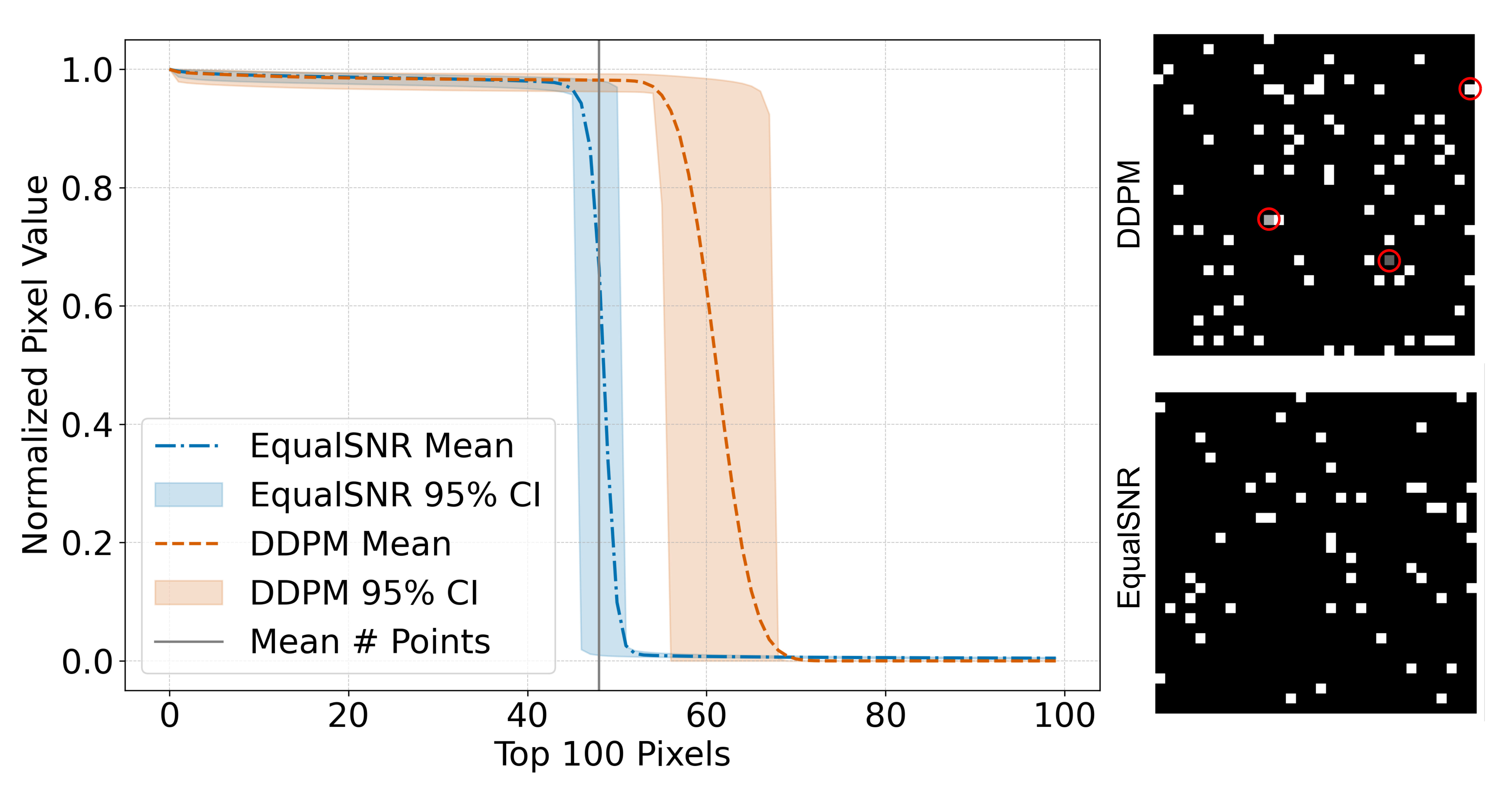}
    \caption{\textit{EqualSNR outperforms DDPM on data where high-frequency information is dominant.}
    Pixel intensity distribution (sorted descendingly) of 1000 generated samples for DDPM and EqualSNR, and two examples.
    }
    \label{fig:dots-intensity}
\end{figure}

We simulate real-world experiments where high-frequency, fine-grained information is of primary interest, such as astronomy, satellite and medical imaging, with a synthetic \textit{Dots} experiment.
We generate $32\times32$ images with between 46 and 50 white pixels randomly placed on a black background (see \Cref{fig:dots-intensity} [right] for generated examples). %
In \Cref{fig:dots-intensity} [left] we examine the average pixel intensity and the corresponding 95\% confidence intervals  over 1000 images for both DDPM and EqualSNR. %
EqualSNR significantly outperforms DDPM on this task: 
EqualSNR generates the correct amount of dots and has a steeper  slope than DDPM. 
DDPM, however, generates more points on average than present in the real data (gray line).
In conclusion, this suggests that an EqualSNR diffusion model is better at generating sparse, high frequency features, achieving a closer match to the original data distribution.

\paragraph{EqualSNR performs on par with DDPM on imaging benchmarks.}

Given the advantageous performance of EqualSNR on high frequencies, a natural question to ask is if---in light of the No Free Lunch Theorem \cite{wolpert1997no}---this harms generation quality on other data. 
We answer this question by training unconditional DDPM and EqualSNR diffusion models, where two runs are calibrated to each other, on standard imaging benchmarks, reporting obtained Clean-FID \cite{heusel2017gans,parmar2021cleanfid} values in \Cref{tab:results_all-schedules}.
Overall, we observe that EqualSNR performs on par with DDPM. 
We note that FID, which correlates with human perception of fidelity, may not capture the demonstrated issues in high-frequency accurateness observed for DDPM, contributing to the discussion of the appropriateness of the metric \cite{jayasumana2024rethinking}. 
On LSUN Church, the highest resolution dataset which contains more high frequencies, EqualSNR (calibrated) outperforms all DDPM variants.
Furthermore, we observe that EqualSNR requires less timesteps $T$ than DDPM to saturate performance.
This underlines that EqualSNR is a competitive forward process for diffusion models with beneficial performance on high-frequency data.

\section{Related work}
\label{sec:Related work}

\paragraph{Diffusion in Fourier space and other function spaces. }
Several works have studied diffusion models in Fourier space.  %
Most notably, concurrent work by \citet{gerdes2024gud} proposes a framework for diffusion processes on general function spaces, covering both Fourier and wavelet spaces.
Furthermore, they show that the forward process controls the degree of overlap of the `active time' of adjacent frequencies. 
They draw an important \textit{parallel between diffusion models and autoregressive models}: at 
one extreme, autoregressive models such as Large Language Models (LLMs) would generate each frequency component sequentially conditional on all previous ones.  %
The other extreme generates data without any hierarchy. 
The latter is implemented by \texttt{Equal SNR} which they propose but to the best of our knowledge do not experiment with (see \Cref{fig:fwd-bwd-high-low}).
\citet{jiralerspong2025shaping} concurrently analyse shaping the inductive bias of the forward process in Fourier space.  
They demonstrate how the forward process should be chosen in relation to the approximated data distribution.
\citet{crabbe2024time} show how to map the continuous-time formulation of diffusion models between an origin space and its Fourier transform, and apply diffusion models in the Fourier space to time series data, demonstrating superior performance. 
\citet{phillips2022spectral} likewise study diffusion in the spectral domain, particularly on multi-modal data.

\citet{guth2022wavelet} and \citet{phung2023wavelet} propose a diffusion model acting on multiple wavelet subspaces simultaneously, the former showing that time complexity increases linearly with image size.
\citet{jiang2023low} and \citet{huang2024wavedm} exploit wavelet representations of images for the inverse problems of low-light enhancement and image restoration.

\paragraph{Diffusion schedules and the importance of the rate of noising. }
Previous work also studied the importance of the noising schedule, also in light of how fast and early frequency information is corrupted.
\citet{kingma2021variational} learn the parameters of a variance-preserving noising schedule during training, observing that it spends more time in the high-SNR regime compared to other schedules, i.e. on generating high frequencies. 
\citet{williams2024score} provide an adaptive algorithm for finding a noising schedule which is optimal given a cost measuring the work required to transport samples along the diffusion path, and similarly observe on image experiments that this schedule spends more time on high-frequency details compared to a cosine schedule.
\citet{ziashahabi2024frequency} adapt the noising schedule in Fourier space to focus the limited resources on certain frequencies ranges, observing significant speedups at inference time while maintaining image quality.
\citet{voleti2022score}, similar to our work, propose a non-isotropic covariance structure in the forward diffusion process. 
\citet{yang2023diffusion} also found that high-frequency generation is harmed by DDPM diffusion models, but provide different explanations for this insight.

\paragraph{The low-to-high frequency hierarchy in diffusion models. }  %
Several works either explicitly or implicitly state the hypothesis that generating data from low to high frequency, and the corresponding inductive bias on the forward process which enforces this, is crucial for the success of diffusion models.  
In the seminal DDPM paper, \citet{ho2020denoising} observe that when generating samples conditional on noisy iterate, the sample shares almost all features (except for high-frequency details) with the iterate when the iterate is less noisy (i.e. is from late in the reverse process), but only shares the large-scale (i.e. low frequency) features when the iterate is early in the reverse process, and refer to this as ``conceptual compression''.
\citet{rissanen2022generative} observe the ``implicit spectral inductive bias'' of generating low frequencies before high frequencies and proposes a model which leverages this structure explicit. 
In this context, \citet{dieleman2024spectral} characterises diffusion models as approximate autoregressive models in Fourier space.   %
\citet{hoogeboom2023simple} hypothesise that high-frequency details can be generated in few diffusion steps when conditioning on already generated low-frequency features, which they exploit to adjust the noising schedule.
Our EqualSNR schedule puts this to the extreme in the sense that it spends the same time on all frequencies.
Importantly, to the best of our knowledge, no previous work showed to what degree the hierarchical structure in Fourier space of generating low frequencies before high frequencies is essential for diffusion models.
Several applications such as text-to-image and video generation of diffusion models also benefit from the low-to-high frequency generation of diffusion models \cite{yi2024towards,ruhe2024rolling}. 

Conditioning from low-to-high frequency information is further exploited in U-Nets \cite{ronneberger2015u}, a go-to architecture for diffusion models. %
\citet{williams2023unified} show that under DDPM diffusion noise, high-frequency components in a Haar wavelet basis have a substantially lower SNR than low frequencies.
This is exploited by U-Nets as average pooling, a common downsampling operation, is conjugate with projection to a lower-resolution Haar wavelet subspace, an alternate, frequency-sensitive basis. %
This means that the noisy high-frequency components, where most information is corrupted, are nullified along the levels of the encoder, rendering U-Nets efficient denoisers in diffusion models \cite{falck2022multi}.

Unrelated to the above works, \citet{rahaman2019spectral,wang2024frequency} identify a spectral bias of neural networks towards low frequencies (under strict network constraints), offering a potential alternative explanation for poor high-frequency generation.

\section{Conclusion}

In this work, we analysed the forward process of diffusion models in Fourier space, specifically the rate of noising and the induced hierarchy of frequencies, and its inductive bias on the learned reverse process.
We theoretically analysed and experimentally demonstrated that faster noising of high-frequency components as commonly done in DDPM degrades their generation quality, and showed that a diffusion model with a hierarchy-free forward process (EqualSNR) alleviates this issue, performing on par with the standard (hierarchical) DDPM forward process in terms of FID on standard imaging benchmarks.
We refer to \Cref{app:Additional experimental details and results} for a discussion of our limitations.
Future work should further investigate SNR-governed, alternate forward processes and tune them to the specific (Fourier) properties of the modality and task at hand.   %

\section*{Acknowledgements}
We acknowledge the blog post by `Diffusion is spectral autoregression' \cite{dieleman2024spectral} by Sander Dieleman which strongly motivated this work and inspired the design of \cref{fig:fig1}.
We thank Sam Bond-Taylor, Heiner Kremer, Andrew Y. K. Foong, Markus Heinonen and the colleagues at Microsoft Research Cambridge for useful discussions.

\section*{Impact Statement}

While our work is of methodological nature, it has potential, yet important safety concerns for the ability to generate highly-realistic DeepFake data \cite{westerlund2019emergence}, in particular images and videos, which have the Fourier power law property.
As we demonstrated in \Cref{sec:experiments}, a diffusion model with the alternate EqualSNR forward process showed marked improvements in the generation quality of high-frequency components. 
As a result, the generated samples were significantly harder to discriminate from real data by a classifier than samples from DDPM.
More generally, our SNR-governed framework for designing forward processes and the corresponding analysis may be abused to design `adversarial' forward processes which either make the generated samples even more realistic, or allow the modification of real data to appear generated from a diffusion model.
Future work should further investigate these safety concerns.

\nocite{langley00}

\bibliography{9_bib.bib}

\newpage
\appendix
\onecolumn

\section{Extended theoretical framework and proofs}
\label{app:Extended theoretical framework and proofs}

In this appendix, we will describe theoretical results in more detail. 
We begin by defining our notation and setting out basic concepts.

\subsection{Preliminaries}

\subsubsection{Notation}
Lowercase Greek letters, such as $\alpha, \beta, \gamma, \varepsilon$ denote scalars. 
Lowercase bold letters, such as $\bx, \by, \bz, \boldsymbol{\epsilon}, \dots$ denote vectors. 
If $\bx \in \C^d$ then $\bx = (\bx_1, \bx_2, \cdots, \bx_d)$.
Uppercase bold letters, such as  $\bI, \bC, \bF, \dots$ denote matrices. 
$[T] := \{1, \dots, T\}$.
$\| \cdot \|_2$ denotes the $\ell_2$-norm, i.e. for $\bx \in \bR^d$, $\| \bx \|_2 = \sqrt{\sum_{i=1}^d \bx_i^2}$. 
For two vectors $\bx, \by \in \C^d$, 
we define the inner product between $\bx$ and $\by$ to be $\bx \cdot \by := \sum_{i=1}^d \bx_i \by_i$.
$\cN(\boldsymbol \mu, \mathbf \Sigma)$ denotes the multivariate normal distribution with mean $\mathbf \mu$ and covariance $\mathbf \Sigma$. 
Sometimes we may write $\cN(\bx; \boldsymbol{\mu}, \mathbf \Sigma)$, this just means
that the random variable $\bx$ is drawn from $\cN(\boldsymbol \mu, \mathbf \Sigma)$.
$\mathcal{CN}(\boldsymbol{\mu}, \boldsymbol{\Sigma})$ is a complex Gaussian.
If $p(\bx_1, \dots, \bx_T)$ denotes the distribution (or the pdf of the distribution depending on context) over $T$ random variables and $S \subset [T]$, $p((\bx_i)_{i \in S})$ denotes the distribution marginalized over $\overline S$. 
For any vector $\bv \in \C^d$, 
let $i := \sqrt{-1}$ and $\text{Re}(\bv) \in \R^d$ denote the real part of $\bv$, 
and let $\text{Im}(\bv) \in \R^d$ denote the imaginary part. 

In the following subsections, we discuss the Fourier transform and calculate the covariance of the transformed data.

\subsubsection{Fast Fourier Transform}
\label{app:fft}

We define the discrete Fourier transform operator \(\bF_1\) acting on a vector 
\(\mathbf{x} = (x_0, x_1, \ldots, x_{N-1})^\top \in \mathbb{C}^N\) as
\[
(\bF_1 \mathbf{x})_k = \sum_{n=0}^{N-1} x_n \, e^{-2\pi i \, \frac{n k}{N}},
\quad k = 0, \dots, N-1.
\]
Note that \(\bF_1 \mathbf{x}\) is generally \emph{complex-valued}.

Images (which can be viewed as 2D grids of pixel values) often use a two-dimensional version of the FFT, for a vectorized image, this mapping corresponds to the tensor power of $\bF$, $\bF := \bF_1^{\otimes 2}$. Conceptually, the two-dimensional FFT applies the 1D transform (as defined above) first along each row and then along each column (or vice versa), yielding a complex-valued frequency representation of the image. 

\subsubsection{Covariance of Fourier transformed diffusion noise}
\label{app:cov_x_t}

Let $\bx_t := \bF \by_t$ where $\by$ is distributed according to Eq. \eqref{eq:ddpm_yt}. Then, 
\begin{align*}
    \mathbb{E}[\bx_t] &= \sqrt{\overline{\alpha}_t} \bF \by_0 \\
    \text{Cov}[\bx_t] &= \mathbb{E}[(\bx_t - \mathbb{E}[\bx_t]) (\bx_t - \mathbb{E}[\bx_t])^\dagger] \\
    &= \mathbb{E}[(\bF \by - \sqrt{\overline{\alpha}_t} \bF \by_0) (\bF \by - \sqrt{\overline{\alpha}_t} \bF \by_0)^\dagger] \\
    &= \mathbb{E}[(\bF (\by_t - \sqrt{\overline{\alpha}_t}  \by_0 )) (\bF (\by_t - \sqrt{\overline{\alpha}_t}  \by_0))^\dagger] \\
    &= \bF \mathbb{E}[ \by_t - \sqrt{\overline{\alpha}_t}  \by_0 )  (\by_t - \sqrt{\overline{\alpha}_t}  \by_0)^T] \bF^\dagger \\
    &= \bF \text{Cov}(\by_t) \bF^\dagger \\
    &= \bF (1 - \overline{\alpha}_t) \mathbf{I} \bF^\dagger \\
    &= (1 - \overline{\alpha}_t) \mathbf{I} %
\end{align*}

\subsection{Forward Processes and their Derivation}
\label{subsec:fwd-processes-variance-calculation}

\begin{definition}[Modified forward process in Fourier space]  %
\label{prop:noise_variance}
Let the random variable $\bx_0$ denote the vectorized signal in the data space (e.g. pixel space), and $\by_0 \in \C^d$ denote its Fourier transform.
Let $\bC_i := \Var[(\by_0)_i]$ denote the coordinate-wise variance in Fourier space, which we illustrate for three imaging datasets in \Cref{app:Additional experimental details and results} \Cref{fig:c_matrix}, and
let $\boldsymbol{\epsilon} \sim \mathcal{C}\mathcal{N}(0, \Sigma)$.
Then the forward process in Fourier space is written as
\begin{equation}
 \by_t = \sqrt{\overline \alpha_t} \by_0 + \sqrt{1-\overline \alpha_t} \boldsymbol{\epsilon}_\Sigma,
 \label{eq:pushforward_fourier}
 \end{equation}
which has SNR at timestep $t$ and frequency $i$ given by 
\begin{equation}\label{eq:SNR_noise_controlled}
    s_t(i) := \frac{\overline \alpha_t ~\bC_i}{(1-\overline \alpha_t) \Sigma_{ii}}.
\end{equation}
This implies:
\begin{enumerate}
    \item DDPM: The forward process for DDPM has SNR given by $s_t^{\text{DDPM}}(i) = \frac{\overline \alpha_t ~\bC_i}{(1-\overline \alpha_t)}.$
    \item {Equal SNR}: The forward process has equal SNR across every coordinate if and only if 
    $\Sigma_{ii} = c~ \bC_i$, where $c$ is a universal constant. 
    The process is `variance preserving' (in the sense of \cite{song2020score}) if $c = 1$.
    \footnote{Note that if we insist that $\Sigma = c~\Cov(\by_0)$ then 
    the equal SNR property holds in \emph{all bases} (this is shown in \Cref{sec:snr-high-d}).}
    \item {Flipped SNR}: The forward process has flipped SNR, i.e., $s_t^{\text{FLIP}}(i) := s_t^{\text{DDPM}}(d-i)$ if and only if $\Sigma_{ii} = \bC_i/\bC_{d-i}$.
\end{enumerate}
\end{definition}

We contrast the SNR of the standard DDPM noise schedule and the two alternate ones computed for CIFAR10 in \Cref{fig:snr_heatmaps}, referring to \Cref{app:Additional experimental details and results} Figures \ref{fig:heatmap_celeba} and \ref{fig:heatmap_lsun} for the SNR profiles on CelebA and LSUN Church, respectively.
EqualSNR corrupts information in all frequencies at the same rate by enforcing the same SNR for all frequencies at each timestep.
This schedule removes all hierarchy among the frequencies and generates them at the same rate across diffusion time \cite{gerdes2024gud}.
{FlippedSNR} on the other hand reverses the ordering of frequencies of DDPM in terms of SNR, noising low frequencies first.
This forces the reverse process to generate high frequencies first, then low frequencies.

Both EqualSNR and {FlippedSNR} have a base distribution and forward process which are data-dependent and match the coordinate-wise variance of the Gaussian noise added to the data. %
In fact, we propose that the SNR at timestep $t$ is a better choice of parameterization of the corruption process
for any frequency $i$.
That is, rewriting \Cref{eq:SNR_noise_controlled} as $\overline \alpha_t = \frac{s_t(i)\Sigma_{ii}}{\bC_i + s_t(i) \Sigma_{ii}}$, we can derive the mixing coefficients as a function of the variances of the data, noise, and $s_t(i)$.

\subsection{Calibration}
\label{subsec:calibration}

Since we are comparing different noising schemes which corrupt the signal at different rates (in the sense of having different
SNR values at differing time steps) we try to ensure that the average SNR across frequencies is the same for any pair of 
foward processes that we are comparing. We illustrate how this can be done for the case of DDPM and EqualSNR below. 

Let $A_{\text{ddpm}}(t)$ and $A_{\text{eq}}(t)$ denote the average SNR for DDPM and EqualSNR at timestep $t$,
let $d$ denote the number of frequencies, and suppose we fix $\{ \overline{\alpha}_t^{ddpm} \}_t$. 
We would like to solve for $\{ \overline \alpha_t^{eq} \}_t$ such that $A_{\text{eq}}(t) = A_{\text{ddpm}}(t)$ for all $t$. 
Using the formulae derived in \Cref{subsec:fwd-processes-variance-calculation}, we see that this is the same as:
\[ \frac{\overline \alpha_t^{eq}}{1-\overline \alpha_t^{eq}} = \frac{1}{d}\sum_i \frac{\overline \alpha_t^{ddpm} \bC_i}{1-\overline{\alpha_t}^{ddpm}} =  \frac{\overline \alpha_t^{ddpm} \frac{1}{d}\sum_i\bC_i}{1-\overline{\alpha_t}^{ddpm}}\]
Solving for $\overline \alpha_t^{eq}$ gives us 
\[\overline \alpha_t^{eq} =   \frac{\overline{\alpha}_t^{ddpm} \Bigl(\frac{1}{d}\sum_i \mathbf{C}_i\Bigr)}{\Bigl(1 - \overline{\alpha}_t^{{ddpm}}\Bigr)
+ \overline{\alpha}_t^{{ddpm}} \Bigl(\frac{1}{d}\sum_i \mathbf{C}_i\Bigr)}.\]

On the other hand, if we choose to parameterize or schedule in terms of the SNR, then having decided on the 
SNR schedule for DDPM, the SNR schedule for EqualSNR is just the average value across frequencies of the DDPM SNR.

\subsection{Our Training Algorithm and Connection to the ELBO}
\label{sec:loss_justification}

We denote the corrupted signal by $\by_t := \bF\,\bx_t \in \C^d$, where $\bF$ is a linear operator and $\bx_t$ is the signal at timestep $t$. We train a model $f_\theta(\by_t,t)$ by minimizing the MSE loss 
\(\cL_t(\theta) := \E\!\bigl[\|\bC^{-\tfrac12}(\by_0 - f_\theta(\by_t,t))\|^2\bigr]\), 
where $\bC \in \C^{d\times d}$ is Hermitian positive semidefinite. As in \cite{ho2020denoising,song2020score}, minimizing this loss at each $t$ is equivalent to maximizing a standard ELBO (evidence lower bound) on $\log p_\theta(\by_0)$.

\begin{proposition}[Loss Minimization as ELBO Maximization]\label{thm:loss_is_ELBO_complex}
Consider a forward corruption process $q(\by_1,\dots,\by_T\mid\by_0)$ and a reverse reconstruction process $p_\theta(\by_0,\dots,\by_T)$ with 
\(p_\theta(\by_{t-1}\mid\by_t) := \mathcal{CN}\!\bigl(\by_{t-1};\,\mu_\theta(\by_t,t),\,\tfrac{\alpha_t(1-\overline{\alpha}_{t-1})}{1-\overline{\alpha}_t}\,\bC\bigr)\). Minimizing the MSE loss 
\(\E\!\bigl[\bigl(\by_0-\mu_\theta(\by_t,t)\bigr)^\dagger\,\bC^{-1}\bigl(\by_0-\mu_\theta(\by_t,t)\bigr)\bigr]\)
is (up to constants) equivalent to maximizing the usual ELBO on $\log p_\theta(\by_0)$.
\end{proposition}

\begin{proof}[Proof Sketch]
By the standard argument from diffusion-based generative models \cite{ho2020denoising,song2020score}, we have 
\(\log p_\theta(\by_0)\ge \E_{q}[\log\tfrac{p_\theta(\by_0,\dots,\by_T)}{\,q(\by_1,\dots,\by_T\mid\by_0)\,}]\), 
which expands into a sum of KL divergences plus a prior term. Hence maximizing the ELBO is equivalent to minimizing  
\(\!\sum_{t=1}^T \E_{q}[D_\mathrm{KL}(q(\by_{t-1}\mid\by_t,\by_0)\,\|\,p_\theta(\by_{t-1}\mid\by_t))]\). 
Both $q(\by_{t-1}\mid\by_t,\by_0)$ and $p_\theta(\by_{t-1}\mid\by_t)$ are (circularly) complex Gaussian distributions with the same covariance structure (up to a scaling factor), so their KL divergence is proportional to 
\(\|\by_0 - \mu_\theta(\by_t,t)\|^2_{\bC^{-1}}\) (this follows from \Cref{fact:gaussian_KL} below). 
Summing over $t$ matches the MSE loss $\cL_t(\theta)$, implying that minimizing $\cL_t(\theta)$ at all timesteps maximizes the ELBO. 
\end{proof}
Let  \( \det(\cdot) \) denote the determinant, \( \text{tr}(\cdot) \) denote the trace and \( (\cdot)^\dagger \) denote the conjugate transpose. We use the following fact above:

\begin{fact}[KL Divergence Between Complex Gaussian Distributions]
\label{fact:gaussian_KL}
Let \( p(\mathbf{z}) = \mathcal{CN}(\mathbf{z}; \boldsymbol{\mu}_p, \boldsymbol{\Sigma}_p) \) and \( q(\mathbf{z}) = \mathcal{CN}(\mathbf{z}; \boldsymbol{\mu}_q, \boldsymbol{\Sigma}_q) \) be two \( d \)-dimensional complex Gaussian distributions, where \( \boldsymbol{\mu}_p, \boldsymbol{\mu}_q \in \mathbb{C}^d \) are the mean vectors, \( \boldsymbol{\Sigma}_p, \boldsymbol{\Sigma}_q \in \mathbb{C}^{d \times d} \) are the covariance matrices, which are Hermitian (\( \boldsymbol{\Sigma} = \boldsymbol{\Sigma}^\dagger \)) and positive semidefinite.

The Kullback-Leibler (KL) divergence \( D_{\text{KL}}(p \| q) \) between \( p \) and \( q \) is given by:
\[
D_{\text{KL}}(p \| q) = \frac{1}{2} \left[ \log \frac{\det(\boldsymbol{\Sigma}_q)}{\det(\boldsymbol{\Sigma}_p)} 
+ \text{tr}\left(\boldsymbol{\Sigma}_q^{-1} \boldsymbol{\Sigma}_p\right) 
+ (\boldsymbol{\mu}_q - \boldsymbol{\mu}_p)^\dagger \boldsymbol{\Sigma}_q^{-1} (\boldsymbol{\mu}_q - \boldsymbol{\mu}_p)
- d \right],
\]
\end{fact}
\noindent
This confirms that with identical (or proportionally scaled) covariances, the KL depends only on the mean mismatch $(\boldsymbol{\mu}_p-\boldsymbol{\mu}_q)$ and is thus proportional to the MSE term. Consequently, the loss $\cL_t(\theta)$ directly aligns with the KL terms in the ELBO decomposition.

\subsection{Signal to Noise Ratio in High Dimensions}
\label{sec:snr-high-d}
In this section, we define the signal-to-noise
ratio for high-dimensional vectors and note
some properties.

\begin{definition}[SNR]
\label{defn:snr}
Let $s$ (signal) and $\epsilon$ (noise) be random variables in $\mathbb{C}$, and let $f(s, \epsilon) = s + \epsilon$ represent a measurement process. The \textbf{signal-to-noise ratio} is defined as:
\[
\mathsf{SNR}(f) = \frac{\mathrm{Var}[s]}{\mathrm{Var}[\epsilon]},
\]
where for any random variable $x$ taking
values in $\C$, $\mathrm{Var}(x) = \mathrm{Var}(\mathrm{Re}(x)) + \mathrm{Var}(\mathrm{Im}(x))$. 
\end{definition}

For a multivariate random variable, 
we define the signal-to-noise ratio
in a particular direction $v$ below.

\begin{definition}[Multivariate SNR]
Let $\bs, \boldsymbol{\epsilon}$ be two 
random vectors in $\C^d$ denoting the signal and noise respectively. Let 
$f(\bs, \boldsymbol{\epsilon}) = \bs + \boldsymbol{\epsilon}$ be a measurement process.  
Then, for $\bv \in \C^d$, 
we define the signal-to-noise ratio of $f$ in the direction $v$ to be 
\[ \snr_v(f) := \frac{\var(\bs \cdot \bv)}{\var(\boldsymbol{\epsilon} \cdot \bv)}. \]
\end{definition}

In this paper, we often consider the case 
where the SNR is constant across different
frequencies of our signal in the diffusion corruption 
process defined. 
Let $\bF$ denote the fourier transform matrix.
This corresponds to having equal SNR in the directions $\{\bF_i \mid i \in [d]\}$.

An alternative way to define a constant-rate corruption process is one in which the SNR is identical \emph{in all directions}. In \Cref{fact:cov-snr} we show that a measurement process achieves equal SNR in all directions if and only if the covariance of the signal and the noise are proportional to each other.

A consequence of this is that it suffices to simply add covariance-matching noise in the standard pixel space to achieve constant SNR across all frequencies. 
Note that, since we only match the diagonal entries of the covariance, i.e. only 
aim for equal SNR over $\{ \bF_i \mid i \in [d]\}$, this does not apply to our noise schedule in \Cref{sec:technical_explanation}. 

\begin{lemma}[SNR and Covariance]
\label{fact:cov-snr}
Let $\bs$ and $\boldsymbol{\epsilon}$ 
be random vectors taking values in $\C^d$ 
such that for every $\mathbf{v} \in \mathbb{C}^d \setminus \{0\}$ , $0 < \var(\bv \cdot \beps), \var(\bv \cdot \bs) < \infty$. 
Let $f(\bs, \boldsymbol{\epsilon}) := \bs + \boldsymbol{\epsilon}$ be a measurement process, 
and let the $(p,q)$-th entry of the covariance matrix of a complex random variable $\bx$ be given by $[\Cov(\bx)]_{p,q} := \E_{\bx}[\bx_p \bx_q^*]$, where the asterisk denotes the conjugate. 
Then, the following are equivalent: 
\begin{enumerate}
    \item For every $\bv, \bv' \in \C^d \setminus \{0\}$, 
$\snr_{\bv}(f) = \snr_{\bv'}(f)$. 
    \label{item:snr_condition}
    \item For some positive constant $c$, $\Cov(\bs) = c ~\Cov(\boldsymbol{\epsilon})$. 
    \label{item:cov_condition}
\end{enumerate}
\end{lemma}
\begin{proof}
Since for every $\mathbf{v} \in \mathbb{C}^d \setminus \{0\}$ , the variances of $\mathbf{v} \cdot \mathbf{s}$ and $\mathbf{v} \cdot \boldsymbol{\epsilon}$ are in $(0, \infty)$, all quantities below are strictly positive.

Observe that for any $\bv, \bv' \in \C^d \setminus \{0\}$,
\begin{align}
\label{eqn:snr-cov}
\frac{\snr_{\bv'}(f)}{\snr_{\bv}(f)} 
= \frac{\var(\bv' \cdot \bs)/ \var(\bv' \cdot \beps)}{\var(\bv \cdot \bs)/ \var(\bv \cdot \beps)}
= \frac{\bv'^T \Cov(\bs) (\bv')^*/ \bv'^T \Cov(\beps) (\bv')^*}{\bv^T \Cov(\bs) \bv^*/ \bv^T \Cov(\beps) \bv^*}
= \frac{\bv'^T \Cov(\bs) (\bv')^*}{\bv'^T \Cov(\beps) (\bv')^*} \cdot \left( \frac{\bv^T \Cov(\bs) \bv^*}{\bv^T \Cov(\beps) \bv^*} \right)^{-1}.
\end{align}
To see that \Cref{item:cov_condition} implies \Cref{item:snr_condition}, substitute  
$\Cov(\bs) = c~\Cov(\beps)$ in \Cref{eqn:snr-cov}.

To see that \Cref{item:snr_condition} implies \Cref{item:cov_condition}, 
substitute $\snr_\bv(f) = \snr_{\bv'}(f)$ in \Cref{eqn:snr-cov}.
Rearranging, we see that for any $\bv, \bv' \in \C^d \setminus\{0\}$,
\begin{align}
\label{eqn:snr-cov-substituted}
\frac{\bv^T \Cov(\bs) \bv^*}{\bv^T \Cov(\beps) \bv^*}  = \frac{\bv'^T \Cov(\bs) (\bv')^*}{\bv'^T \Cov(\beps) (\bv')^*} = c,
\end{align}
where $c$ is some constant. 
Rearranging again and ignoring the equation involving $\bv'$, we see that for any $\bv \in \C^d \setminus\{0\}$, $\bv^T \Cov(\bs) (\bv)^* = c~ \bv^T \Cov(\beps) (\bv)^*$. The conclusion then follows from \Cref{fact:eq-quadratic} below. 

\begin{fact}[Equal Quadratic Forms]
\label{fact:eq-quadratic}
    Let $\bA$ and $\bB$ be conjugate-symmetric matrices in $\C^{d \times d}$.
    If $\bv^T \bA \bv^* = \bv^T \bB \bv^*$ for all $\bv \in \C^d \setminus \{0\}$,  then $\bA = \bB$.
\end{fact}
\begin{proof}
Let $\bC = \bA - \bB$. 
Then the condition is equivalent to having 
$\bv^T \bC \bv^* = 0$ for all $\bv$. 
Suppose towards a contradiciton that 
$\bC$ is not the zero matrix. 
Then there is some entry of $\bC$ that is nonzero.
Let this be the entry indexed by $(p, q)$. 
Consider $\bv$ where $\bv_t = 0$ for all $t \in [d] \setminus \{p, q\}$. 
Then, we see
\begin{align}
\label{eqn:C-expansion}
\bv^T \bC \bv^* = |\bv_p|^2 \bC_{p,p} + |\bv_q|^2 \bC_{q,q} + \bv_p \bv_q^* \bC_{p,q} + \bv_p^* \bv_q \bC_{q,p}.
\end{align}
If $p = q$, then set $\bv_p = \bv_q = 1$ to get a contradiction. Hence all the diagonal entries of $\bC$ are $0$. 
If $p \neq q$, then taking into account the fact that $\bC_{p,q} = \bC_{q, p}^*$, we see that \Cref{eqn:C-expansion} reduces to, 
\begin{align*}
\bv^T \bC \bv^* = \bv_p \bv_q^* \bC_{p,q} + \bv_p^* \bv_q \bC_{p,q}^*.
\end{align*}
If $\bC_{p,q}$ has a nonzero real component, then it suffices to set $\bv_p = \bv_q = 1$ to see a contradiction, since 
$\bv_p \bv_q^* \bC_{p,q} + \bv_p^* \bv_q \bC_{p,q}^* = 2 \re(\bC_{p,q})$
If $\bC_{p,q}$ has a nonzero imaginary component, 
then it suffices to set $\bv_p = 1, \bv_q = i$ to see a contradiction, since
$\bv_p \bv_q^* \bC_{p,q} + \bv_p^* \bv_q \bC_{p,q}^* = -i \bC_{p,q} + i \bC_{p,q}^* = 2 \im(\bC_{p,q})$. 
\end{proof}

\end{proof}

\clearpage

\subsection{Counterexample: Breaking the Gaussian Assumption}
\label{app:counterexample}

In this section we give a simple example showing that—even if each transition 
in the forward noising process is Gaussian—the corresponding reverse conditional distribution $q(x_{t-1}\mid x_t)$
need not itself be close (in total variation) to \emph{any} single Gaussian.  
Intuitively, if the marginal \(q(x_t)\) arises from adding noise to a mixture of two well-separated Gaussians, then the reverse conditional remains bimodal and cannot collapse into a unimodal Gaussian.

\begin{proposition}\label{prop:counterexample_TV}
Let $\tau \in (0, 0.5)$ and \(\delta \in (0, \tau^{20})\) be constants. Set
\[
D_0 \;=\; \tfrac12\,\mathcal{N}(-1,\delta^2)\;+\;\tfrac12\,\mathcal{N}(1,\delta^2),
\]
Suppose
\[
x_{t-1}\sim D_0,
\qquad
\varepsilon\sim \mathcal{N}(0,4),
\qquad
x_t = x_{t-1} + \varepsilon.
\]
Then with probability $1-\tau$ over the draw of $x_t$, the reverse kernel \(q(x_{t-1}\mid x_t)\) satisfies
\[
\inf_{\mu\in\mathbb R,\;\sigma>0}
D_{\mathrm{TV}}\bigl(q(x_{t-1}\mid x_t)\,,\,\mathcal{N}(\mu,\sigma^2)\bigr)
\;\ge\;\Omega(\tau^{18})
\]
In other words, no single Gaussian can approximate
\(q(x_{t-1}\mid x_t)\) to an accuracy beyond $O(\tau^{18})$ in total variation. We think of $\tau$ as being a constant, and so no Gaussian can approximate \(q(x_{t-1}\mid x_t)\) to an arbitrary accuracy. 
\end{proposition}

\paragraph{Proof idea.}
We will prove Proposition~\ref{prop:counterexample_TV} in three stages:
\begin{enumerate}
  \item \textbf{Compute \(q(x_{t-1}\mid x_t)\) explicitly.} 
    By Bayes’ rule, we show it is proportional to the original mixture
    times a shifted Gaussian,
    \[
      q(x_{t-1}=x\mid x_t= y)
      \;\propto\;
      D_0(x)
      \;\cdot\;
      \exp \left( -\frac{(x-y)^2}{8} \right),
    \]
    This exhibits two well-separated modes and symmetric, sub-Gaussian tails.
  \item \textbf{Show sub-Gaussian decay around each mode.}
    With probability $1-\tau$ for $\tau < 0.8$, $y \in [-1 - 4 \sqrt{ \log(1/\tau)}, 1+4 \sqrt{\log(1/\tau)}]$, which will imply that in the high probability regions, $\exp \left( -\frac{(x-y)^2}{8} \right)$ remains bounded. 
  \item \textbf{Lower-bound the total variation.}
    For any candidate \(\mathcal{N}(\mu,\tau^2)\), we split the argument into two cases, either
    \(\tau^2\le\delta\) (i.e. the Gaussian is too narrow to cover one of the peaks) or \(\tau^2>\delta\) (i.e. the Gaussian is too flat to have substantial overlap with either piece)
    and apply standard concentration/anticoncentration bounds to
    obtain a gap of at least \(\Omega(\tau^{18})\).
\end{enumerate}

\begin{proof}
\textbf{Step 1: Exact form via Bayes’ Rule.}
Bayes’ rule asserts:
\[
\underbrace{p(x\mid y)}_{\text{posterior}}
= \frac{\underbrace{p(y\mid x)}_{\text{likelihood}}\,\underbrace{p(x)}_{\text{prior}}}{\underbrace{p(y)}_{\text{evidence}}}.
\]
Here:
\begin{itemize}
  \item \(p(x)=D_0(x)\) is the prior density of \(x_{t-1}\).
  \item \(p(y\mid x)=q(x_t=y\mid x_{t-1}=x)=\mathcal{N}(y-x;0,4)\) is the forward-noise likelihood.
  \item \(p(y)=D_1(y)=\int p(y\mid x)\,p(x)\,dx\) is the marginal of \(x_t\).
\end{itemize}

Since convolution of Gaussians yields another Gaussian,
\(
D_1 = D_0 * \mathcal{N}(0,4)
= \tfrac12\,\mathcal{N}(-1,\delta^2+4)
+\tfrac12\,\mathcal{N}(1,\delta^2+4).
\)
Thus for any \(y\):
\[
q(x_{t-1}=x\mid x_t=y)
= \frac{\exp\bigl(-\tfrac{(y-x)^2}{8}\bigr)
\bigl[\tfrac12\exp\bigl(-\tfrac{(x+1)^2}{2\delta^2}\bigr)
+\tfrac12\exp\bigl(-\tfrac{(x-1)^2}{2\delta^2}\bigr)\bigr]}
{\tfrac12\exp\bigl(-\tfrac{(y+1)^2}{2(4+\delta^2)}\bigr)
+\tfrac12\exp\bigl(-\tfrac{(y-1)^2}{2(4+\delta^2)}\bigr)}.
\]

Since we want to view this as a distribution over $x$ for a fixed value of $y$, we may drop the denominator (since it is not a function of $x$), to see that  
\[
q(x_{t-1}=x\mid x_t=y)
\propto \exp\bigl(-\tfrac{(y-x)^2}{8}\bigr)
D_0(x)
\]

Next, we would like to show that with reasonable probability over $x$ and $y$, $A(x,y) :=\exp\bigl(-\tfrac{(y-x)^2}{8}\bigr)$ is bounded between two constants for $x, y$ lying in the high probability region.

\textbf{Step 2: Uniform control on \(A(x,y)\).}
For any $0.5 > \tau > 0$, define
\[
I = [-1-8\sqrt{\log(1/\tau)},\,1+8\sqrt{\log(1/\tau)}].
\]
Since $\delta < 2$, by standard Gaussian tail bounds,
\(\Pr_{y\sim D_1}[y\notin I]\le\tau\)
.
We now check that for $y \in I$; $x \in [-1-4\delta \sqrt{\log(1/\tau)}, -1+4\delta \sqrt{\log(1/\tau)}] \cup [1-4\delta \sqrt{\log(1/\tau)}, 1+4\delta \sqrt{\log(1/\tau)}]$ (the high probability region for $D_0$) and $\delta \in (0,1)$, 
\[\Theta\left(\tau^{4^2 \cdot 3^2/8}\right) < \exp\left(-\frac{(2+4\cdot 3\sqrt{\log(1/\tau)})^2}{8}\right) < |A(x,y)| < 1\]
This means that one of the modes is scaled down by at least $\Theta(\tau^{18})$.

\textbf{Step 3: Lower-bounding total variation.}
Let \(N=\mathcal{N}(\mu,\sigma^2)\) be arbitrary.  We consider two regimes. First, observe that we may think of $q(x_{t-1} = x | x_t = y) \propto A(x,y)\cdot D_0$, where $A(x,y) \in (\Theta(\tau^{18}), 1)$. $A$ will serve as a re-weighting of the mixture components; in the extreme case, the ratio of the mass of the smaller component to the mass of the larger component is $\Theta(\tau^{18})$. 

\begin{itemize}
\item[\emph{(a)}] If \(\sigma^2\le\delta\):  \(N\) is too narrow.  Align its mean with the mode having a larger mass, since this is the maximum overlap we can achieve; say WLOG this is \(\mu=1\).
Outside a window of width \(O(\delta)\) around \(x=-1\), the overlap is
\(\ge \Theta(\tau^{18}) - \exp(-\Omega(1)/\delta^2)\) (mass of the smaller component under $q$ minus overlap of the tail of $N$).

\item[\emph{(b)}] If \(\sigma^2>\delta\):  \(N\) is too flat.  Consider the tail
beyond \(10\delta\) of the peak at +1, when viewed as a mixture component of $q$, and when $\delta$ is sufficiently small:
\[
q(|x-1|\ge10\delta) \ge (1-\Theta(\tau^{18})) \cdot  \left(1 - \exp(-10^2/2) \right)
\]
However, the mass that $N$ places in this region is small:
\[
N(|x-1|\ge10\delta) \le O(\sqrt{\delta})
\]
Hence \(D_{TV}(q,N)\ge 0.9 - \Theta(\tau^{18}) - O(\sqrt{\delta})\).
\end{itemize}

In both cases we obtain
\[
\inf_{\mu,\tau}D_{TV}\bigl(q(x_{t-1}\mid x_t),\mathcal{N}(\mu,\tau^2)\bigr)
\ge \Theta(\tau^{18}) - O(\sqrt{\delta}).
\]
Which, for a sufficiently small choice of $\delta$, completes the proof. %
\end{proof}

\clearpage

\section{Additional experimental details and results}
\label{app:Additional experimental details and results}

In this appendix we provide further details and results on the four experimental analyses in \Cref{sec:experiments} (\Cref{app:Analysis1} to \Cref{app:Analysis4}),  further describe the implementation of figures in the main text (\Cref{app:Further details on other figures}), and present additional results beyond the main text (\Cref{app:Further experimental illustrations and results}).
We also briefly present the limitations of our work.

\paragraph{Limitations. }
\label{app:limitations}

Our work has three major limitations: 1) While we saw marked improvements for higher resolution experiments, our largest imaging dataset is of resolution $128 \times 128$, well beyond state-of-the-art generative models. 
2) While the FlippedSNR forward process did not succeed in numerous experiments, underlining the importance of low-to-high generation, we cannot prove that such a forward process cannot be learned.
3) Even though many other modalities have the Fourier power law property, our real-world experiments investigate standard imaging benchmarks only.

\subsection{Analysis 1: The reverse process learns to mirror its forward}
\label{app:Analysis1}

\paragraph{On Figures \ref{fig:snr_heatmaps}, \ref{fig:heatmap_celeba} and \ref{fig:heatmap_lsun}.}

\Cref{fig:snr_heatmaps} illustrates the SNR of the variance of the Fourier coefficients of the forward and reverse processes for DDPM and EqualSNR on the CIFAR10 dataset \cite{krizhevsky2009learning}.
Frequencies are sorted by Manhattan distance to the corner of Fourier space (ascending order), resulting in low frequencies arranged at the top, and high frequencies at the bottom of the heatmap, and are averaged across channels.
Figures show that DDPM noises high-frequencies at very early time steps, while EqualSNR noises all frequencies at a constant rate.
The forward process shows how the Fourier frequency values are corrupted in time, and the reverse process shows how a trained model denoises the corrupted frequencies.
We observe that both models learn to denoise the Fourier coefficients at the same rate as the forward process dictates.
Every plot shows the contour curves on 10 dB, 0 dB, -10 dB, -20 dB and -30 dB to facilitate visualize the trend of corruption across frequencies.

To draw the forward process figures, we follow equation \eqref{eq:ddpm_fourier} and compute SNR for all timesteps $t$, given a schedule $\overline{\alpha}_t$.
Frequencies are sorted from lowest (top) to highest (down); time runs between the data distribution ($t=0$) and the base distribution ($t=T$) with $T=1000$ steps. 
Explicitly, we sample an image and white noise, compute their Fourier representations and combine them together following equation \eqref{eq:ddpm_fourier}.
Noise is weighted with the schedule in both cases, and normalized with $\bC^{1/2}$ in the case of EqualSNR as in \cref{algo:training}.
After computing the variances of the Fourier coefficients across a dataset, values are converted to decibels following the formula $\text{dB} = 10\log_{10}(\cdot)$ and displayed in a heatmap.
Contour curves are running averages over the values that approximate best the dB threshold across frequencies.

To draw the reverse process figures, we start with a corrupted image at time $t=T$ and do a step by step denoising process until $t=1$ following \Cref{algo:sampling}.
Our parametrized UNet predicts $\hat{\by}_t$ after Fourier transformation, and from it we estimate the noise $\hat{\boldsymbol{\epsilon}}_0$ at time $t=0$:
\begin{equation}
    \hat{\boldsymbol{\epsilon}}_0 = \frac{1}{\sqrt{\alpha}_t}(\by_t - \sqrt{\overline{\alpha}_t} \by_0).
\end{equation}
We compute the SNR by weighting the predicted signal and noise appropriately, as signal and noise are now distinct:
\begin{equation}
    \mathsf{SNR}(\by_{t-1}) = \frac{\sqrt{\overline{\alpha}_{t-1}} \mathrm{Var}[\hat{\by}_0]}{{\sqrt{1-\overline{\alpha}_{t-1}}}\mathrm{Var}[\hat{\boldsymbol{\epsilon}}_0]}
\end{equation}  
We compute the variances of the coefficients across the whole dataset and convert the SNR values to decibels same as in the forward process for representation purposes.

Figures \ref{fig:heatmap_celeba} and \ref{fig:heatmap_lsun} illustrate the forward process on the CelebA and LSUN datasets, computed on 64x64 and 128x128 resolutions, respectively.  
The plotted figures retain the same schedule as the one from \Cref{fig:snr_heatmaps} and we notice that the high frequencies are noised excessively very early on.
This observation agrees with the analysis from \cite{hoogeboom2023simple} that requires the schedules to be rescaled depending on the resolution of the images.
Our proposed visualization provides a useful tool to study the correct planning of the forward process.

\paragraph{On \Cref{fig:fwd-bwd-high-low}. }
This figure presents the forward and backward pass trajectory of our diffusion models trained on standard imaging benchmarks. 
For the forward pass, we sample using the push-forward distributions \Cref{eq:ddpm_yt} for DDPM, and $\by_t = \sqrt{\overline \alpha_t} \by_0 + \sqrt{1 - \overline \alpha_t} \boldsymbol{\epsilon}_\Sigma$ (see \Cref{prop:noise_variance}) for EqualSNR and FlippedSNR, respectively.
For the backward, we draw trajectories using standard DDIM for the model trained with a DDPM forward process, and using \Cref{algo:sampling} for EqualSNR and FlippedSNR, respectively.

The low- and high-pass filtered images are computed by setting a subset of the Fourier space coefficients, specifically those within (low-pass) or outside (high-pass) a certain distance to the center of the Fourier space representation, to 0 while using identity for all others. This is achieved by masking the Fourier space representation of the (noisy) images (see \Cref{fig:filters} for an illustration).

\begin{figure}[h]
    \centering
    \begin{minipage}{0.3\textwidth}
        \centering
        \includegraphics[width=\textwidth]{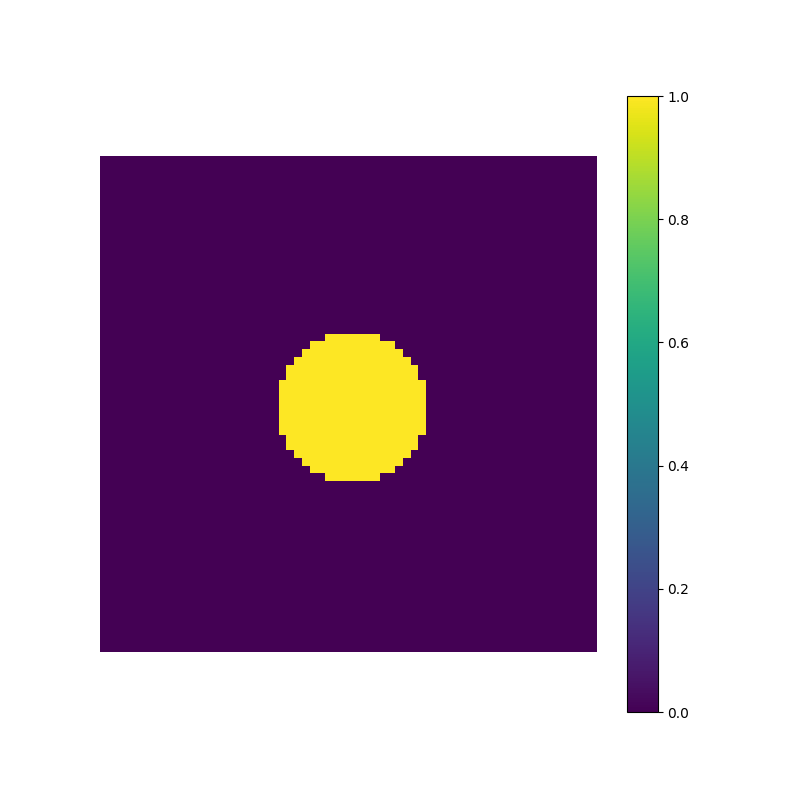}
    \end{minipage}
    \begin{minipage}{0.3\textwidth}
        \centering
        \includegraphics[width=\textwidth]{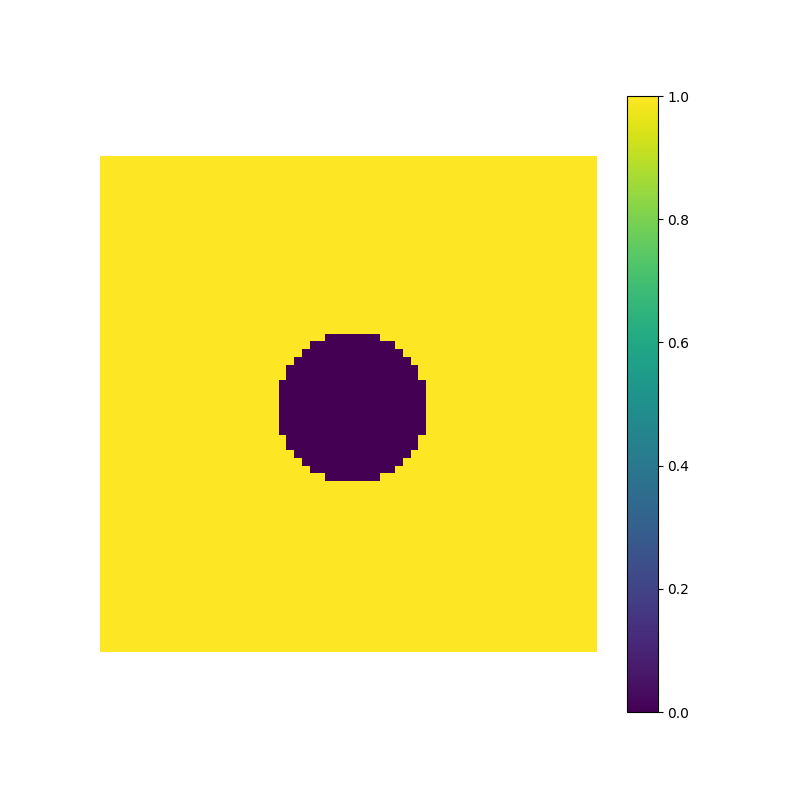}
    \end{minipage}
    \caption{Visualisation of the [left] low- and [right] high-pass filter masks.}
    \label{fig:filters}
\end{figure}

Figures \ref{fig:fwd_ddpm_app} to \ref{fig:bwd_flippedsnr_app} present further examples of the forward and backward process of DDPM, EqualSNR and FlippedSNR for the complete time interval $t=0$ to $t=T$, and also show the phase of the Fourier space representation as an additional row.

\subsection{Analysis 2: Faster noising degrades high-frequency generation in DDPM; EqualSNR overcomes this}

\paragraph{High Frequency/Low Frequency Classifier and Statistical Testing Procedure. }
\label{app:subsec_hf_lf_classifier}
The main idea is to differentiate between real and generated data by focusing on the high-frequency magnitudes of the images. This difference can be quantified using a classifier trained to distinguish these features. High classifier accuracy indicates low generation quality of the high frequencies. Following \citet{dzanic2020fourier}, instead of using all or a portion of the high-frequency magnitudes as classifier inputs, we fit a regression line to the log-magnitudes and use the slope and intercept as input parameters. Specifically, we find $b$ and $a$ such that: 
\begin{align*} 
m_i = a \left(\frac{f_i}{\tau}\right)^b, 
\end{align*} 
where $m_i$ is the magnitude of the $i$-th frequency and $f_i=\frac{i}{F}$, with $F$ being the number of frequencies. Furthermore, $\tau$ is the proportion of frequencies considered, such that $f_i \in [\tau,1]$. For each image, we obtain parameters $a$ and $b$, which are used as features for a logistic regression model with targets indicating whether the image is real or generated.
To quantify the significance of the obtained accuracy, we split the datasets into equal-sized chunks. For example, with 50,000 generated and real images, we create 100 splits, each containing 500 images per class. We then use each batch to compute a $p$-value via permutation testing that tests whether the accuracy is close to a random chance i.e. $0.5$. If the $p$-value is less than the chosen significance  (in this case $0.05$ or $0.01$), we can reject the null hypothesis that the two distributions are the same. Then, we count how many out of the 100 independent tests are correctly rejected, i.e. how many are the true positives. A large number of true positives indicates that we can differentiate the two classes with a high probability.

Similarly, we can construct a classifier for the low-frequency components. \Cref{tab:classifier-cifar10-low} shows that the mean accuracies for both cases are close to chance. The number of true positives is similar between them, with standard DDPM performing slightly better.

\begin{table*}[h]
\centering
\caption{Low frequency generation of DDPM and EqualSNR: Classifier mean accuracy from 100 runs and the corresponding number of true positives at significance levels $0.05$ and $0.01$ (correctly rejected statistical two-sample tests based on the classifier's performance).}
\label{tab:classifier-cifar10-low}
\begin{tabular}{l c c c c c c}
\toprule
& \multicolumn{3}{c}{DDPM} & \multicolumn{3}{c}{EqualSNR} \\
\cmidrule(lr){2-4} \cmidrule(lr){5-7}
Freq. band & Mean Acc. &   \# TP at $0.05$ & \# TP at $0.01$  & Mean Acc. &  \# TP at $0.05$ & \# TP at $0.01$ \\
\midrule
5\% & 0.492 & 3\% & 2\% & 0.512  & 14\% &2\%\\
15\% & 0.501 & 5\% & 2\%  & 0.509  & 7\% &2\%\\
25\% & 0.508 & 13\% & 4\% & 0.511  & 13\% & 5\% \\
\bottomrule
\end{tabular}
\end{table*}

\subsection{Analysis 3: When high-frequency information matters: a synthetic study}
\paragraph{The Dots Experiment.} 
We sampled 50,000 images with 48 to 50 randomly placed white pixels (see \Cref{fig:dots-examples}) and conducted experiments using 10,000 training steps with a cosine noise schedule, utilizing the same U-Net architecture as in the CIFAR10 experiment. In this experiment, we compare EqualSNR and standard DDPM calibrated to EqualSNR. In the main paper, we argued that EqualSNR generates samples with higher pixel intensity compared to DDPM. Examples of this issue can be seen in \Cref{fig:dots-examples}.

\begin{figure}
    \centering
    \includegraphics[width=0.7\linewidth]{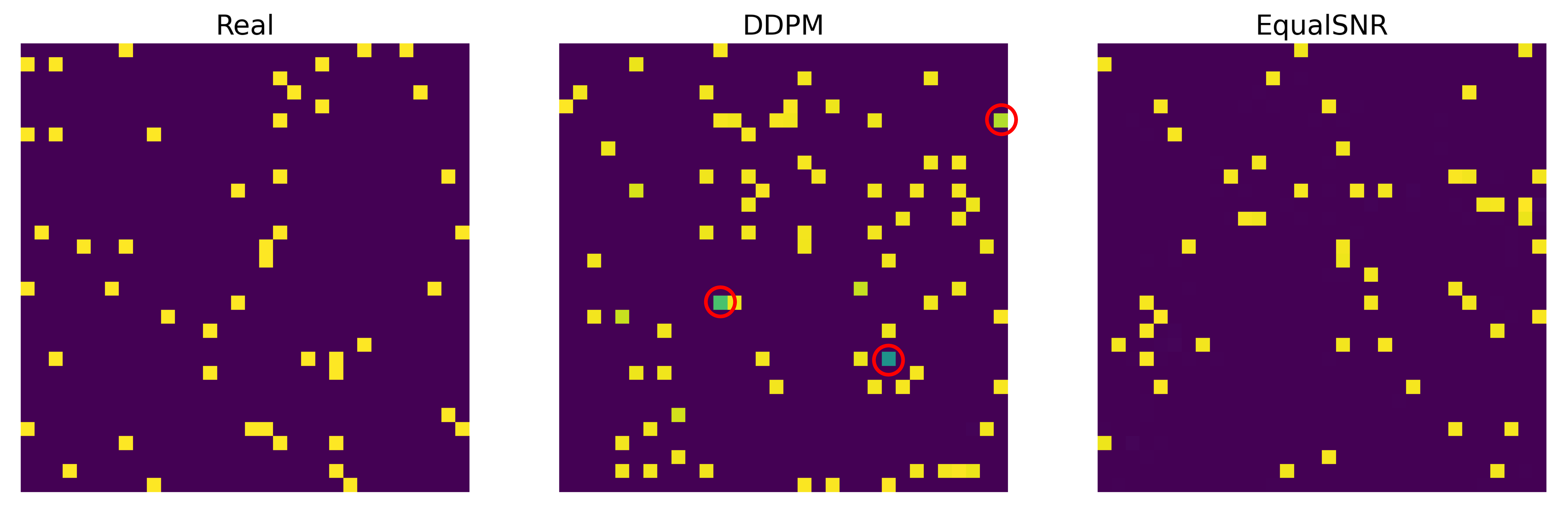}
    \caption{Dots Examples in the DDPM and EqualSNR case. 
    Yellow corresponds to `white pixels' (highest intensity) on this colour scale.
    The lower intensity images are marked by a red circle.}
    \label{fig:dots-examples}
\end{figure}

\subsection{Analysis 4: EqualSNR performs on par with DDPM on imaging benchmarks.}
\label{app:Analysis4}

In our implementation we used a UNet architecture derived from HuggingFace Diffusers' library to model our denoiser.  
We feed the UNet input images and train our models by parametrizing the loss with the output of the UNet corresponding to images of the data distribution ($t=0$) in the case of DDPM.
For EqualSNR we convert the output of the UNet to Fourier domain and use its values as input into the loss $\mathcal{L}_\stept = \|\B{C}^{-1/2} (\B{y}_0 - \hat{\B{y}}_0)\|^2$ (see \Cref{algo:training}) against the input image in Fourier domain.
Note, that the loss using Fourier coefficients is a novel parametrization providing competitive performance.
We train all models for 800k steps using a cosine schedule and a learning rate of 0.0002.
We report results on CIFAR10 (\cite{krizhevsky2009learning}; $32 \times 32$), CelebA ($64 \times 64$), and LSUN Church ($128 \times 128$) datasets.   

\Cref{tab:results_all-schedules} shows Clean-FID scores \cite{parmar2021cleanfid}, where
the values we provide are on par with values provided in their repository from other models they evaluate.
We chose to present this implementation in the main paper because it fixes multiple issues when comparing FID scores used by the community.
We further provide in \Cref{tab:results_all-schedules-torcheval} \texttt{PyTorch-FID} implementation presenting values comparable to DDIM scores \cite{song2020denoising}.
These values are significantly lower, highlighting inconsistencies between package implementations.

We tried feeding Fourier values directly into the input of the UNet but this parametrization did not seem to provide competitive results.

\begin{table*}[t]
    \centering
    \caption{
    EqualSNR performs on par with DDPM on standard imaging benchmarks. We measure performance using torcheval FID ($\downarrow$).
    }
    \label{tab:results_all-schedules-torcheval}
\scriptsize
    \begin{tabular}{cc|cccc|cccc|cccc}
        \toprule
        & &\multicolumn{4}{c|}{\textbf{CIFAR10 (32 $\times$ 32)}} & \multicolumn{4}{c|}{\textbf{CelebA (64 $\times$ 64)}} & \multicolumn{4}{c}{\textbf{LSUN Church (128 $\times$ 128)}} \\
        & \hfill $T$ & 50 & 100 & 200 & 1000& 50 & 100 & 200 & 1000& 50 & 100 & 200 & 1000  \\
        \midrule
        & DDPM schedule & 7.7 & 6.1 &  5.35 & 4.81  & 4.93 &  2.60 & 2.05 & 1.87  &28.72 & 20.23 & 17.99 & 16.63    \\    %
        & EqualSNR (calibrated) schedule & 6.5 & 5.1 & 4.57 & 4.27& 3.48 & 2.56 & 2.36 & 2.19 & 17.67 & 15.4  & 14.43 & 13.61\\  %
        \midrule
        & DDPM (calibrated) schedule &  9.95 & 6.17 & 4.73 & 3.71  & 10.93 &4.00 & 2.22  & 1.86 & 56.88 & 27.9  & 18.61 & 15.3 \\
           & EqualSNR schedule & 8.36 & 5.65 & 4.68 & 3.9   & 5.83 & 3.71 & 3.08  & 2.71 & 27.55 & 21.60  & 18.99 & 17.27  \\ %
        \bottomrule
    \end{tabular}
\end{table*}

\subsection{Further details on other figures}
\label{app:Further details on other figures}

We here provide further details on computing the figures in the main text and corresponding extended figures in the Appendix which have not yet been discussed above.

\paragraph{On Figures \ref{fig:fig1} and \ref{fig:flipped-snr-gif}. }
We acknowledge \cite{dieleman2024spectral} for inspiration on these figures. The [center] and [right] figures are GIFs  best viewed in Adobe Reader. 

[Left].
We present four data modalities and corresponding datasets exhibiting the Fourier power law property: 
images (CIFAR10 \cite{krizhevsky2009learning}), videos (Kinetic600 \cite{kay2017kinetics}), audio (GTZAN Music Speech \cite{gtanz1999music}) and Cryo-EM derived protein density maps (EMDB \cite{wwpdb2024emdb}).
The Fourier transform is applied to a different number of dimensions (D) for each of these datasets: 2D for images, 3D (spatial and time) for videos, 1D (time) for audio, and 3D for protein density maps.
As the protein density maps are of different size between samples, we interpolate them to fixed-size tensors of size $[200,200,200]$, equal for all dimensions.
Audio and video samples are trimmed across the time dimension to 2 seconds (44100 values) and 1 second (30 frames), respectively, such that all items in a dataset are of equal dimension.
We compute the signal variance on (a subset of) the Fourier-transformed dataset. 
We plot running averages and running standard deviations, which illustrates the overall trend.
Frequencies are sorted by Manhattan distance to the center of Fourier space (ascending order, i.e. low to high frequency).

[Center] and [Right].
We plot the signal and noise variance at different timesteps of the forward process for DDPM and EqualSNR, respectively.
The signal variance is computed over the entire CIFAR10 dataset.
We refer to \Cref{fig:snr_heatmaps} for details on the calculation.

\paragraph{On \Cref{fig:gauss_assumption}. }
Further to our explanation in \Cref{sec:technical_explanation} we provide the following details.
While the distributions $q(\by_t)$ and $q(\by_{t-1})$ are in general intractable, we can approximate them as Monte Carlo estimates of $q(\bx_t) = \mathbb{E}_{\by_0 \sim q(\by_0)} q(\by_t | \by_0)$ where $q(\by_t | \by_0)$ is the push-forward distribution \Cref{eq:ddpm_yt} (and similarly for $q(\by_{t-1})$, and $q(\by_0)$ is the Fourier-transformed data distribution. 
We use 5000 samples to estimate $q(\by_t)$ and $q(\by_{t-1})$.
We plot histogrammes of these estimates (green, red) and the Gaussian $q(\by_{t}|\by_{t-1})$ (blue) on the right.
Since $q(\by_{t}|\by_{t-1})$ is datapoint-specific, we arbitrarily choose the 70\%-quantile of $q(\by_{t-1})$ as the mean of $q(\by_{t}|\by_{t-1})$, a representative point.
The smooth posterior distribution $q(\by_{t-1}|\by_t)$ are computed using Kernel Density Estimation (KDE) with the same bandwidth hyperparameter across all frequencies and timesteps.
Figures \ref{fig:gauss_app_start} to \ref{fig:gauss_app_end} present further timesteps and frequencies, and additional plots visualising the marginals estimated via Kernel Density Estimation, and the ratio of the marginals with the Gaussian $q(\by_t | \by_{t-1})$ overlayed, for DDPM and EqualSNR.

\subsection{Further experimental illustrations and results}
\label{app:Further experimental illustrations and results}

In this section we provide additional experimental results augmenting those presented in the main text.

\begin{figure}[p]
    \centering 
    
    \begin{minipage}{0.2\textwidth}
        \centering
        \includegraphics[width=\linewidth]{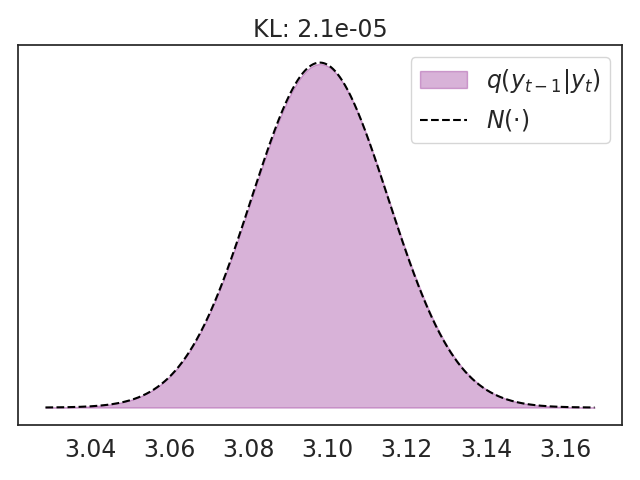}
    \end{minipage}
    \hfill
    \begin{minipage}{0.2\textwidth}
        \centering
        \includegraphics[width=\linewidth]{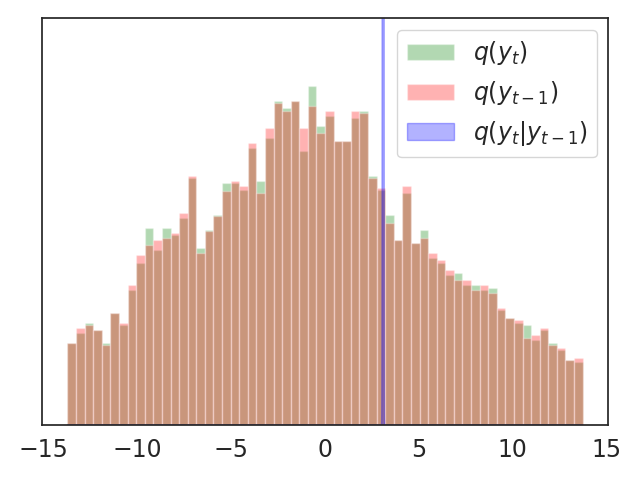}
    \end{minipage}
    \hfill
    \begin{minipage}{0.2\textwidth}
        \centering
        \includegraphics[width=\linewidth]{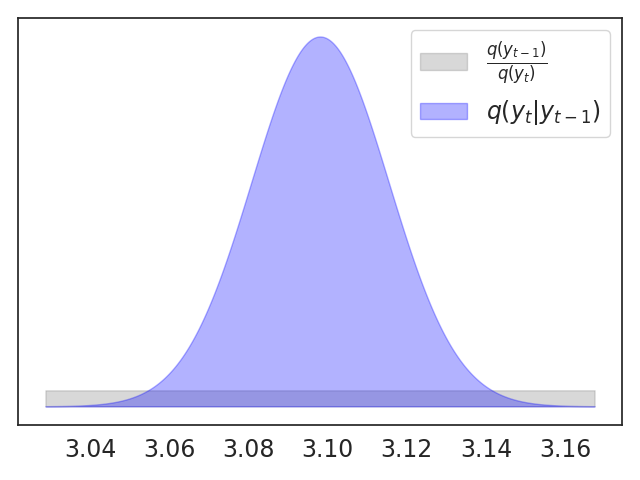}
    \end{minipage}
    \hfill
    \begin{minipage}{0.2\textwidth}
        \centering
        \includegraphics[width=\linewidth]{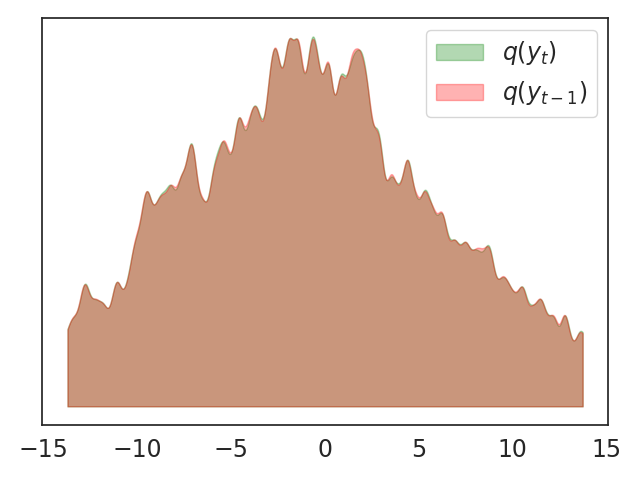}
    \end{minipage}
    \caption*{Frequency 1 (low), $t=1$}

    \vspace{5pt} %

    \begin{minipage}{0.2\textwidth}
        \centering
        \includegraphics[width=\linewidth]{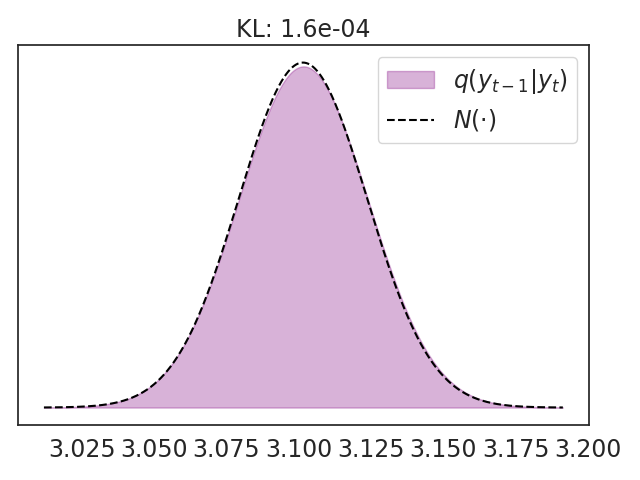}
    \end{minipage}
    \hfill
    \begin{minipage}{0.2\textwidth}
        \centering
        \includegraphics[width=\linewidth]{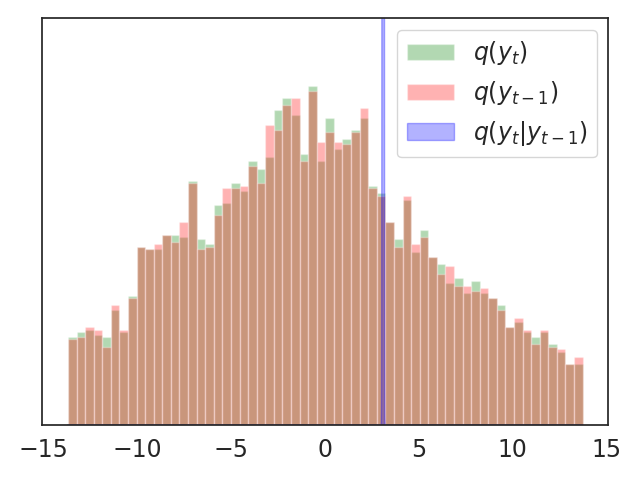}
    \end{minipage}
    \hfill
    \begin{minipage}{0.2\textwidth}
        \centering
        \includegraphics[width=\linewidth]{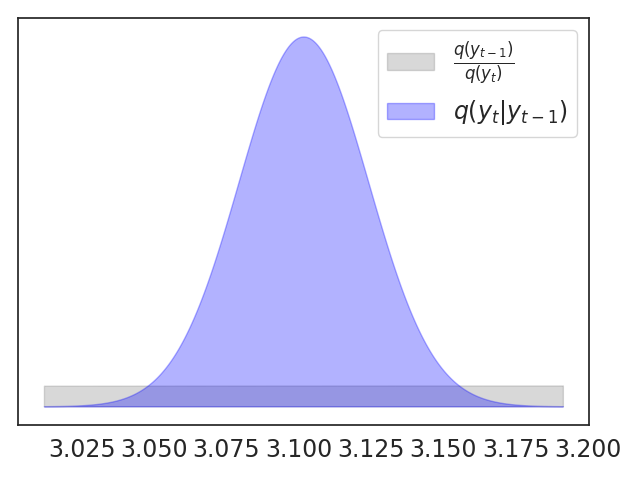}
    \end{minipage}
    \hfill
    \begin{minipage}{0.2\textwidth}
        \centering
        \includegraphics[width=\linewidth]{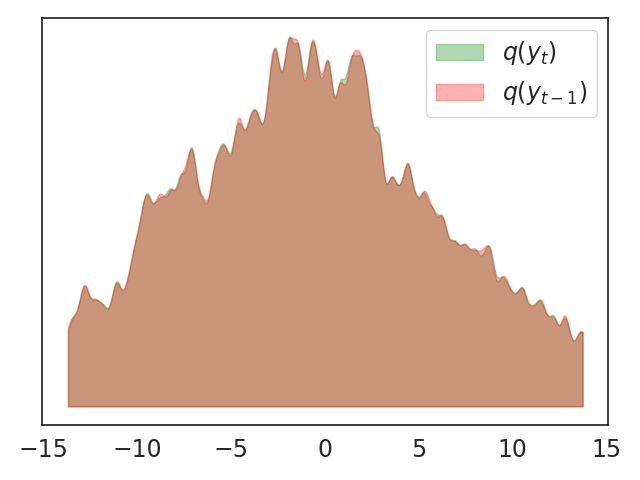}
    \end{minipage}
    \caption*{Frequency 1 (low), $t=2$}

    \vspace{5pt} %

    \begin{minipage}{0.2\textwidth}
        \centering
        \includegraphics[width=\linewidth]{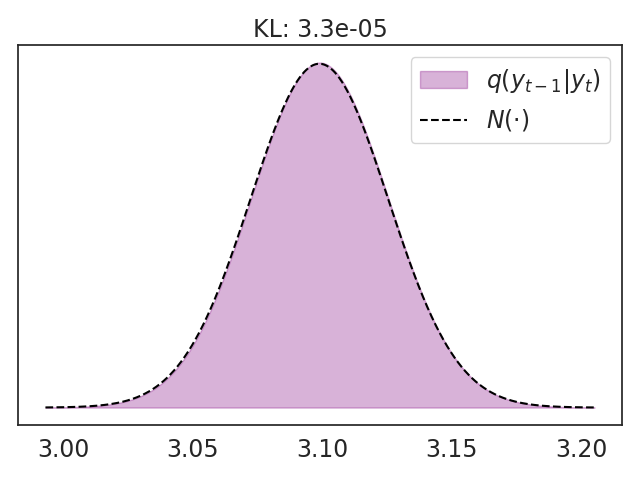}
    \end{minipage}
    \hfill
    \begin{minipage}{0.2\textwidth}
        \centering
        \includegraphics[width=\linewidth]{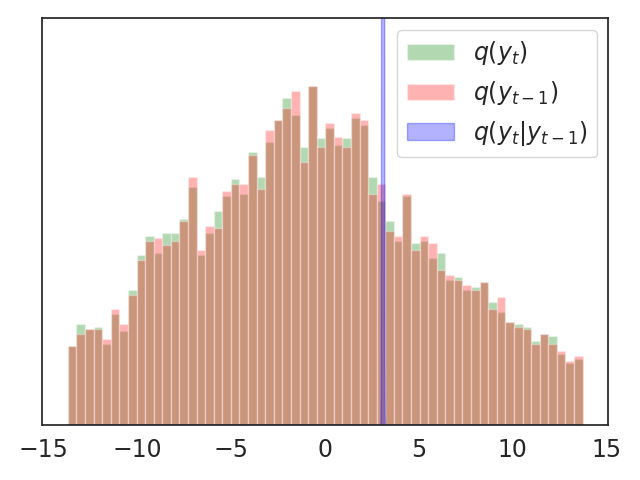}
    \end{minipage}
    \hfill
    \begin{minipage}{0.2\textwidth}
        \centering
        \includegraphics[width=\linewidth]{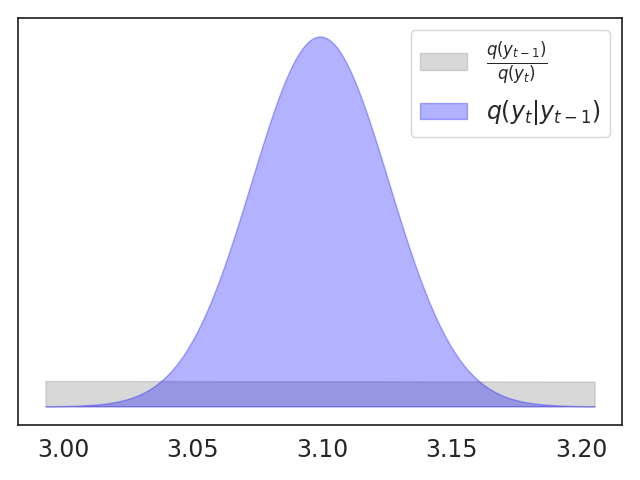}
    \end{minipage}
    \hfill
    \begin{minipage}{0.2\textwidth}
        \centering
        \includegraphics[width=\linewidth]{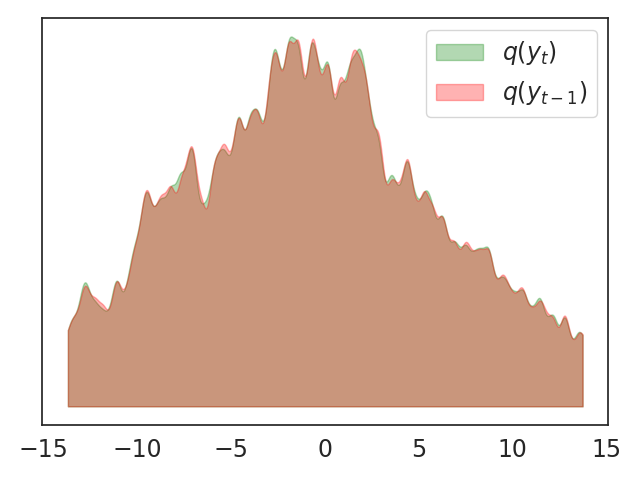}
    \end{minipage}
    \caption*{Frequency 1 (low), $t=3$}

    \vspace{5pt} %

    \begin{minipage}{0.2\textwidth}
        \centering
        \includegraphics[width=\linewidth]{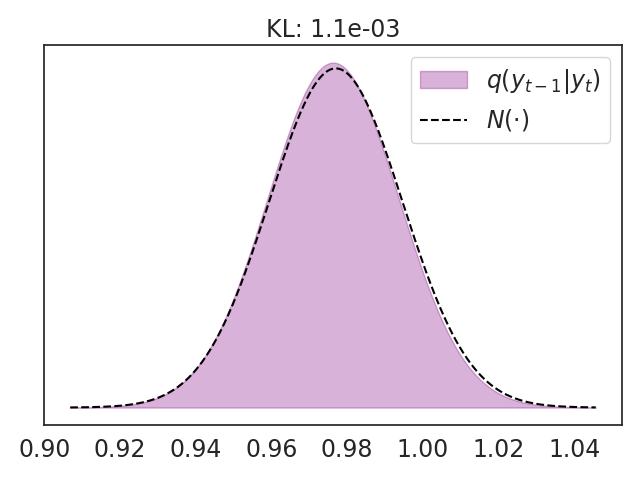}
    \end{minipage}
    \hfill
    \begin{minipage}{0.2\textwidth}
        \centering
        \includegraphics[width=\linewidth]{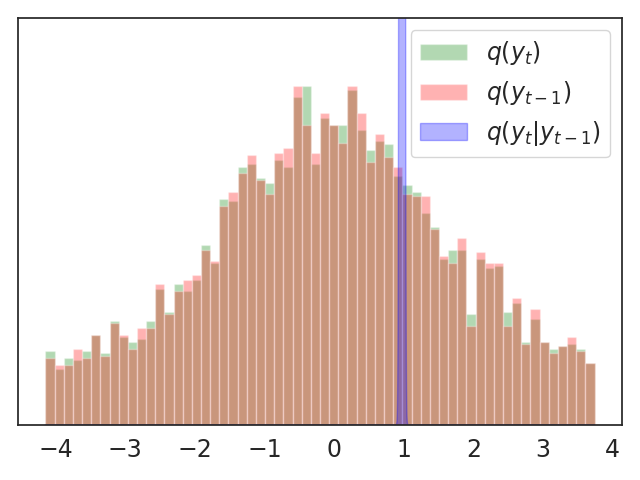}
    \end{minipage}
    \hfill
    \begin{minipage}{0.2\textwidth}
        \centering
        \includegraphics[width=\linewidth]{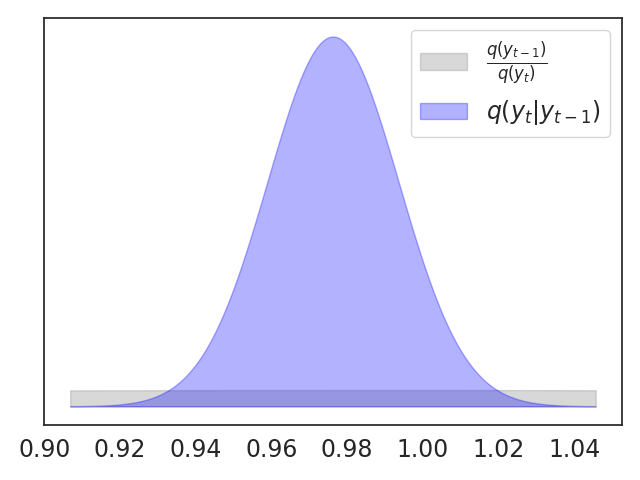}
    \end{minipage}
    \hfill
    \begin{minipage}{0.2\textwidth}
        \centering
        \includegraphics[width=\linewidth]{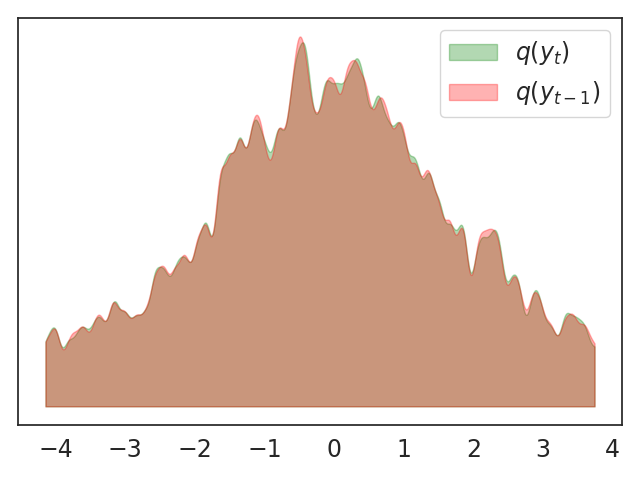}
    \end{minipage}
    \caption*{Frequency 2 (low), $t=1$}

    \vspace{5pt} %

    \begin{minipage}{0.2\textwidth}
        \centering
        \includegraphics[width=\linewidth]{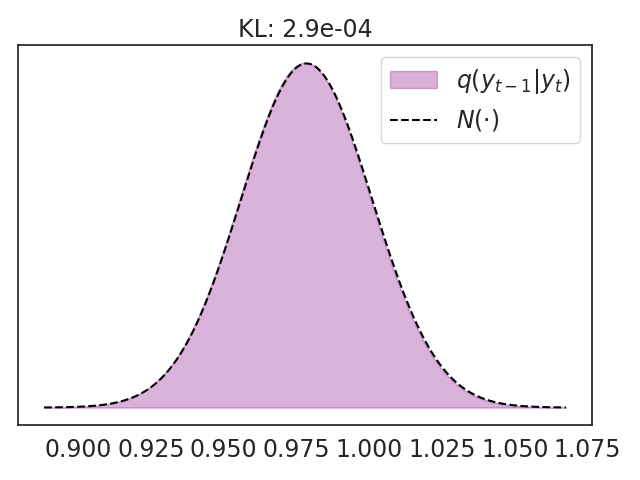}
    \end{minipage}
    \hfill
    \begin{minipage}{0.2\textwidth}
        \centering
        \includegraphics[width=\linewidth]{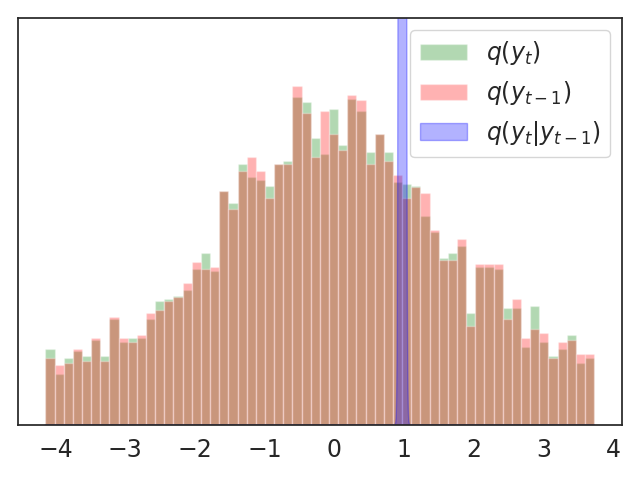}
    \end{minipage}
    \hfill
    \begin{minipage}{0.2\textwidth}
        \centering
        \includegraphics[width=\linewidth]{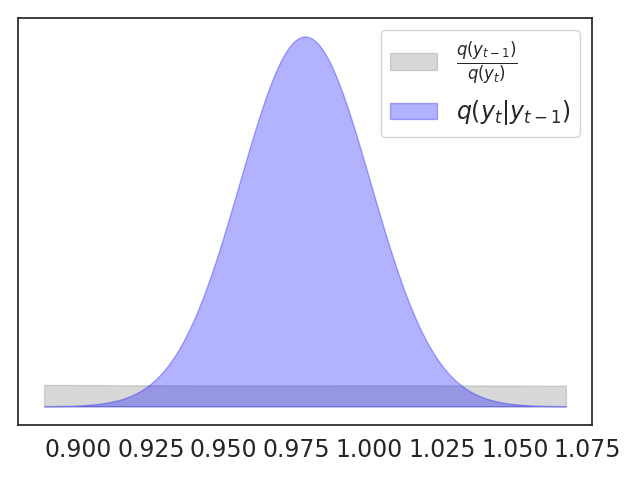}
    \end{minipage}
    \hfill
    \begin{minipage}{0.2\textwidth}
        \centering
        \includegraphics[width=\linewidth]{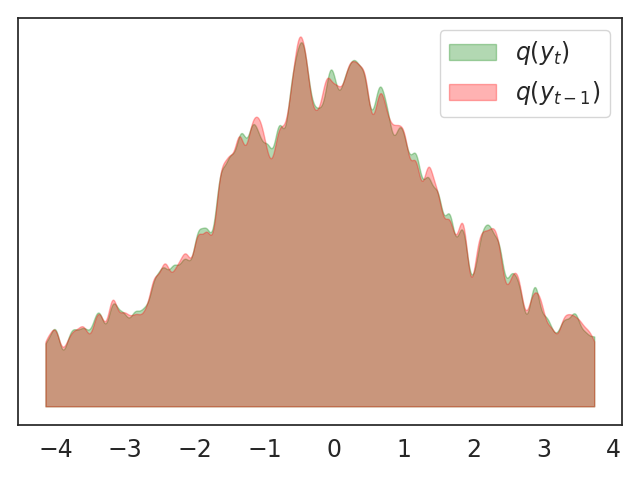}
    \end{minipage}
    \caption*{Frequency 2 (low), $t=2$}

    \vspace{5pt} %

    \begin{minipage}{0.2\textwidth}
        \centering
        \includegraphics[width=\linewidth]{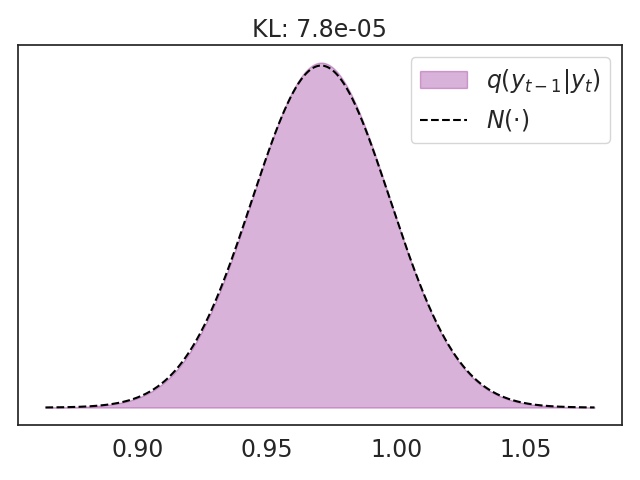}
    \end{minipage}
    \hfill
    \begin{minipage}{0.2\textwidth}
        \centering
        \includegraphics[width=\linewidth]{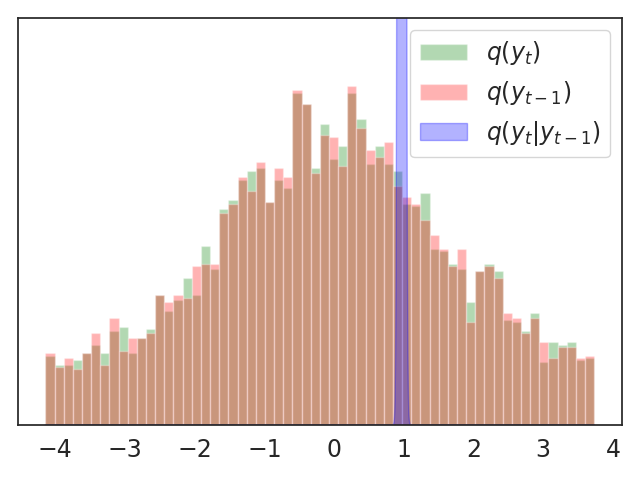}
    \end{minipage}
    \hfill
    \begin{minipage}{0.2\textwidth}
        \centering
        \includegraphics[width=\linewidth]{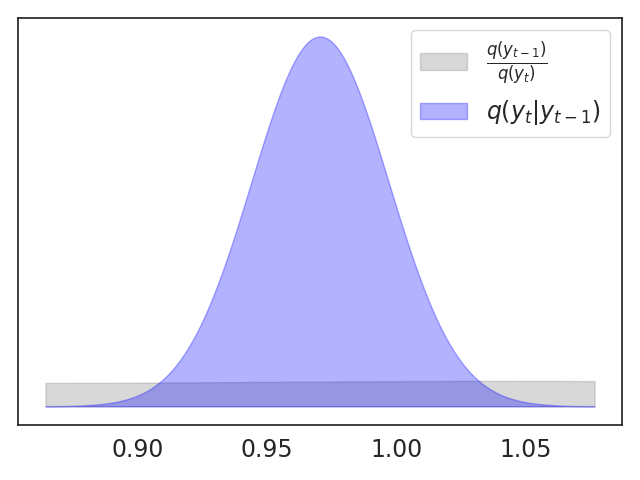}
    \end{minipage}
    \hfill
    \begin{minipage}{0.2\textwidth}
        \centering
        \includegraphics[width=\linewidth]{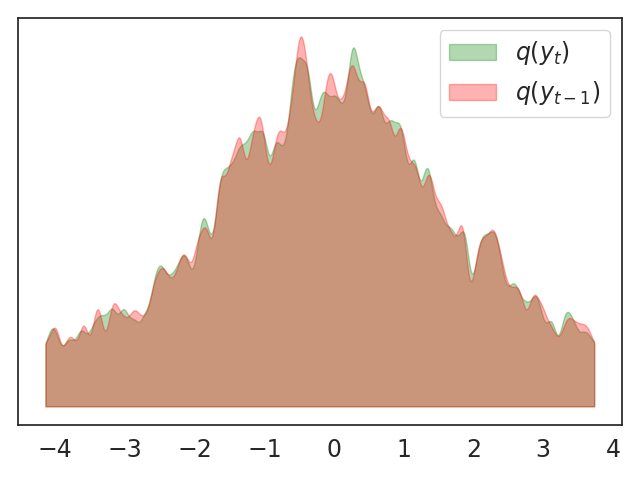}
    \end{minipage}
    \caption*{Frequency 2 (low), $t=3$}

\caption{Analysis of violations of the Gaussian assumption in DDPM (1 of 3).}
\label{fig:gauss_app_start}
\end{figure}

\begin{figure}[p]
    \centering 
    
    \begin{minipage}{0.2\textwidth}
        \centering
        \includegraphics[width=\linewidth]{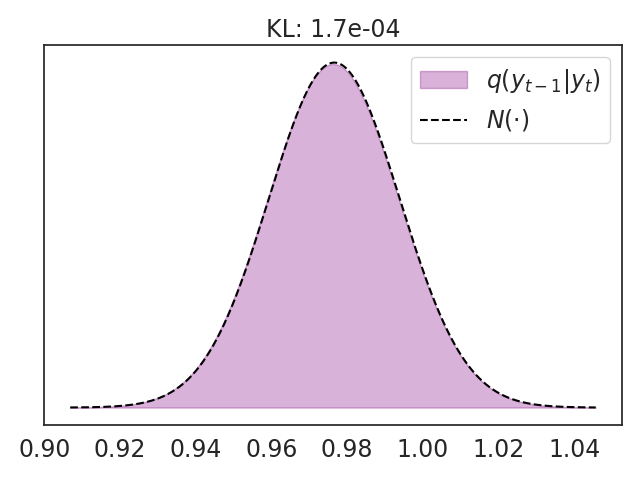}
    \end{minipage}
    \hfill
    \begin{minipage}{0.2\textwidth}
        \centering
        \includegraphics[width=\linewidth]{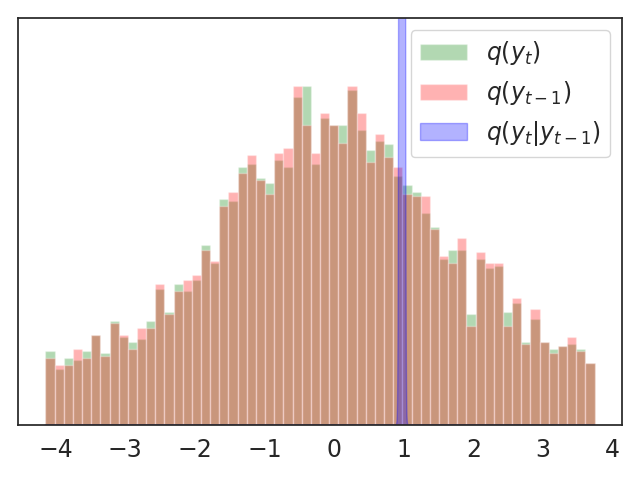}
    \end{minipage}
    \hfill
    \begin{minipage}{0.2\textwidth}
        \centering
        \includegraphics[width=\linewidth]{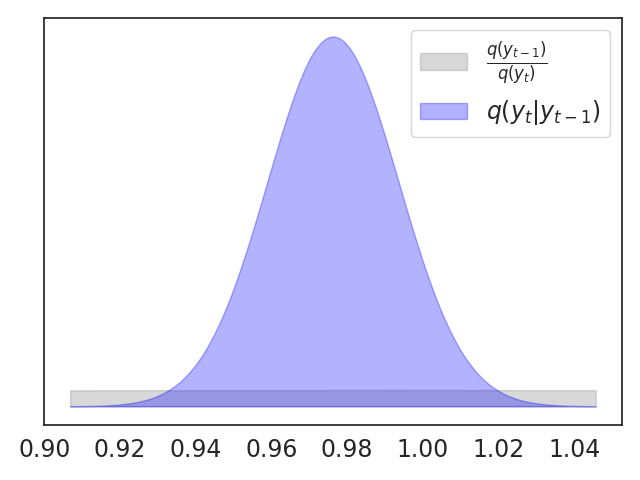}
    \end{minipage}
    \hfill
    \begin{minipage}{0.2\textwidth}
        \centering
        \includegraphics[width=\linewidth]{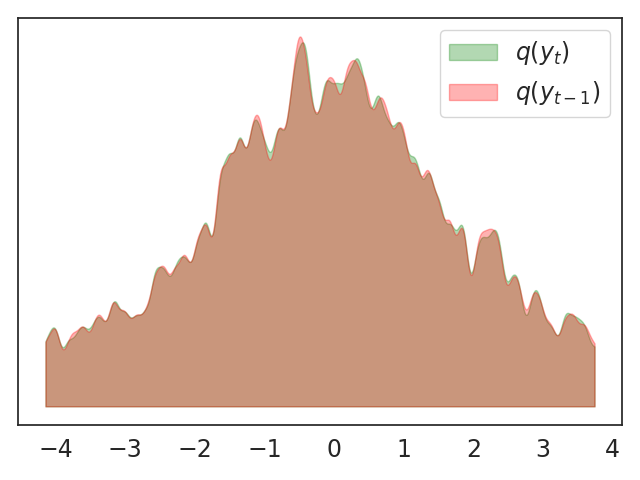}
    \end{minipage}
    \caption*{Frequency 3 (low), $t=1$}

    \vspace{5pt} %

    \begin{minipage}{0.2\textwidth}
        \centering
        \includegraphics[width=\linewidth]{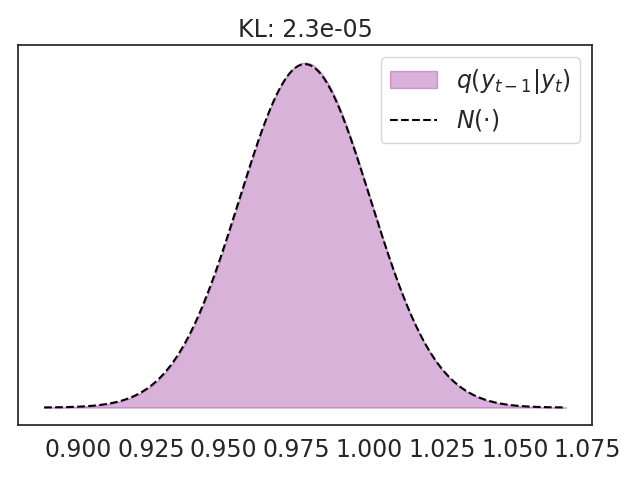}
    \end{minipage}
    \hfill
    \begin{minipage}{0.2\textwidth}
        \centering
        \includegraphics[width=\linewidth]{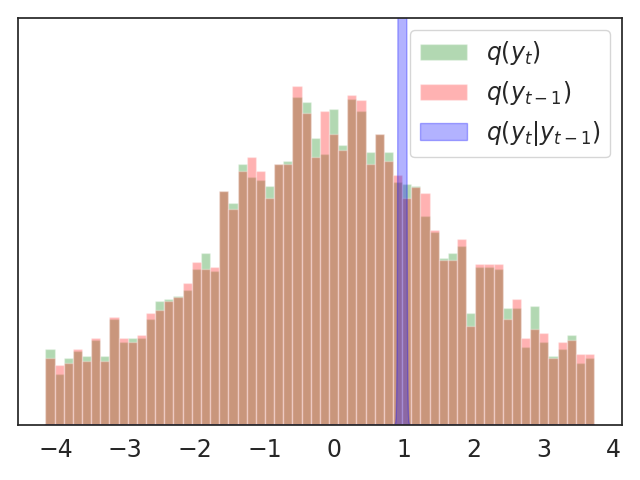}
    \end{minipage}
    \hfill
    \begin{minipage}{0.2\textwidth}
        \centering
        \includegraphics[width=\linewidth]{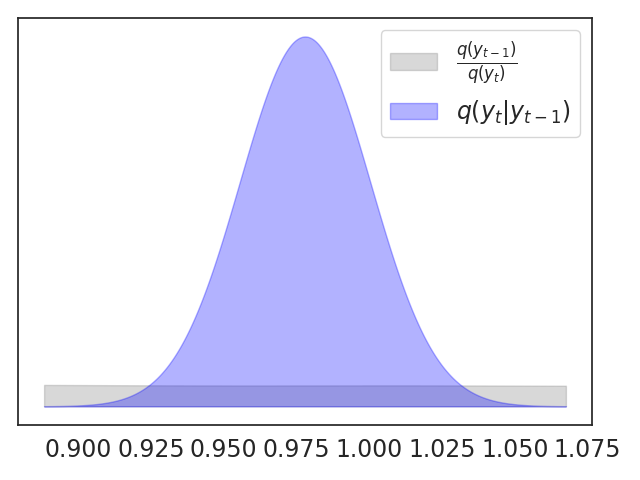}
    \end{minipage}
    \hfill
    \begin{minipage}{0.2\textwidth}
        \centering
        \includegraphics[width=\linewidth]{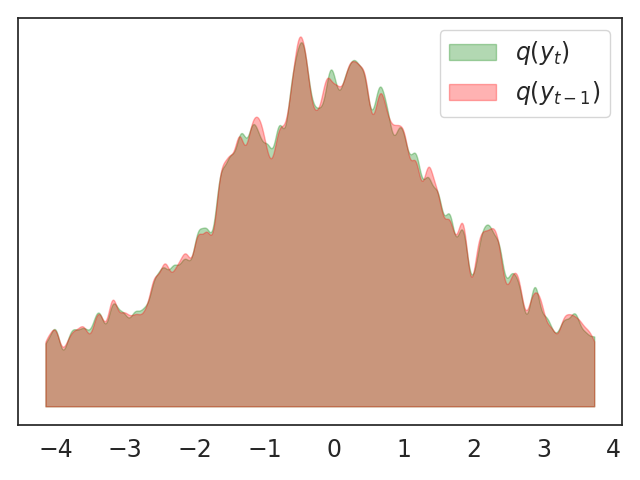}
    \end{minipage}
    \caption*{Frequency 3 (low), $t=2$}

    \vspace{5pt} %

    \begin{minipage}{0.2\textwidth}
        \centering
        \includegraphics[width=\linewidth]{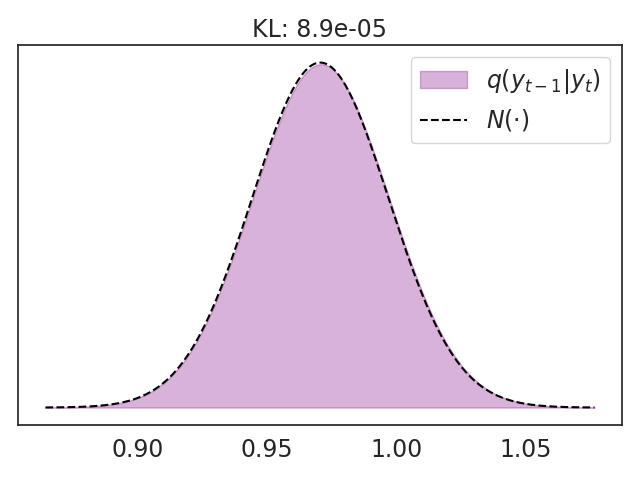}
    \end{minipage}
    \hfill
    \begin{minipage}{0.2\textwidth}
        \centering
        \includegraphics[width=\linewidth]{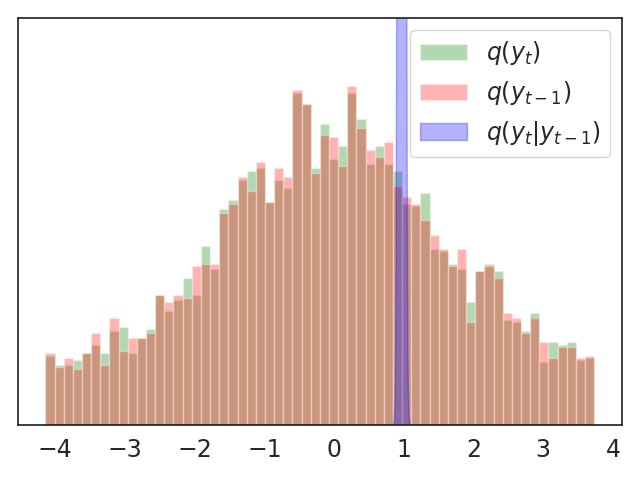}
    \end{minipage}
    \hfill
    \begin{minipage}{0.2\textwidth}
        \centering
        \includegraphics[width=\linewidth]{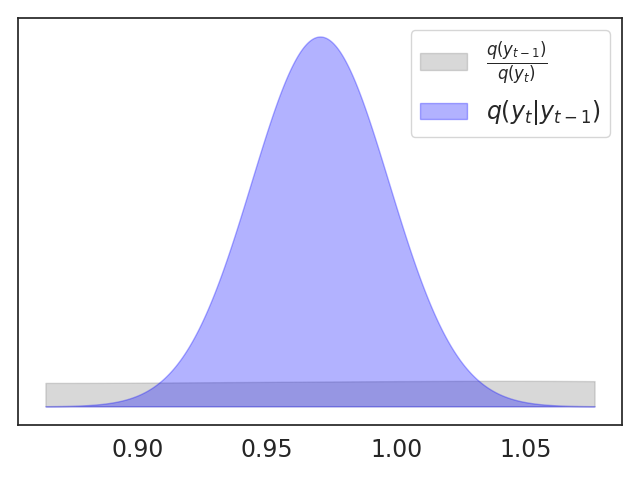}
    \end{minipage}
    \hfill
    \begin{minipage}{0.2\textwidth}
        \centering
        \includegraphics[width=\linewidth]{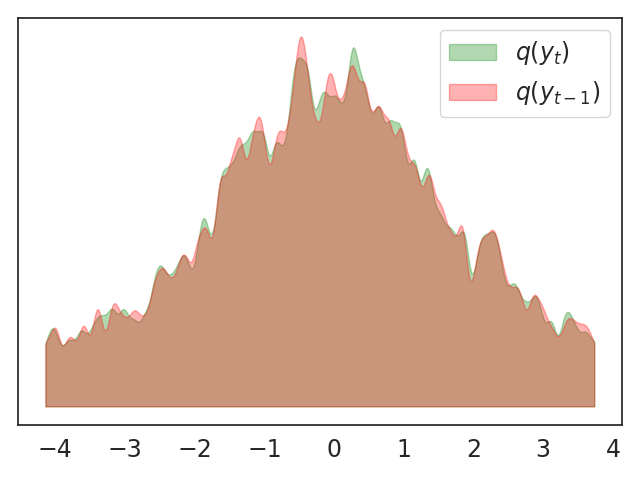}
    \end{minipage}
    \caption*{Frequency 3 (low), $t=3$}

    \vspace{5pt} %

    \begin{minipage}{0.2\textwidth}
        \centering
        \includegraphics[width=\linewidth]{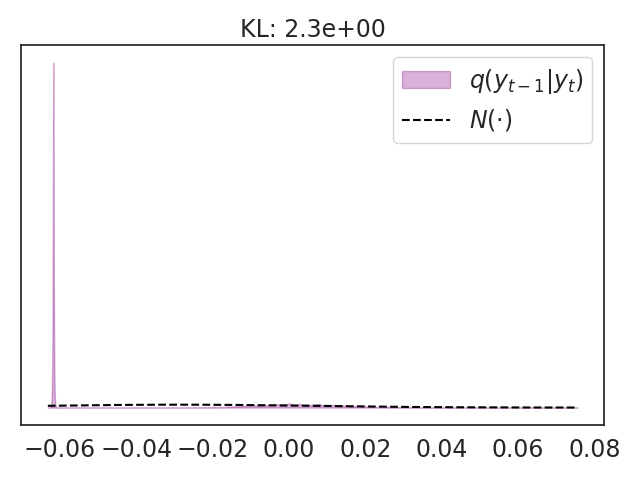}
    \end{minipage}
    \hfill
    \begin{minipage}{0.2\textwidth}
        \centering
        \includegraphics[width=\linewidth]{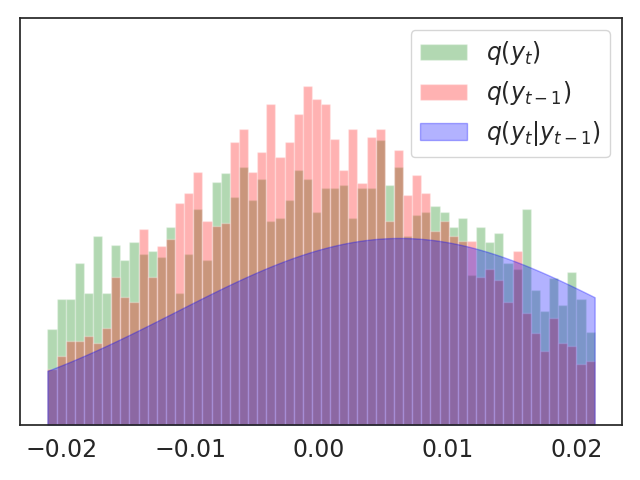}
    \end{minipage}
    \hfill
    \begin{minipage}{0.2\textwidth}
        \centering
        \includegraphics[width=\linewidth]{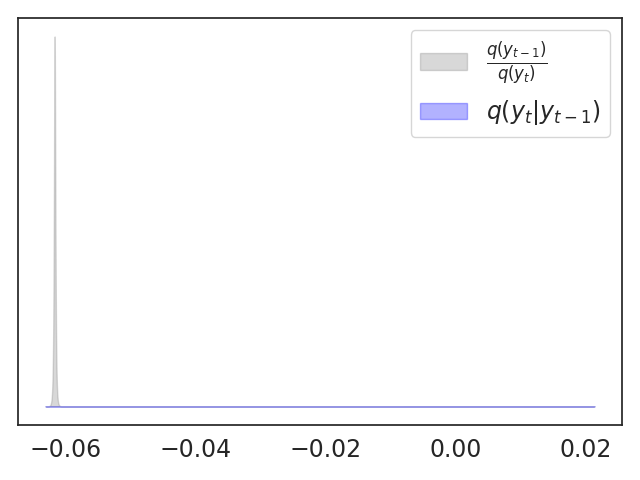}
    \end{minipage}
    \hfill
    \begin{minipage}{0.2\textwidth}
        \centering
        \includegraphics[width=\linewidth]{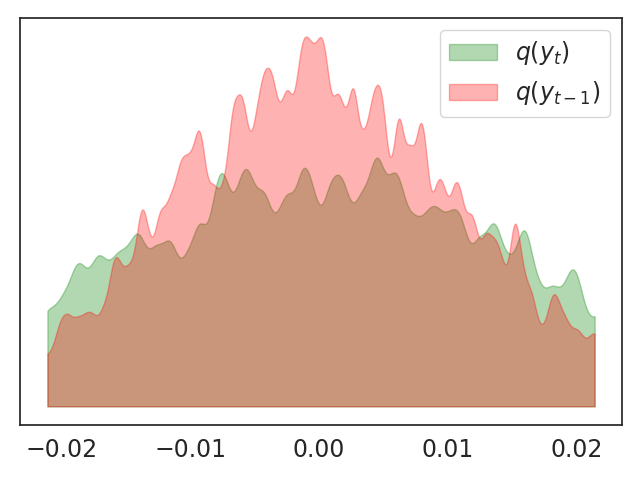}
    \end{minipage}
    \caption*{Frequency 1022 (high), $t=1$}

    \vspace{5pt} %

    \begin{minipage}{0.2\textwidth}
        \centering
        \includegraphics[width=\linewidth]{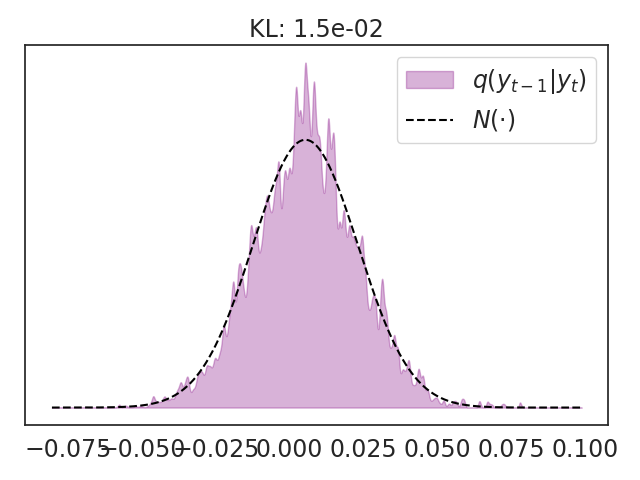}
    \end{minipage}
    \hfill
    \begin{minipage}{0.2\textwidth}
        \centering
        \includegraphics[width=\linewidth]{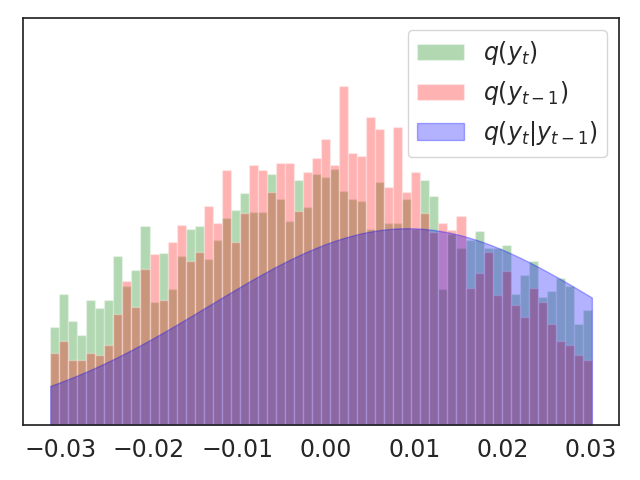}
    \end{minipage}
    \hfill
    \begin{minipage}{0.2\textwidth}
        \centering
        \includegraphics[width=\linewidth]{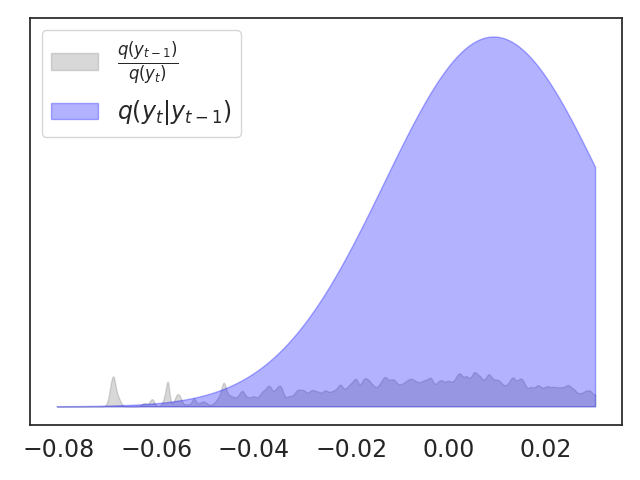}
    \end{minipage}
    \hfill
    \begin{minipage}{0.2\textwidth}
        \centering
        \includegraphics[width=\linewidth]{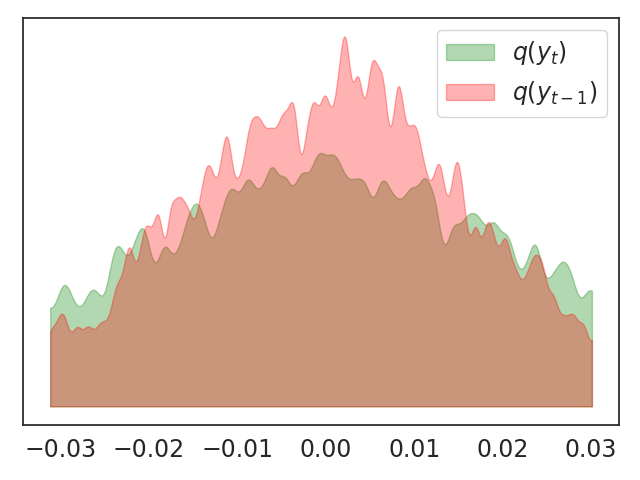}
    \end{minipage}
    \caption*{Frequency 1022 (high), $t=2$}

    \vspace{5pt} %

    \begin{minipage}{0.2\textwidth}
        \centering
        \includegraphics[width=\linewidth]{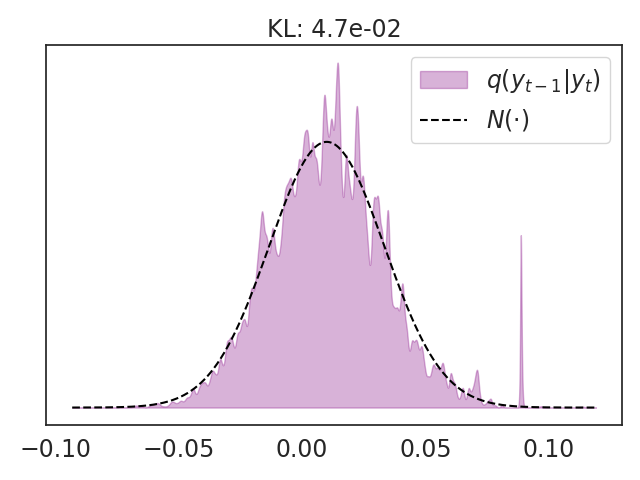}
    \end{minipage}
    \hfill
    \begin{minipage}{0.2\textwidth}
        \centering
        \includegraphics[width=\linewidth]{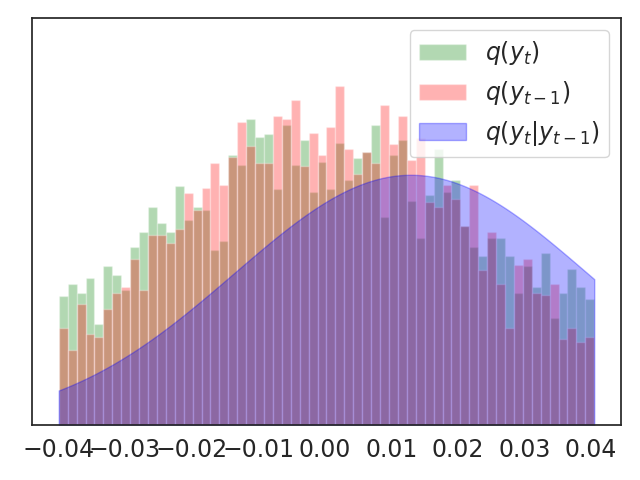}
    \end{minipage}
    \hfill
    \begin{minipage}{0.2\textwidth}
        \centering
        \includegraphics[width=\linewidth]{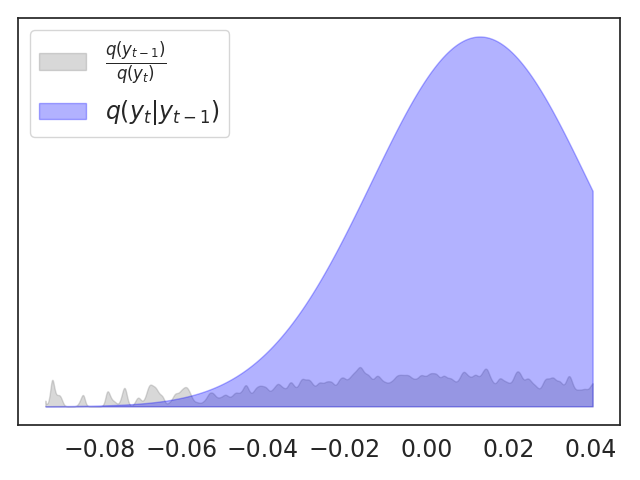}
    \end{minipage}
    \hfill
    \begin{minipage}{0.2\textwidth}
        \centering
        \includegraphics[width=\linewidth]{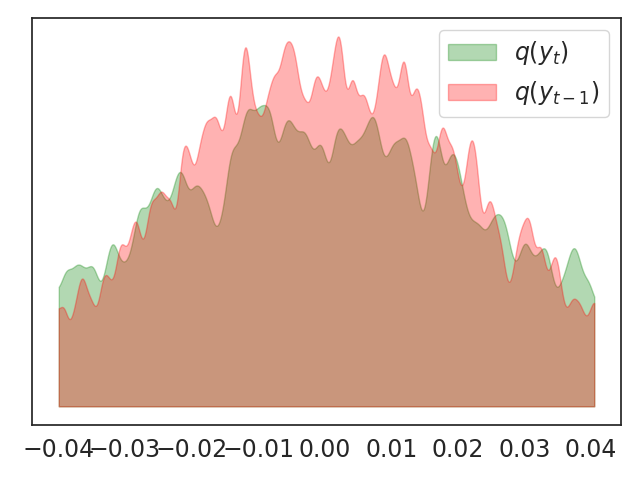}
    \end{minipage}
    \caption*{Frequency 1022 (high), $t=3$}

\caption{Analysis of violations of the Gaussian assumption in DDPM (2 of 3).}
\end{figure}

\begin{figure}[p]
    \centering 
    
    \begin{minipage}{0.2\textwidth}
        \centering
        \includegraphics[width=\linewidth]{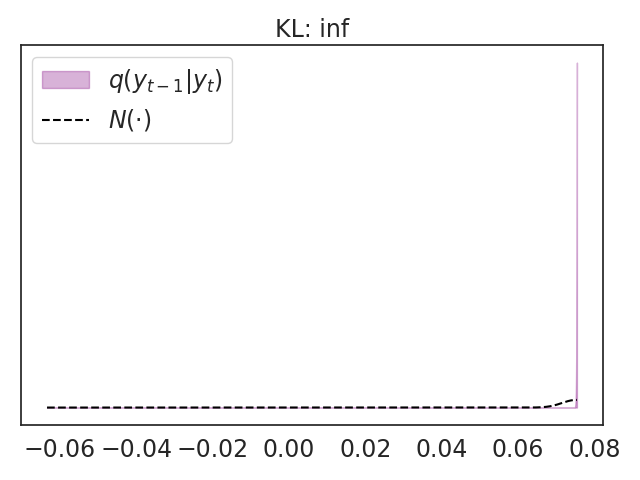}
    \end{minipage}
    \hfill
    \begin{minipage}{0.2\textwidth}
        \centering
        \includegraphics[width=\linewidth]{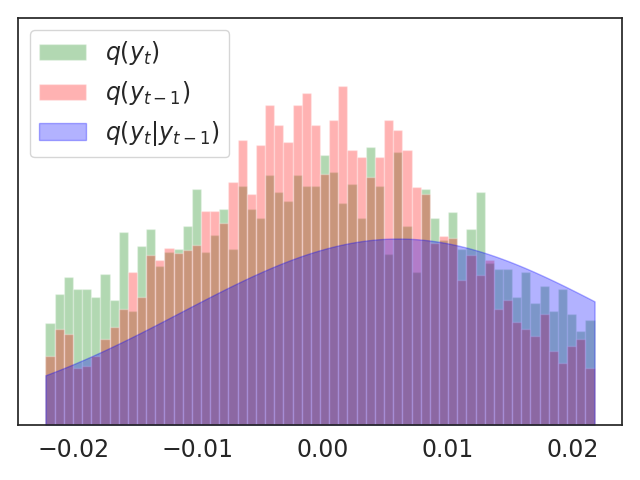}
    \end{minipage}
    \hfill
    \begin{minipage}{0.2\textwidth}
        \centering
        \includegraphics[width=\linewidth]{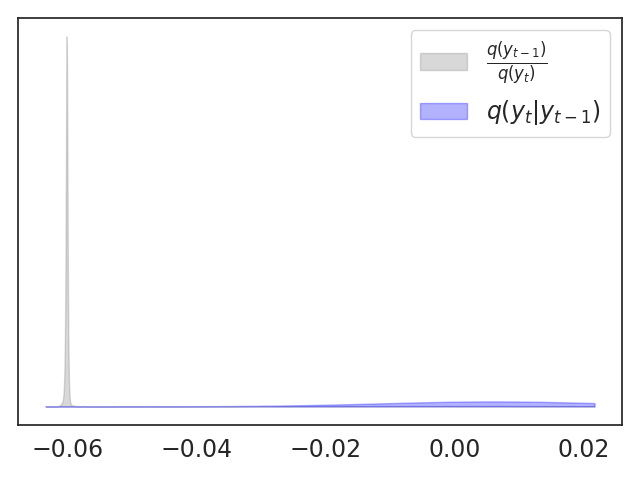}
    \end{minipage}
    \hfill
    \begin{minipage}{0.2\textwidth}
        \centering
        \includegraphics[width=\linewidth]{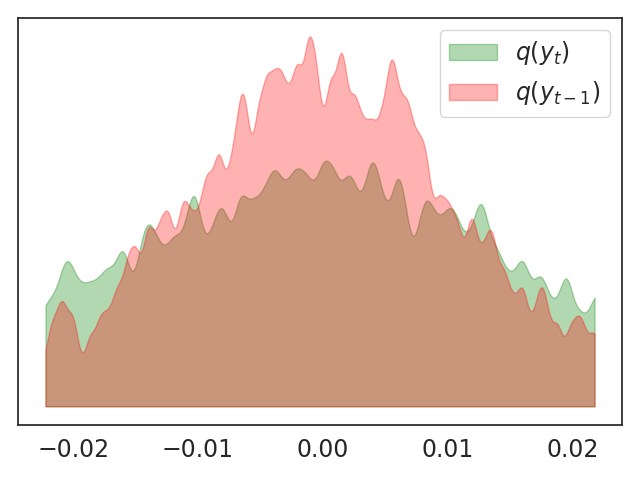}
    \end{minipage}
    \caption*{Frequency 1023 (high), $t=1$}

    \vspace{5pt} %

    \begin{minipage}{0.2\textwidth}
        \centering
        \includegraphics[width=\linewidth]{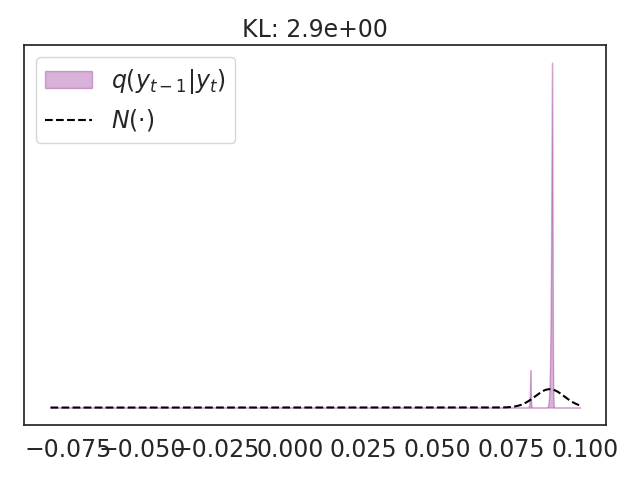}
    \end{minipage}
    \hfill
    \begin{minipage}{0.2\textwidth}
        \centering
        \includegraphics[width=\linewidth]{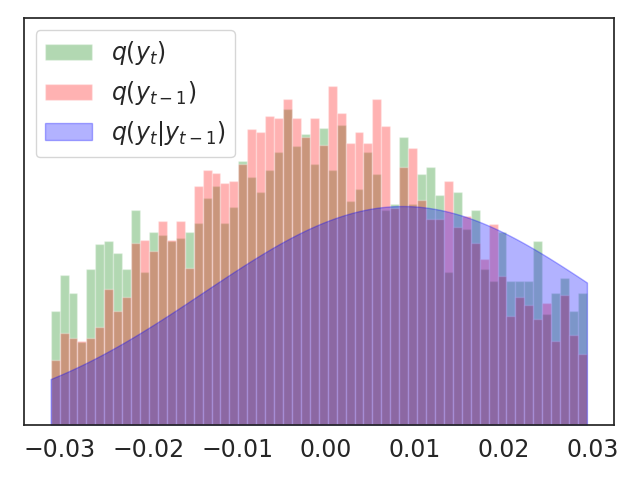}
    \end{minipage}
    \hfill
    \begin{minipage}{0.2\textwidth}
        \centering
        \includegraphics[width=\linewidth]{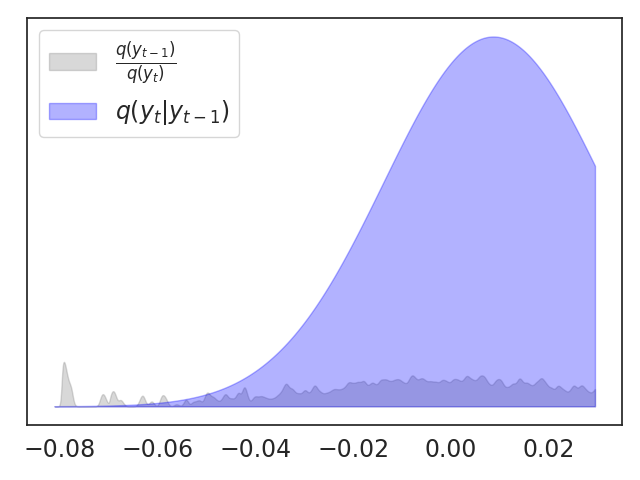}
    \end{minipage}
    \hfill
    \begin{minipage}{0.2\textwidth}
        \centering
        \includegraphics[width=\linewidth]{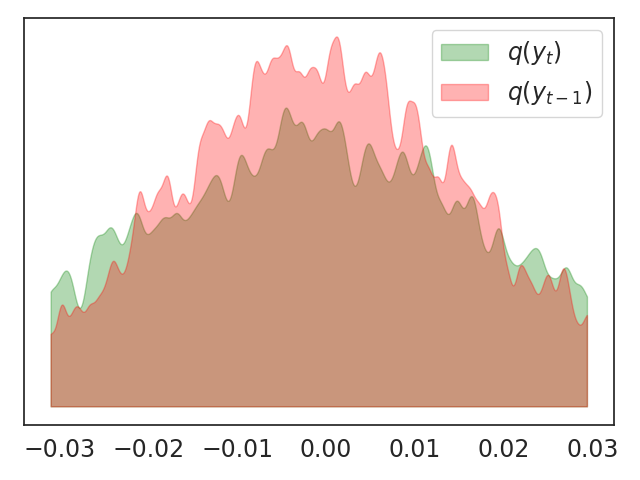}
    \end{minipage}
    \caption*{Frequency 1023 (high), $t=2$}

    \vspace{5pt} %

    \begin{minipage}{0.2\textwidth}
        \centering
        \includegraphics[width=\linewidth]{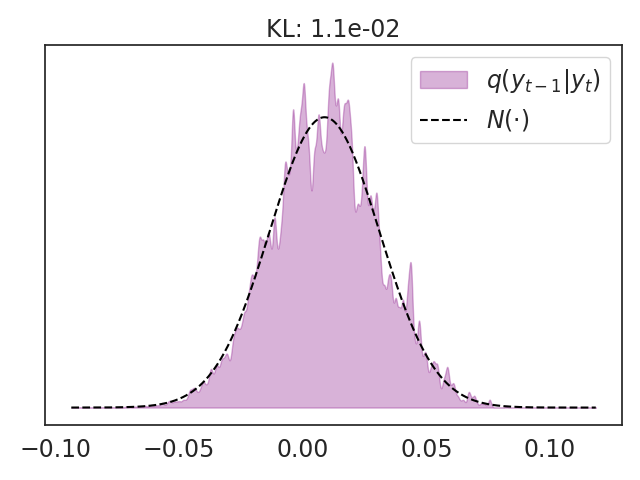}
    \end{minipage}
    \hfill
    \begin{minipage}{0.2\textwidth}
        \centering
        \includegraphics[width=\linewidth]{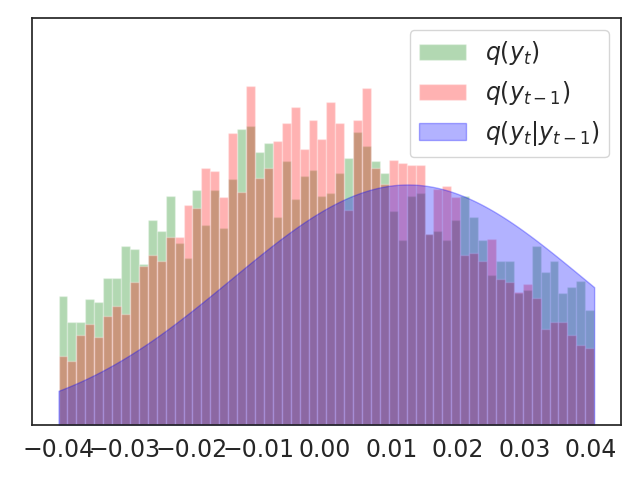}
    \end{minipage}
    \hfill
    \begin{minipage}{0.2\textwidth}
        \centering
        \includegraphics[width=\linewidth]{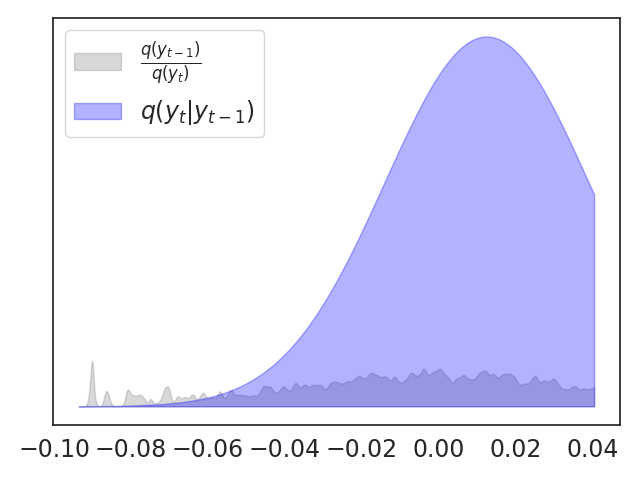}
    \end{minipage}
    \hfill
    \begin{minipage}{0.2\textwidth}
        \centering
        \includegraphics[width=\linewidth]{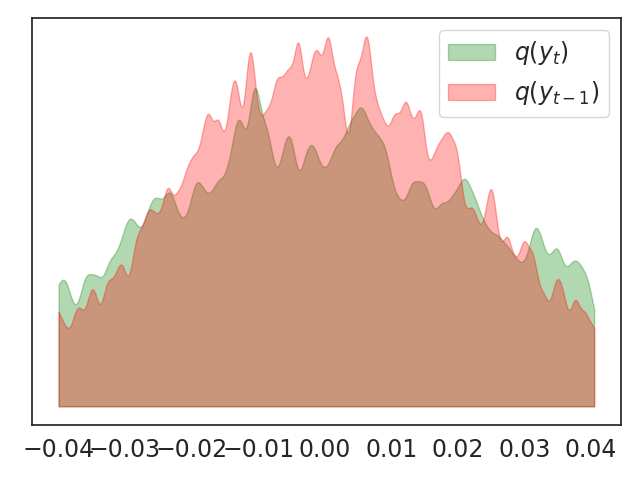}
    \end{minipage}
    \caption*{Frequency 1023 (high), $t=3$}

    \vspace{5pt} %

    \begin{minipage}{0.2\textwidth}
        \centering
        \includegraphics[width=\linewidth]{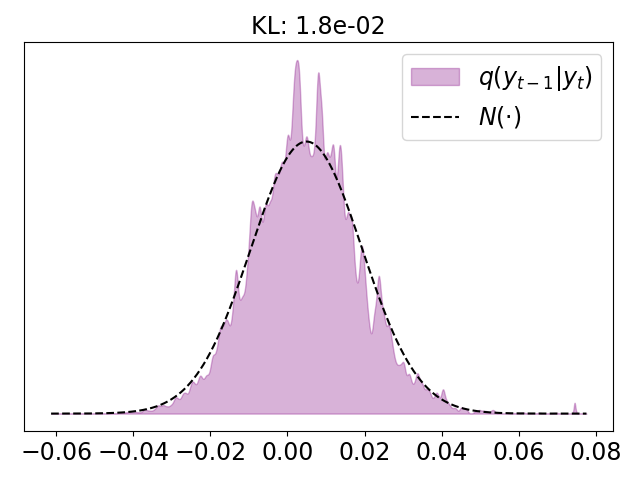}
    \end{minipage}
    \hfill
    \begin{minipage}{0.2\textwidth}
        \centering
        \includegraphics[width=\linewidth]{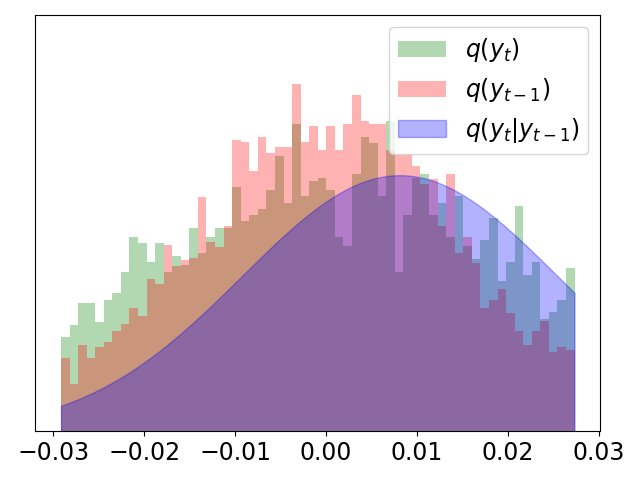}
    \end{minipage}
    \hfill
    \begin{minipage}{0.2\textwidth}
        \centering
        \includegraphics[width=\linewidth]{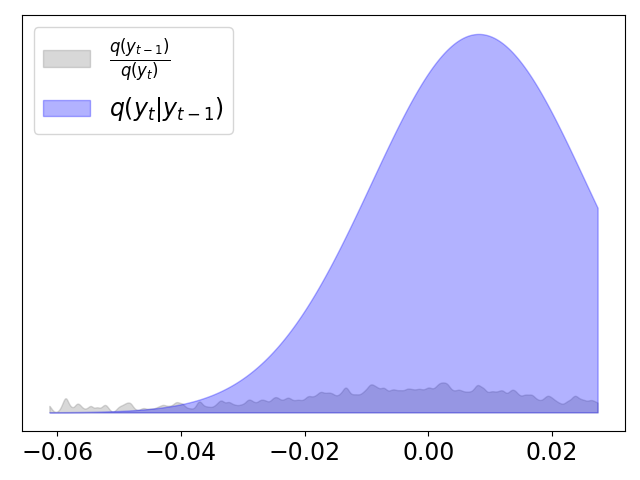}
    \end{minipage}
    \hfill
    \begin{minipage}{0.2\textwidth}
        \centering
        \includegraphics[width=\linewidth]{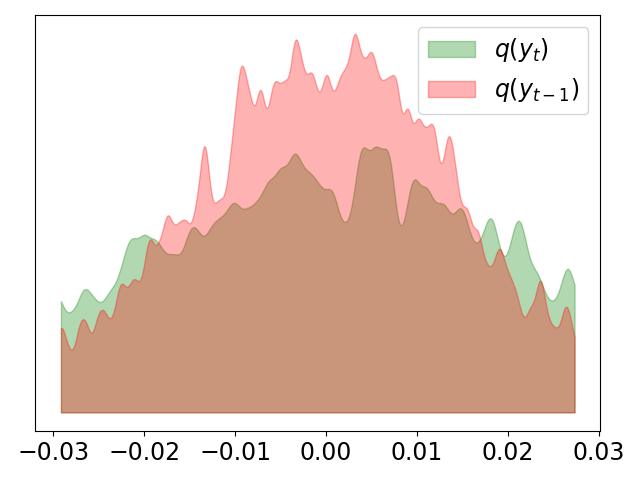}
    \end{minipage}
    \caption*{Frequency 1024 (high), $t=1$}

    \vspace{5pt} %

    \begin{minipage}{0.2\textwidth}
        \centering
        \includegraphics[width=\linewidth]{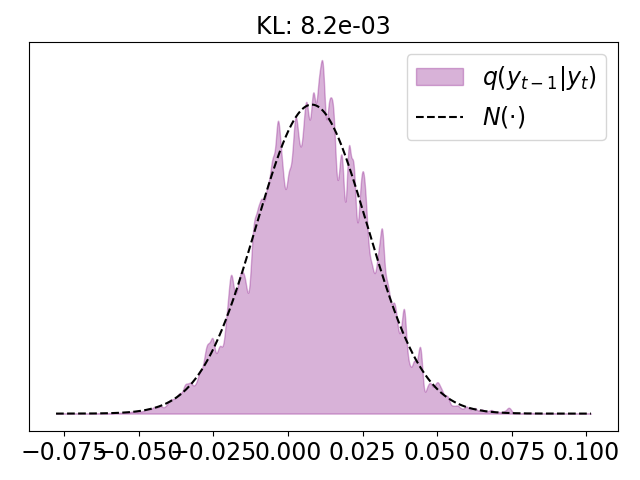}
    \end{minipage}
    \hfill
    \begin{minipage}{0.2\textwidth}
        \centering
        \includegraphics[width=\linewidth]{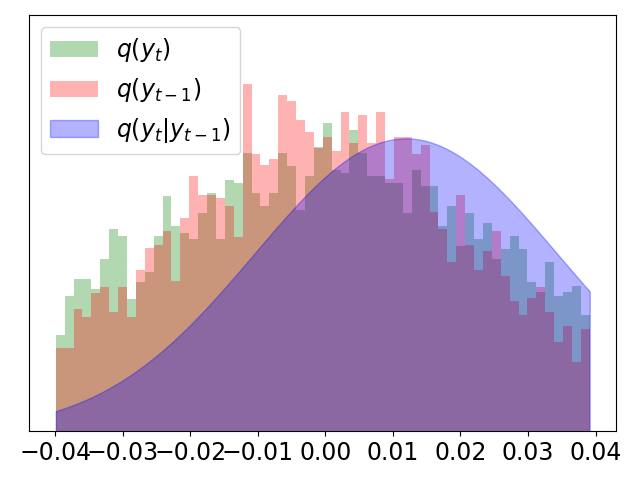}
    \end{minipage}
    \hfill
    \begin{minipage}{0.2\textwidth}
        \centering
        \includegraphics[width=\linewidth]{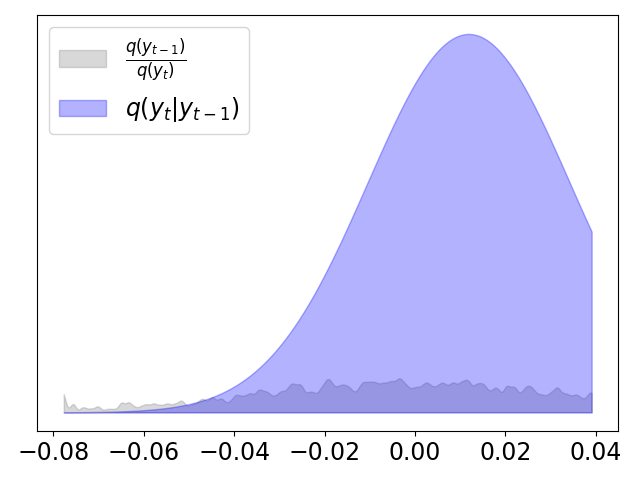}
    \end{minipage}
    \hfill
    \begin{minipage}{0.2\textwidth}
        \centering
        \includegraphics[width=\linewidth]{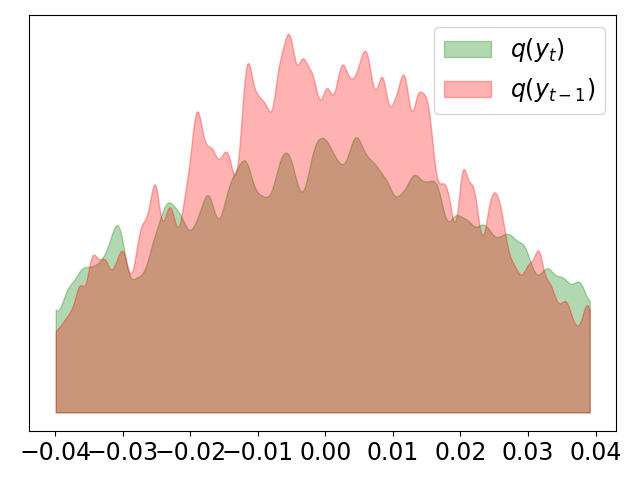}
    \end{minipage}
    \caption*{Frequency 1024 (high), $t=2$}

    \vspace{5pt} %

    \begin{minipage}{0.2\textwidth}
        \centering
        \includegraphics[width=\linewidth]{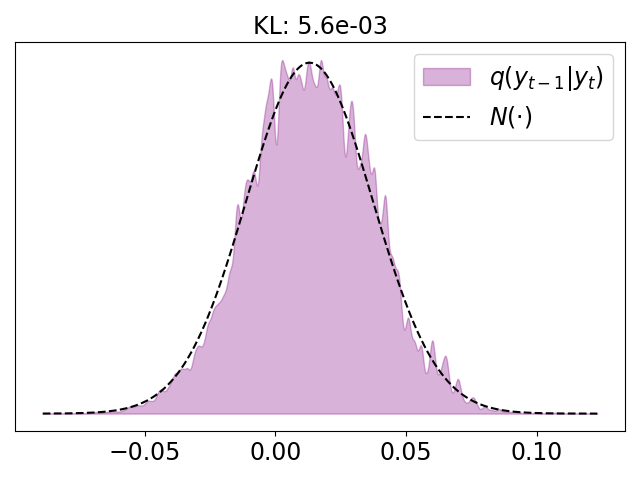}
    \end{minipage}
    \hfill
    \begin{minipage}{0.2\textwidth}
        \centering
        \includegraphics[width=\linewidth]{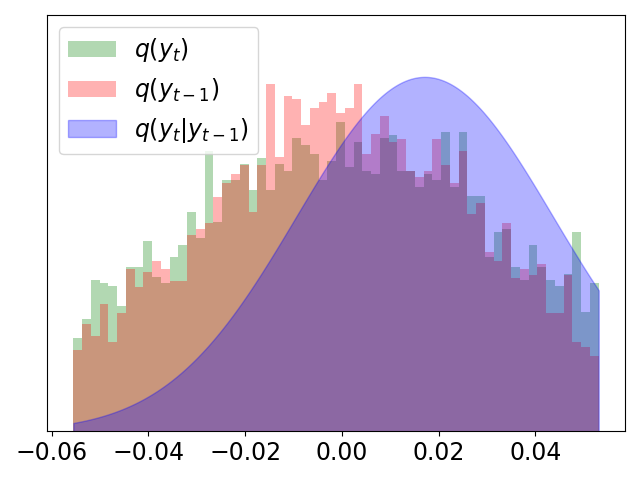}
    \end{minipage}
    \hfill
    \begin{minipage}{0.2\textwidth}
        \centering
        \includegraphics[width=\linewidth]{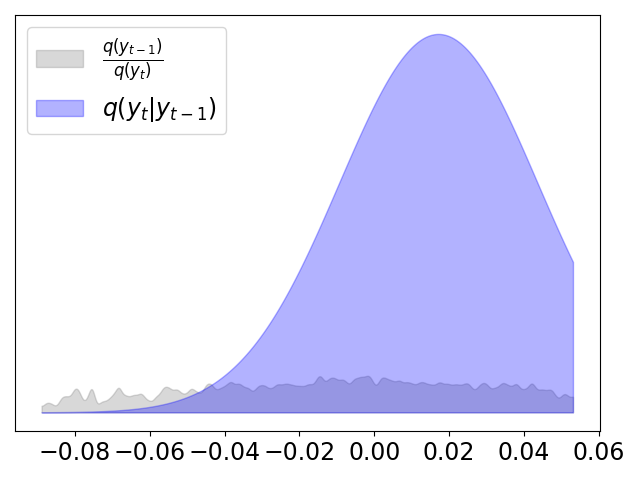}
    \end{minipage}
    \hfill
    \begin{minipage}{0.2\textwidth}
        \centering
        \includegraphics[width=\linewidth]{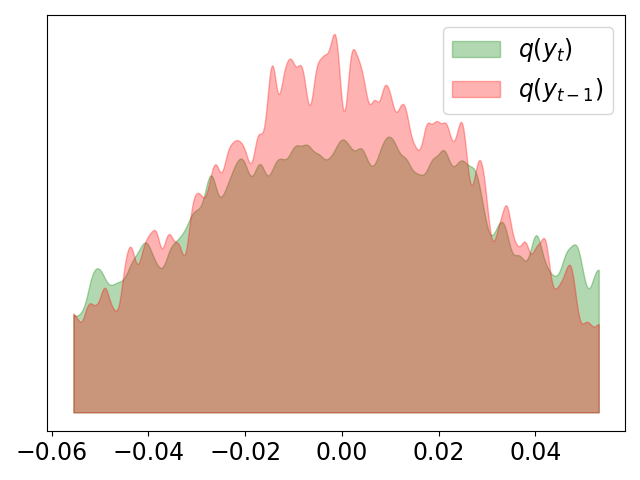}
    \end{minipage}
    \caption*{Frequency 1024 (high), $t=3$}

\caption{Analysis of violations of the Gaussian assumption in DDPM (3 of 3).}
\end{figure}

\begin{figure}[p]
    \centering 
    
    \begin{minipage}{0.2\textwidth}
        \centering
        \includegraphics[width=\linewidth]{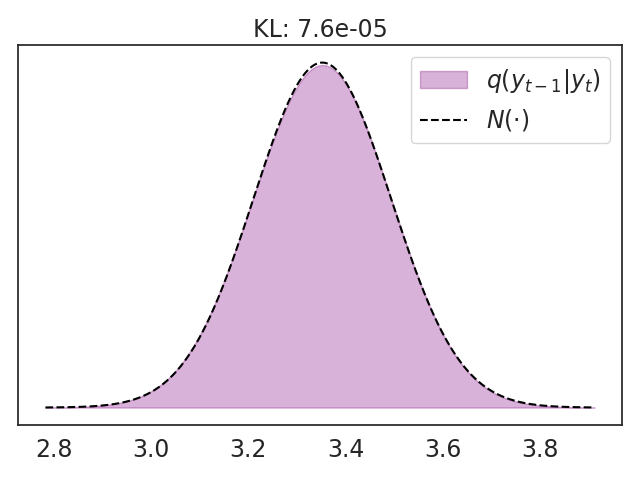}
    \end{minipage}
    \hfill
    \begin{minipage}{0.2\textwidth}
        \centering
        \includegraphics[width=\linewidth]{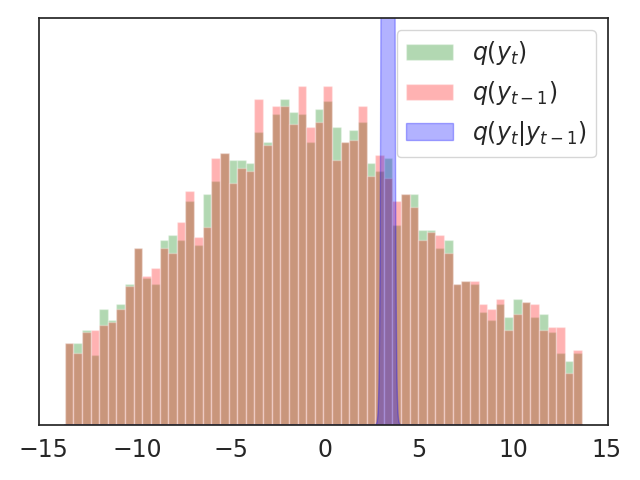}
    \end{minipage}
    \hfill
    \begin{minipage}{0.2\textwidth}
        \centering
        \includegraphics[width=\linewidth]{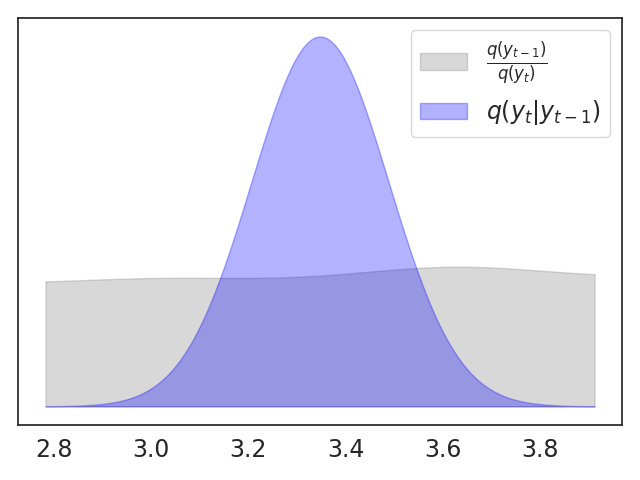}
    \end{minipage}
    \hfill
    \begin{minipage}{0.2\textwidth}
        \centering
        \includegraphics[width=\linewidth]{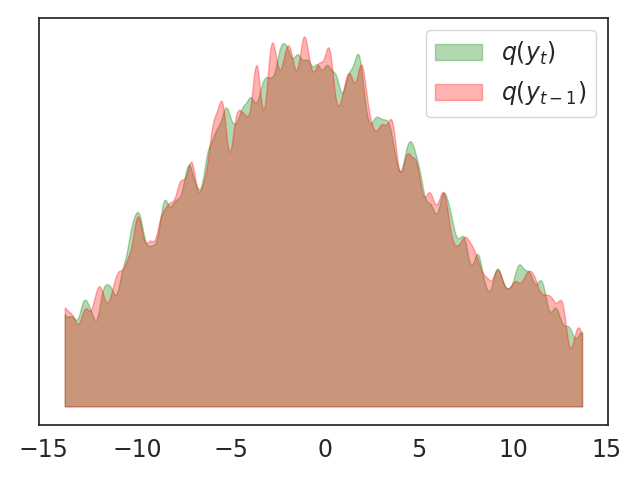}
    \end{minipage}
    \caption*{Frequency 1 (low), $t=1$}

    \vspace{5pt} %

    \begin{minipage}{0.2\textwidth}
        \centering
        \includegraphics[width=\linewidth]{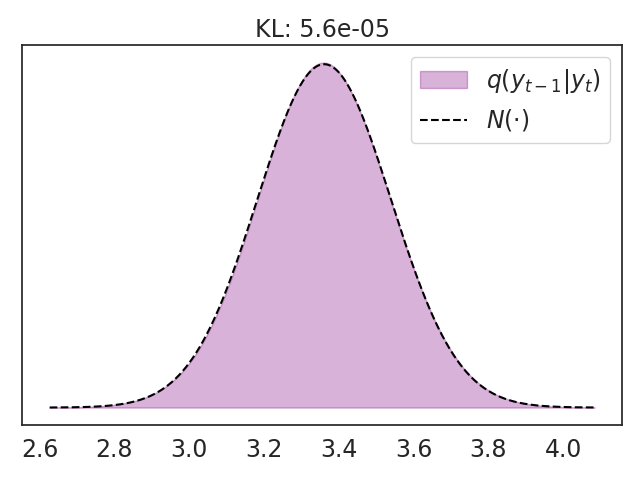}
    \end{minipage}
    \hfill
    \begin{minipage}{0.2\textwidth}
        \centering
        \includegraphics[width=\linewidth]{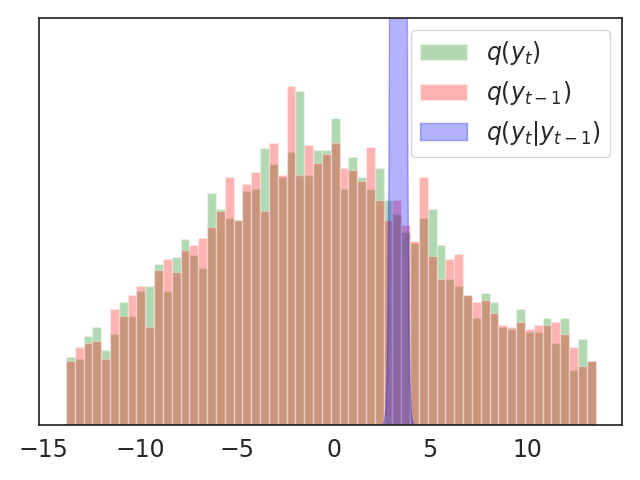}
    \end{minipage}
    \hfill
    \begin{minipage}{0.2\textwidth}
        \centering
        \includegraphics[width=\linewidth]{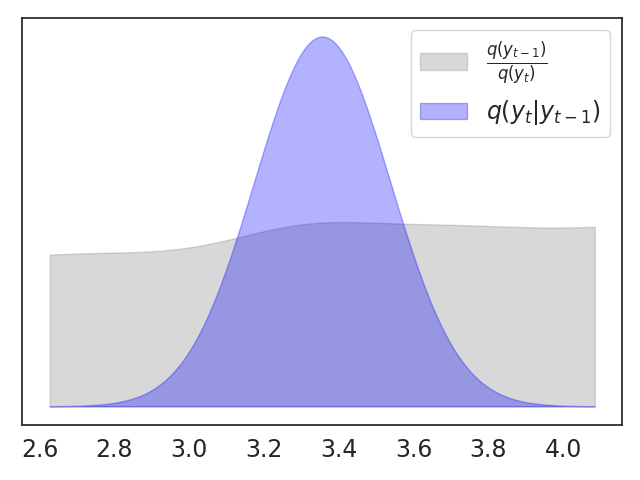}
    \end{minipage}
    \hfill
    \begin{minipage}{0.2\textwidth}
        \centering
        \includegraphics[width=\linewidth]{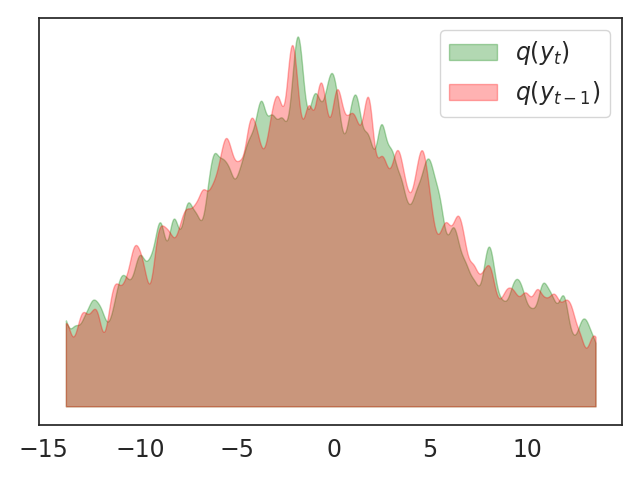}
    \end{minipage}
    \caption*{Frequency 1 (low), $t=2$}

    \vspace{5pt} %

    \begin{minipage}{0.2\textwidth}
        \centering
        \includegraphics[width=\linewidth]{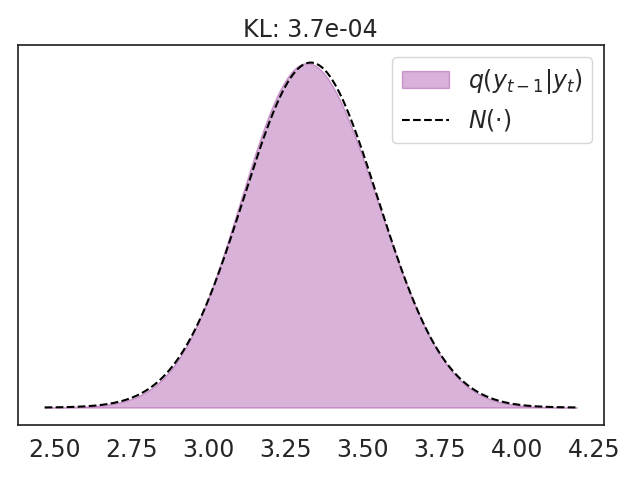}
    \end{minipage}
    \hfill
    \begin{minipage}{0.2\textwidth}
        \centering
        \includegraphics[width=\linewidth]{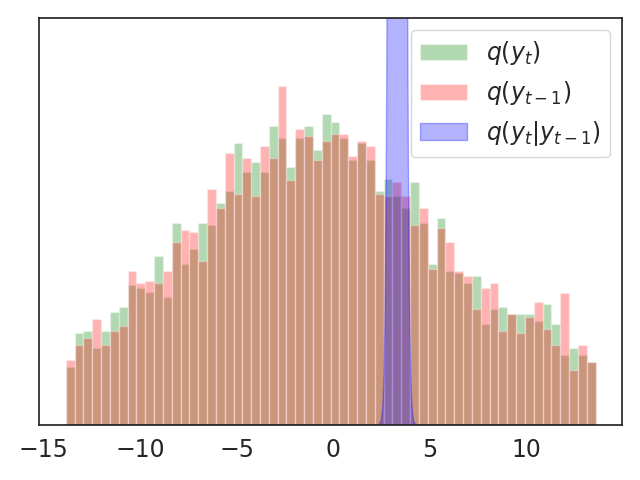}
    \end{minipage}
    \hfill
    \begin{minipage}{0.2\textwidth}
        \centering
        \includegraphics[width=\linewidth]{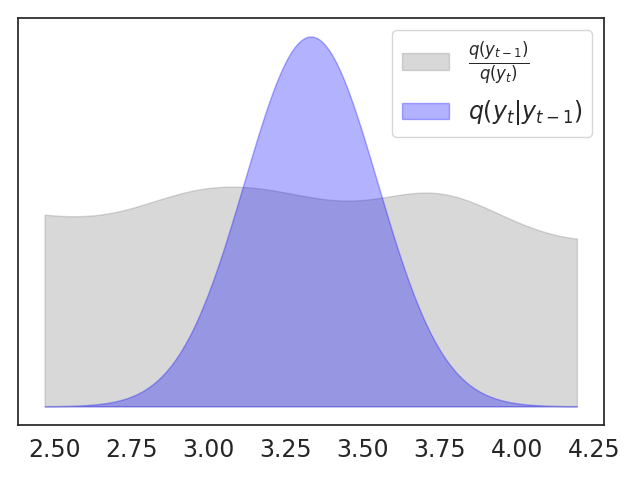}
    \end{minipage}
    \hfill
    \begin{minipage}{0.2\textwidth}
        \centering
        \includegraphics[width=\linewidth]{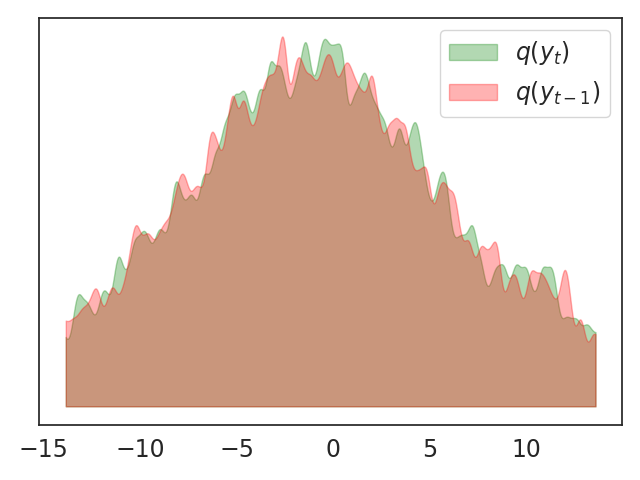}
    \end{minipage}
    \caption*{Frequency 1 (low), $t=3$}

    \vspace{5pt} %

    \begin{minipage}{0.2\textwidth}
        \centering
        \includegraphics[width=\linewidth]{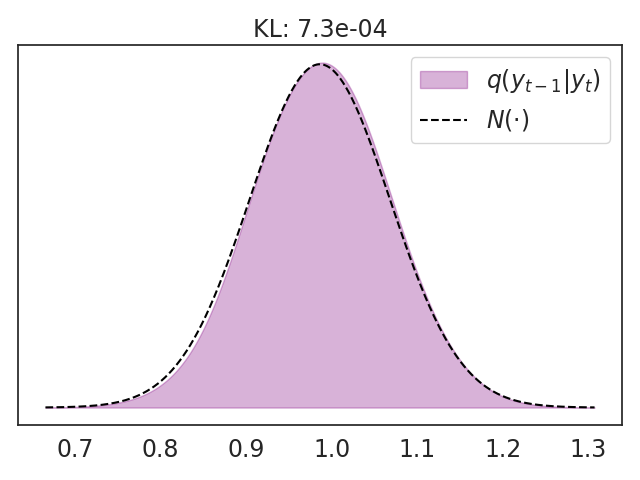}
    \end{minipage}
    \hfill
    \begin{minipage}{0.2\textwidth}
        \centering
        \includegraphics[width=\linewidth]{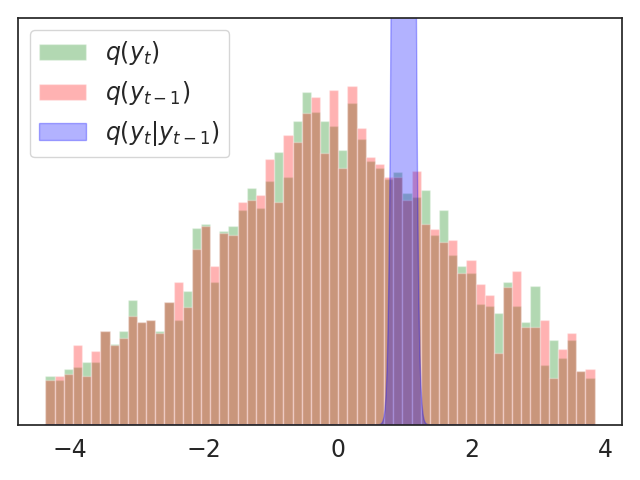}
    \end{minipage}
    \hfill
    \begin{minipage}{0.2\textwidth}
        \centering
        \includegraphics[width=\linewidth]{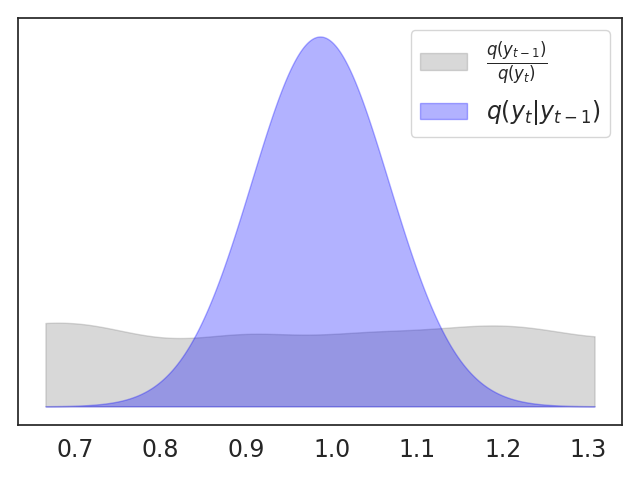}
    \end{minipage}
    \hfill
    \begin{minipage}{0.2\textwidth}
        \centering
        \includegraphics[width=\linewidth]{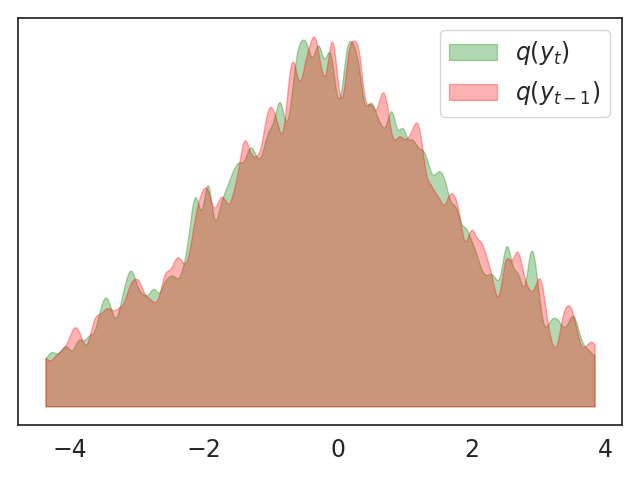}
    \end{minipage}
    \caption*{Frequency 2 (low), $t=1$}

    \vspace{5pt} %

    \begin{minipage}{0.2\textwidth}
        \centering
        \includegraphics[width=\linewidth]{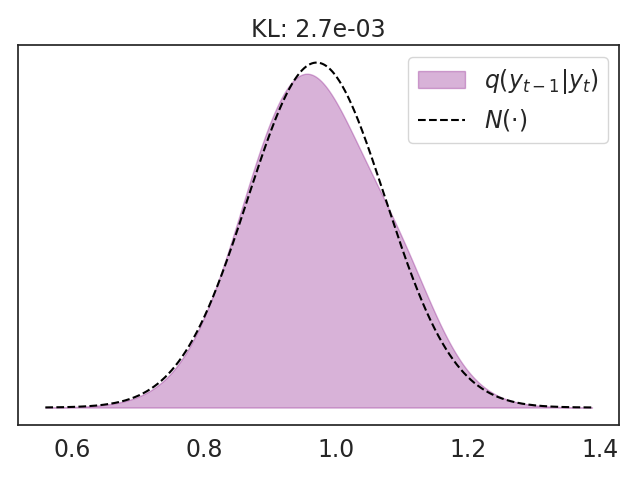}
    \end{minipage}
    \hfill
    \begin{minipage}{0.2\textwidth}
        \centering
        \includegraphics[width=\linewidth]{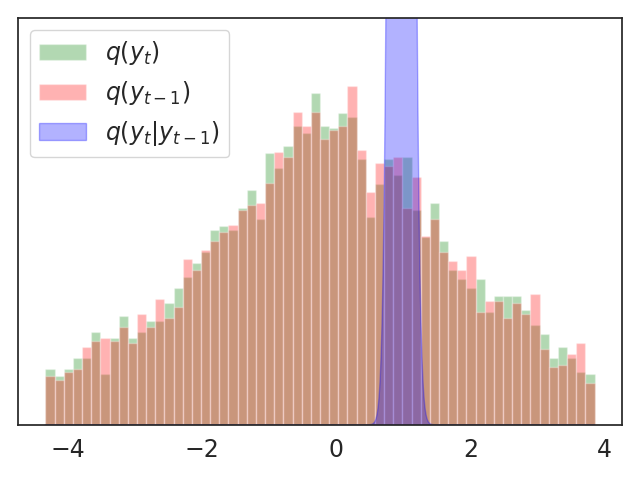}
    \end{minipage}
    \hfill
    \begin{minipage}{0.2\textwidth}
        \centering
        \includegraphics[width=\linewidth]{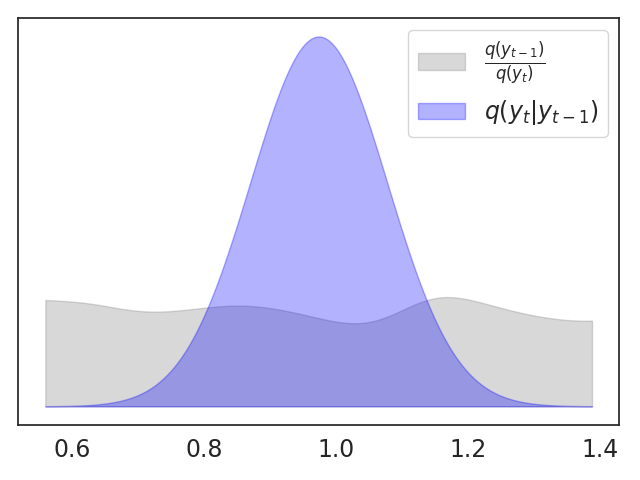}
    \end{minipage}
    \hfill
    \begin{minipage}{0.2\textwidth}
        \centering
        \includegraphics[width=\linewidth]{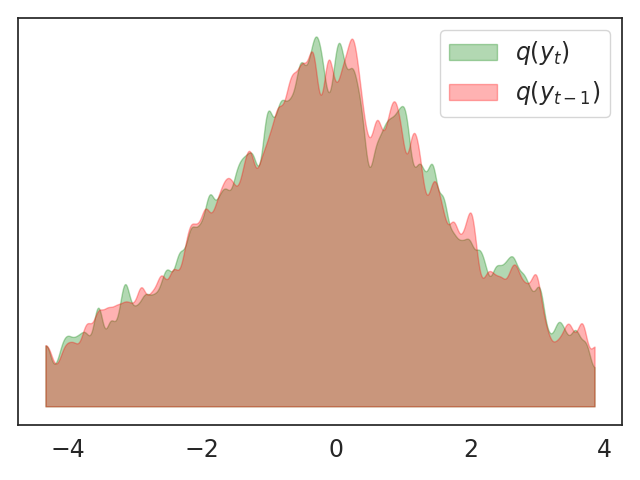}
    \end{minipage}
    \caption*{Frequency 2 (low), $t=2$}

    \vspace{5pt} %

    \begin{minipage}{0.2\textwidth}
        \centering
        \includegraphics[width=\linewidth]{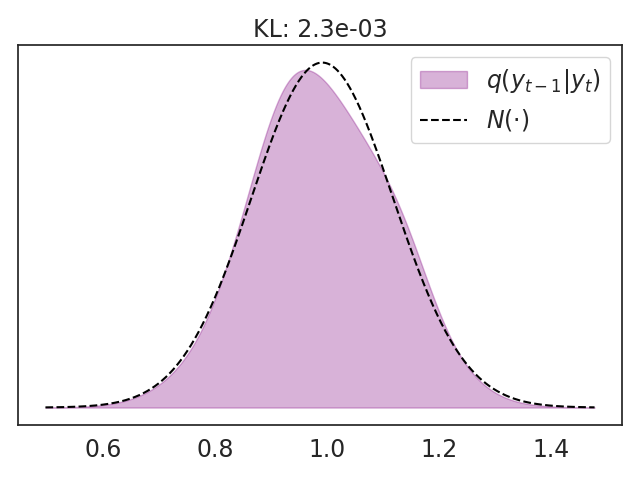}
    \end{minipage}
    \hfill
    \begin{minipage}{0.2\textwidth}
        \centering
        \includegraphics[width=\linewidth]{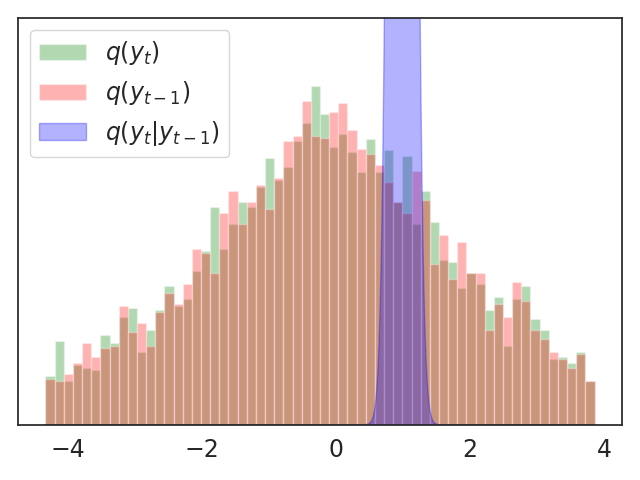}
    \end{minipage}
    \hfill
    \begin{minipage}{0.2\textwidth}
        \centering
        \includegraphics[width=\linewidth]{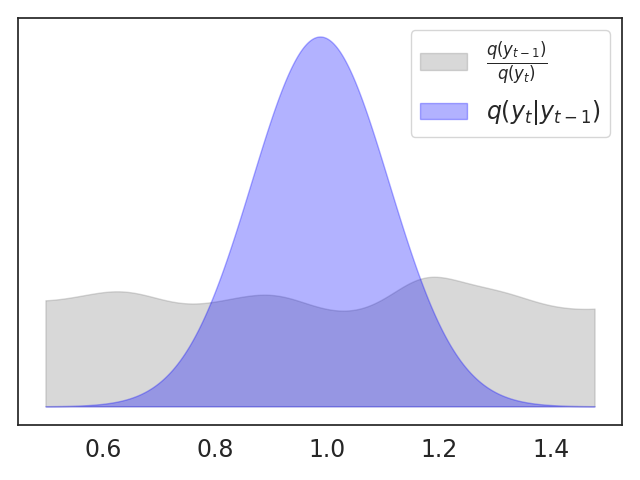}
    \end{minipage}
    \hfill
    \begin{minipage}{0.2\textwidth}
        \centering
        \includegraphics[width=\linewidth]{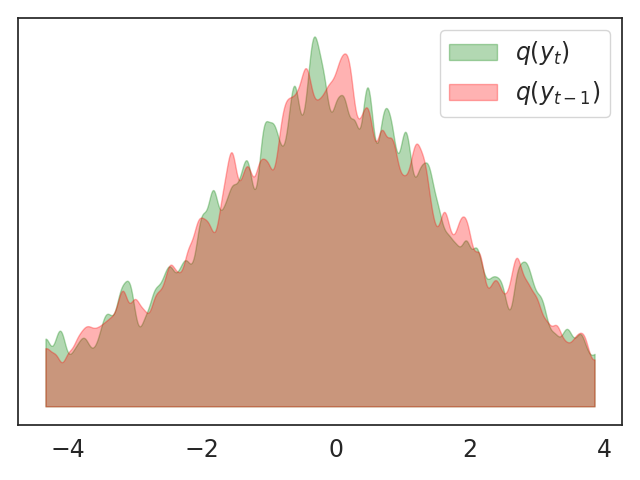}
    \end{minipage}
    \caption*{Frequency 2 (low), $t=3$}

\caption{Analysis of violations of the Gaussian assumption in EqualSNR (1 of 3).}
\end{figure}

\begin{figure}[p]
    \centering 
    
    \begin{minipage}{0.2\textwidth}
        \centering
        \includegraphics[width=\linewidth]{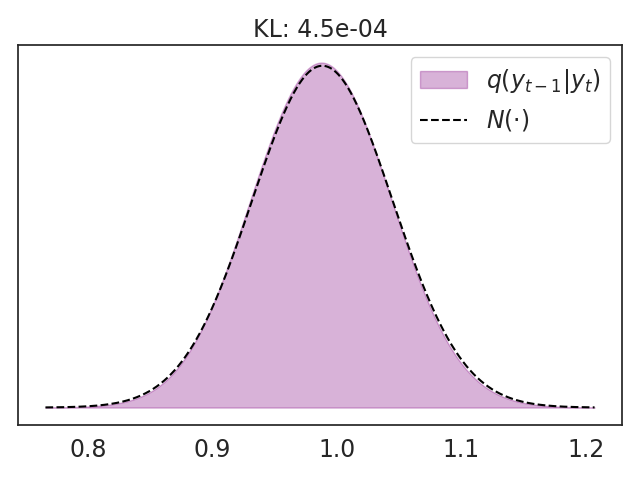}
    \end{minipage}
    \hfill
    \begin{minipage}{0.2\textwidth}
        \centering
        \includegraphics[width=\linewidth]{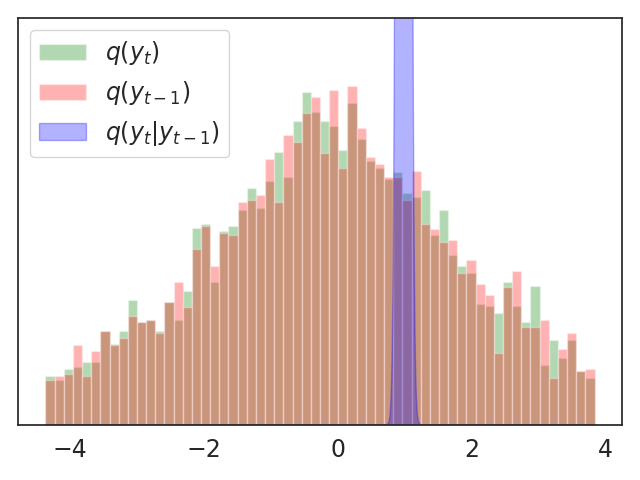}
    \end{minipage}
    \hfill
    \begin{minipage}{0.2\textwidth}
        \centering
        \includegraphics[width=\linewidth]{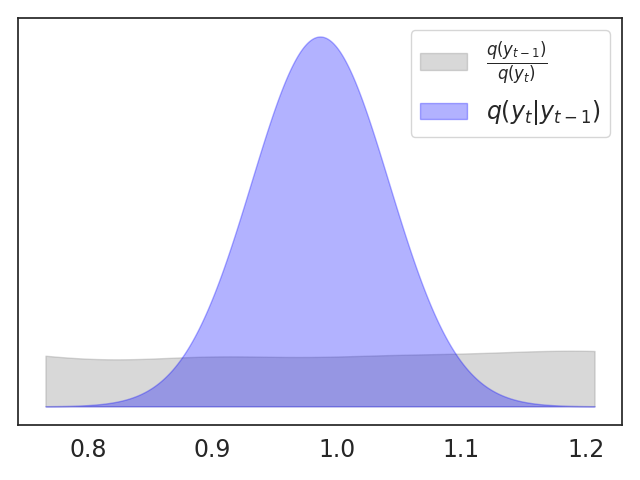}
    \end{minipage}
    \hfill
    \begin{minipage}{0.2\textwidth}
        \centering
        \includegraphics[width=\linewidth]{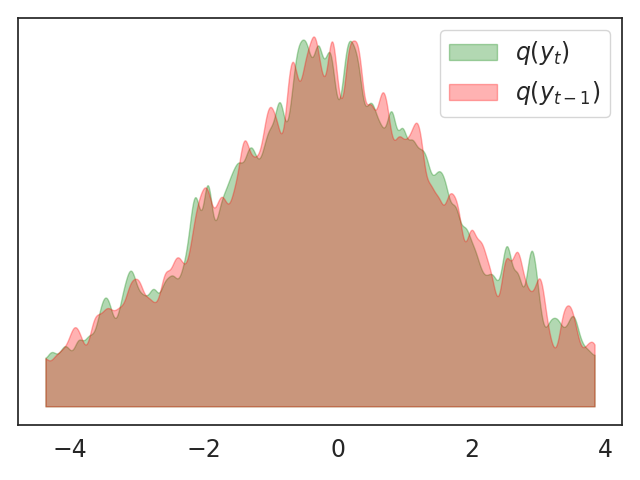}
    \end{minipage}
    \caption*{Frequency 3 (low), $t=1$}

    \vspace{5pt} %

    \begin{minipage}{0.2\textwidth}
        \centering
        \includegraphics[width=\linewidth]{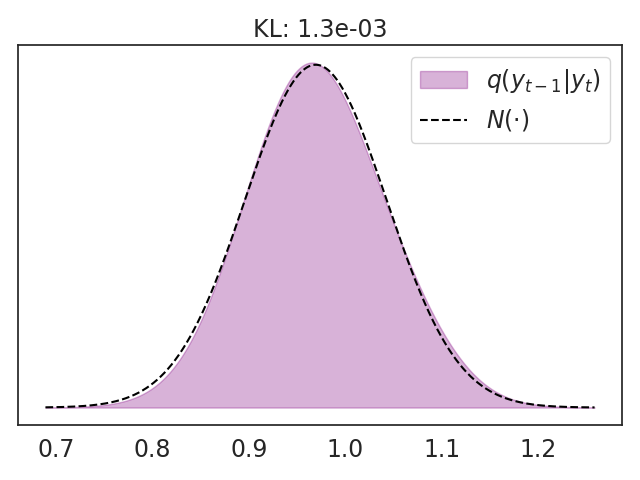}
    \end{minipage}
    \hfill
    \begin{minipage}{0.2\textwidth}
        \centering
        \includegraphics[width=\linewidth]{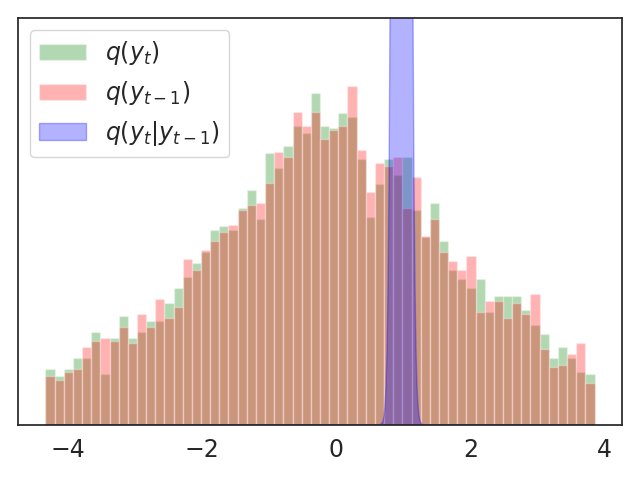}
    \end{minipage}
    \hfill
    \begin{minipage}{0.2\textwidth}
        \centering
        \includegraphics[width=\linewidth]{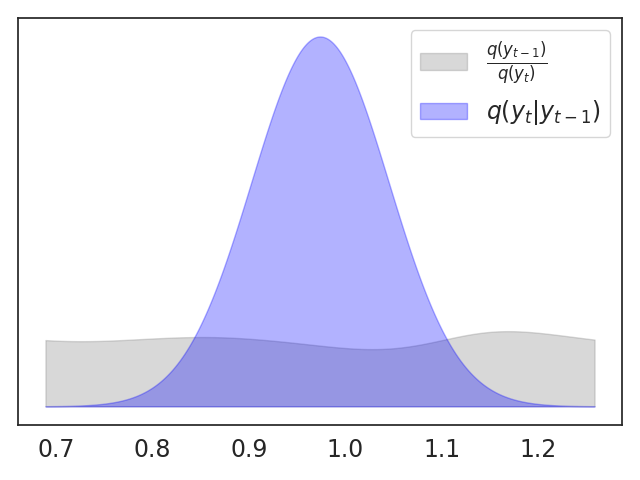}
    \end{minipage}
    \hfill
    \begin{minipage}{0.2\textwidth}
        \centering
        \includegraphics[width=\linewidth]{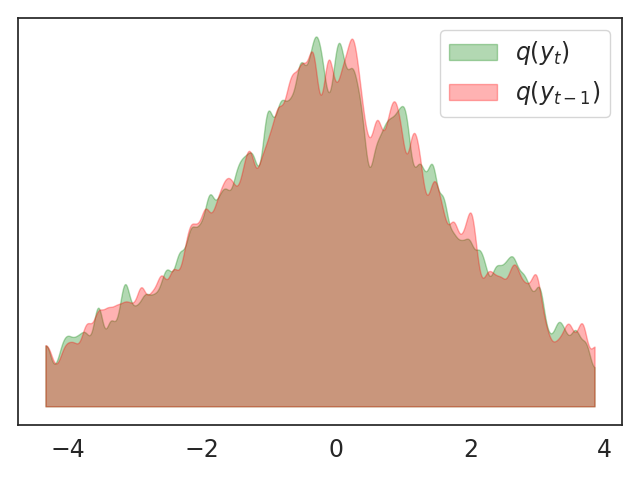}
    \end{minipage}
    \caption*{Frequency 3 (low), $t=2$}

    \vspace{5pt} %

    \begin{minipage}{0.2\textwidth}
        \centering
        \includegraphics[width=\linewidth]{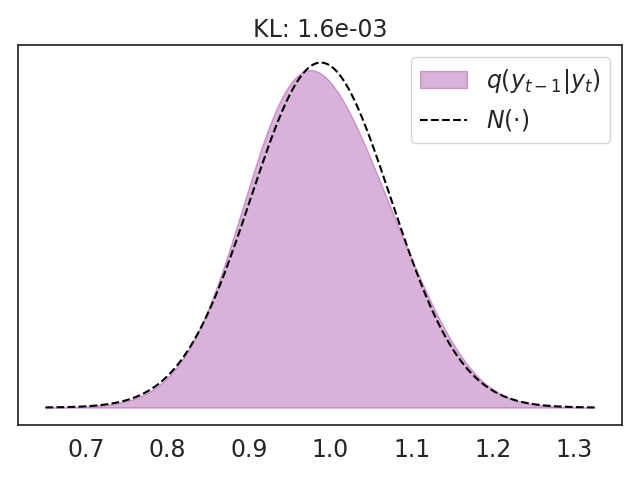}
    \end{minipage}
    \hfill
    \begin{minipage}{0.2\textwidth}
        \centering
        \includegraphics[width=\linewidth]{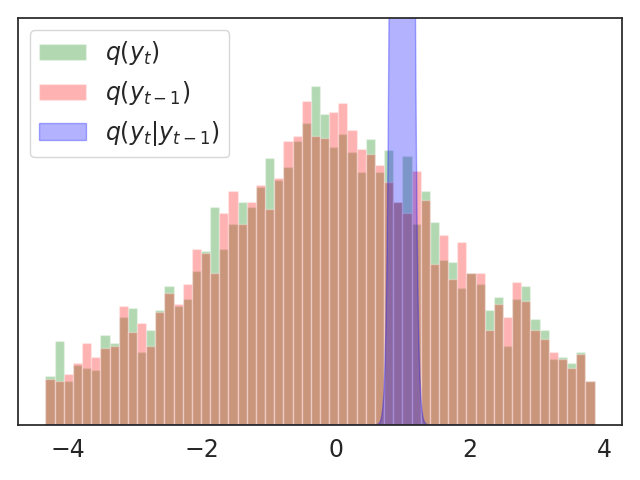}
    \end{minipage}
    \hfill
    \begin{minipage}{0.2\textwidth}
        \centering
        \includegraphics[width=\linewidth]{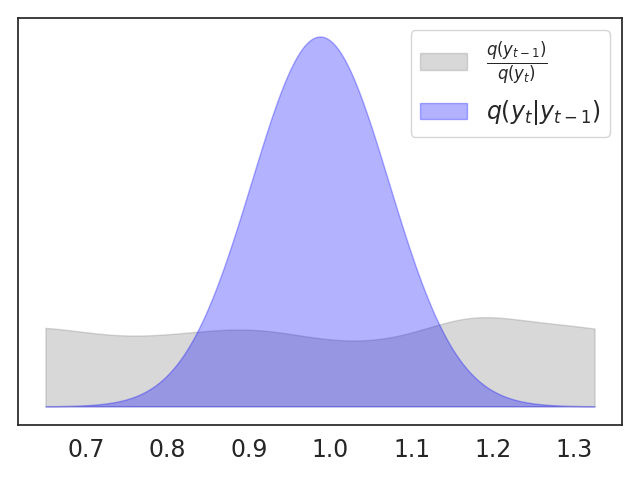}
    \end{minipage}
    \hfill
    \begin{minipage}{0.2\textwidth}
        \centering
        \includegraphics[width=\linewidth]{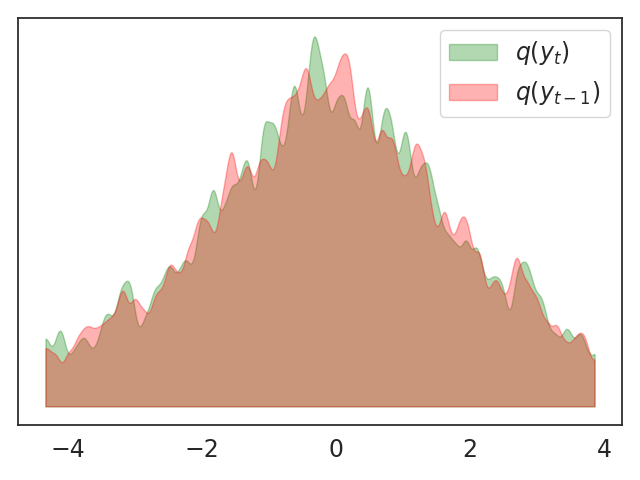}
    \end{minipage}
    \caption*{Frequency 3 (low), $t=3$}

    \vspace{5pt} %

    \begin{minipage}{0.2\textwidth}
        \centering
        \includegraphics[width=\linewidth]{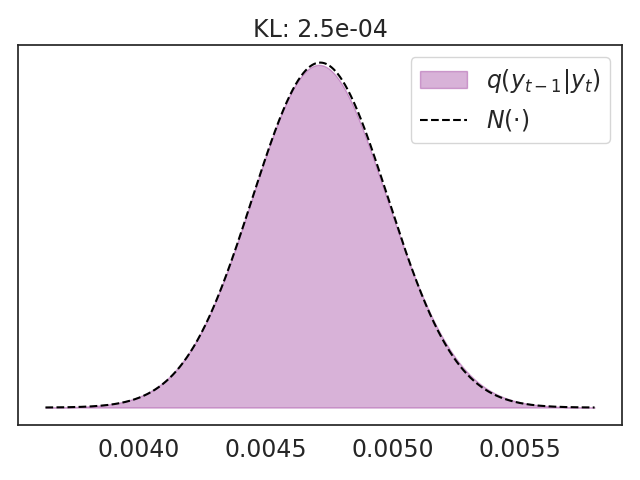}
    \end{minipage}
    \hfill
    \begin{minipage}{0.2\textwidth}
        \centering
        \includegraphics[width=\linewidth]{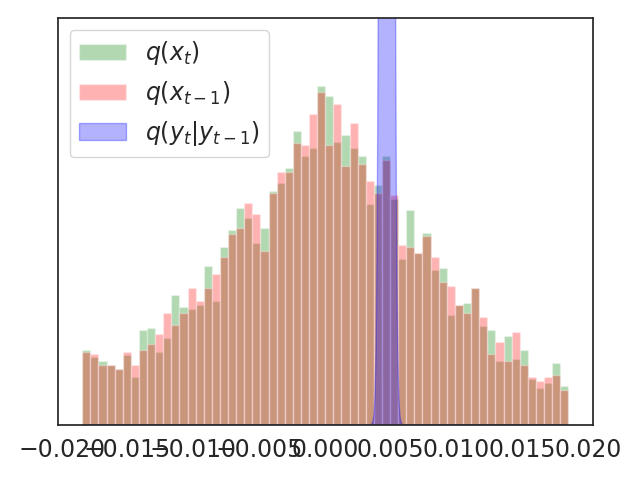}
    \end{minipage}
    \hfill
    \begin{minipage}{0.2\textwidth}
        \centering
        \includegraphics[width=\linewidth]{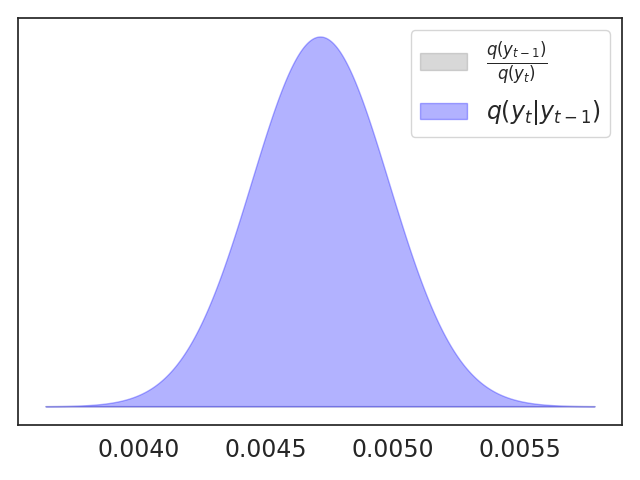}
    \end{minipage}
    \hfill
    \begin{minipage}{0.2\textwidth}
        \centering
        \includegraphics[width=\linewidth]{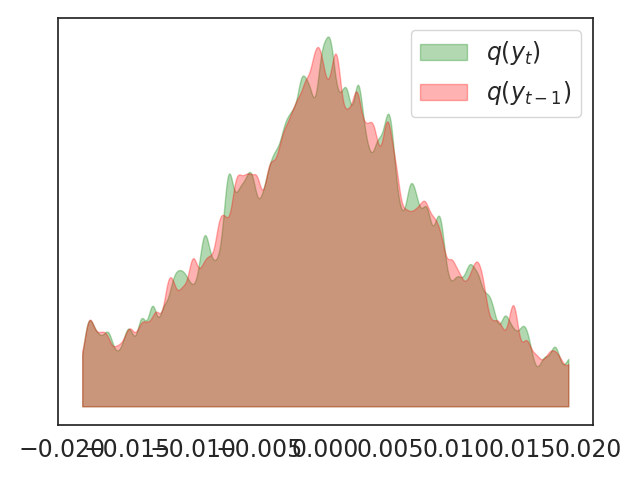}
    \end{minipage}
    \caption*{Frequency 1022 (high), $t=1$}

    \vspace{5pt} %

    \begin{minipage}{0.2\textwidth}
        \centering
        \includegraphics[width=\linewidth]{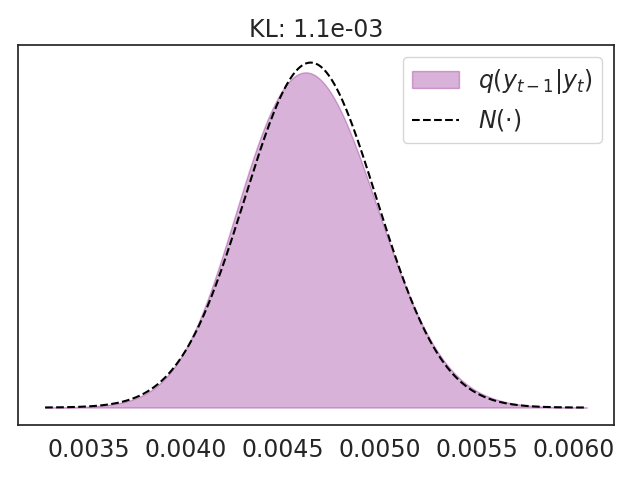}
    \end{minipage}
    \hfill
    \begin{minipage}{0.2\textwidth}
        \centering
        \includegraphics[width=\linewidth]{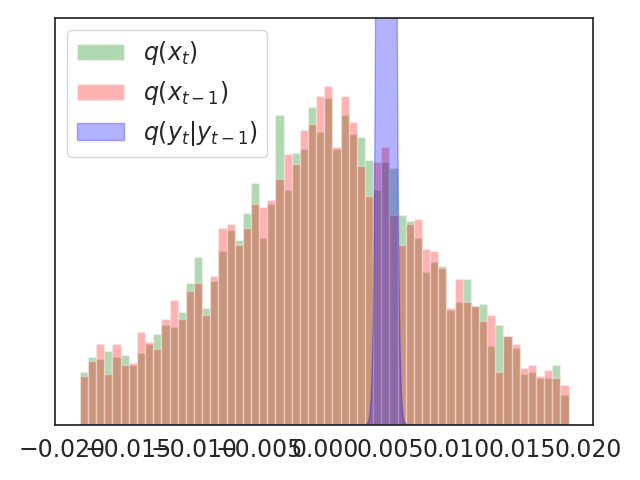}
    \end{minipage}
    \hfill
    \begin{minipage}{0.2\textwidth}
        \centering
        \includegraphics[width=\linewidth]{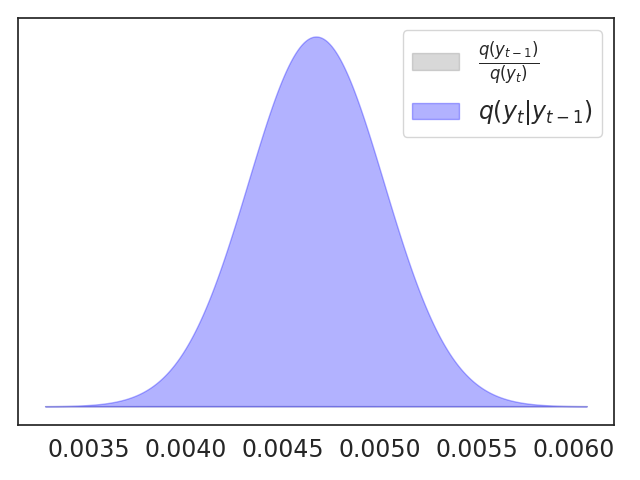}
    \end{minipage}
    \hfill
    \begin{minipage}{0.2\textwidth}
        \centering
        \includegraphics[width=\linewidth]{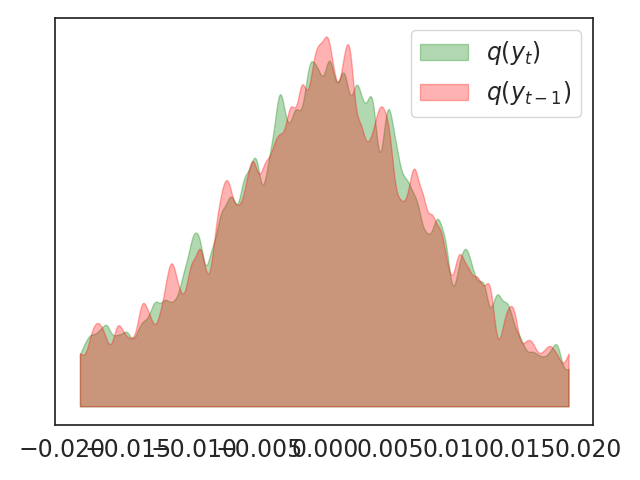}
    \end{minipage}
    \caption*{Frequency 1022 (high), $t=2$}

    \vspace{5pt} %

    \begin{minipage}{0.2\textwidth}
        \centering
        \includegraphics[width=\linewidth]{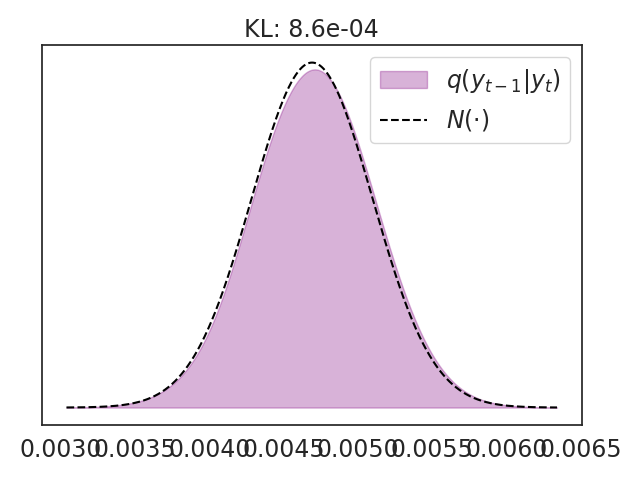}
    \end{minipage}
    \hfill
    \begin{minipage}{0.2\textwidth}
        \centering
        \includegraphics[width=\linewidth]{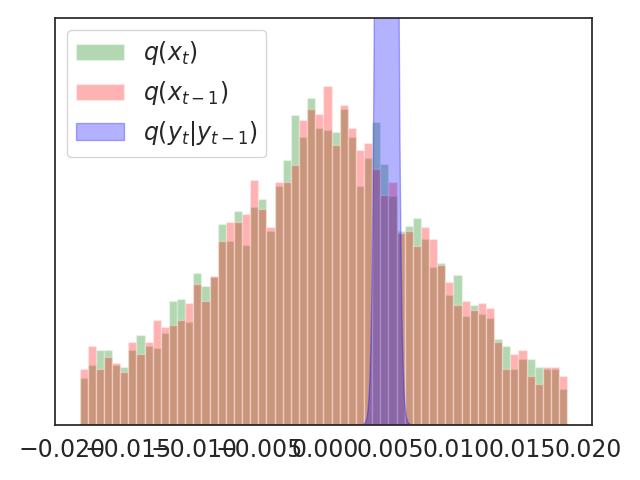}
    \end{minipage}
    \hfill
    \begin{minipage}{0.2\textwidth}
        \centering
        \includegraphics[width=\linewidth]{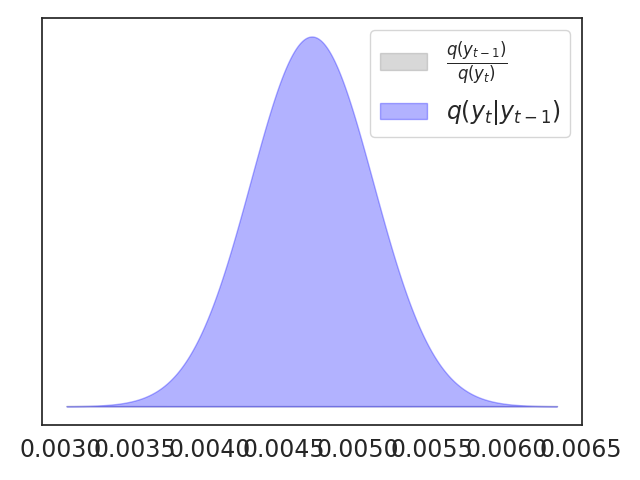}
    \end{minipage}
    \hfill
    \begin{minipage}{0.2\textwidth}
        \centering
        \includegraphics[width=\linewidth]{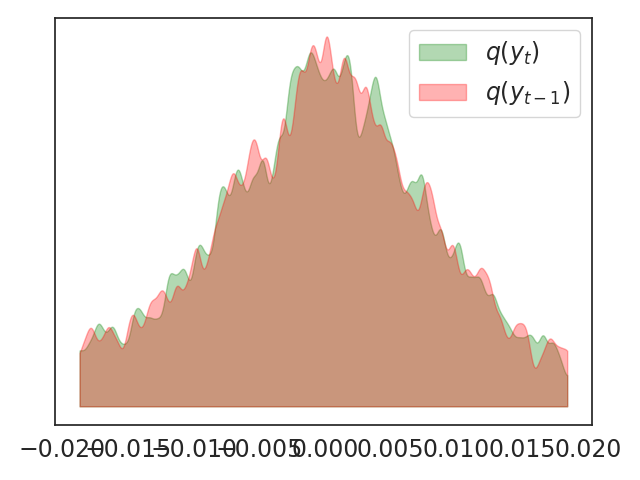}
    \end{minipage}
    \caption*{Frequency 1022 (high), $t=3$}

\caption{Analysis of violations of the Gaussian assumption in EqualSNR (2 of 3).}
\end{figure}

\begin{figure}[p]
    \centering 
    
    \begin{minipage}{0.2\textwidth}
        \centering
        \includegraphics[width=\linewidth]{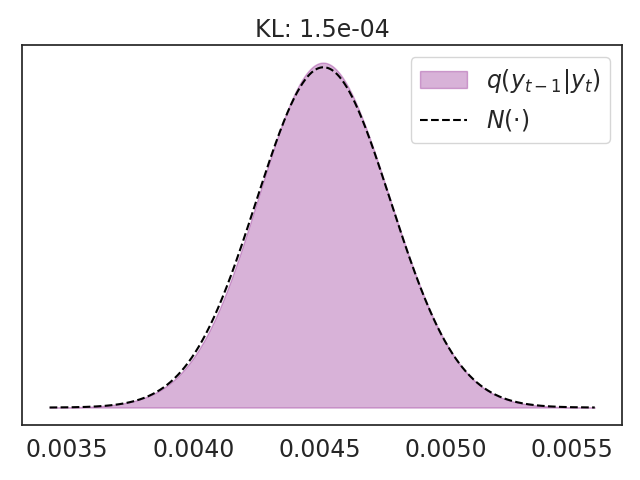}
    \end{minipage}
    \hfill
    \begin{minipage}{0.2\textwidth}
        \centering
        \includegraphics[width=\linewidth]{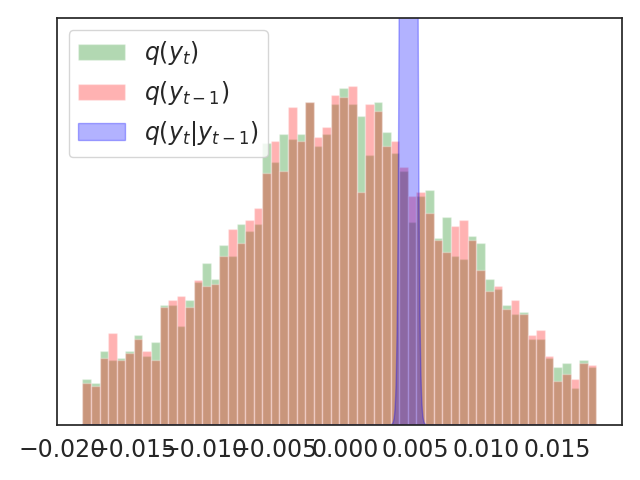}
    \end{minipage}
    \hfill
    \begin{minipage}{0.2\textwidth}
        \centering
        \includegraphics[width=\linewidth]{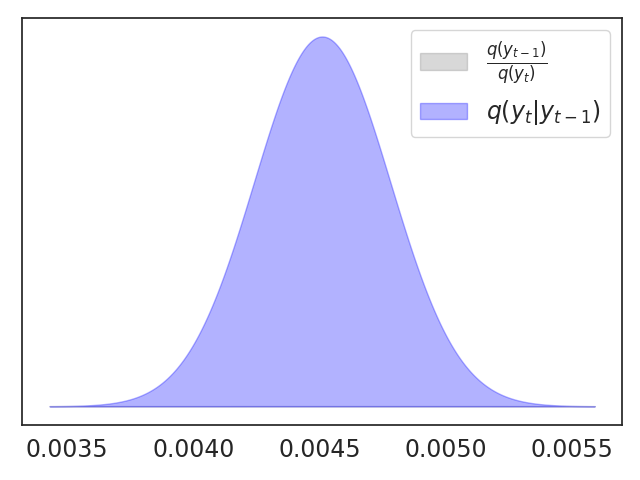}
    \end{minipage}
    \hfill
    \begin{minipage}{0.2\textwidth}
        \centering
        \includegraphics[width=\linewidth]{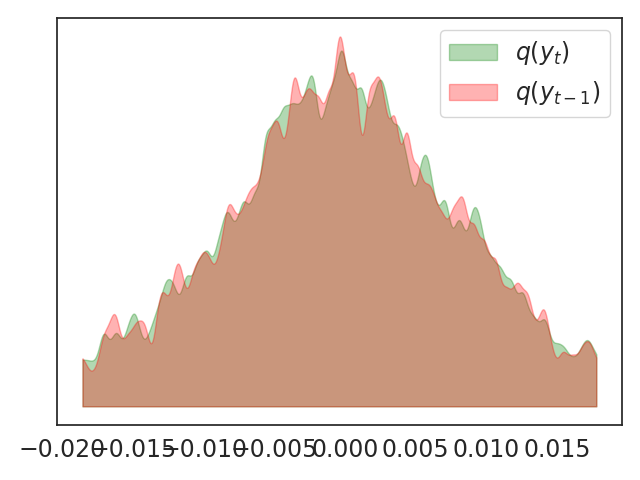}
    \end{minipage}
    \caption*{Frequency 1023 (high), $t=1$}

    \vspace{5pt} %

    \begin{minipage}{0.2\textwidth}
        \centering
        \includegraphics[width=\linewidth]{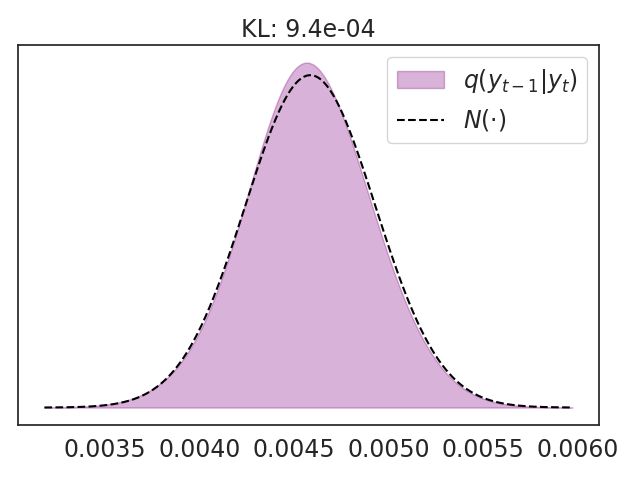}
    \end{minipage}
    \hfill
    \begin{minipage}{0.2\textwidth}
        \centering
        \includegraphics[width=\linewidth]{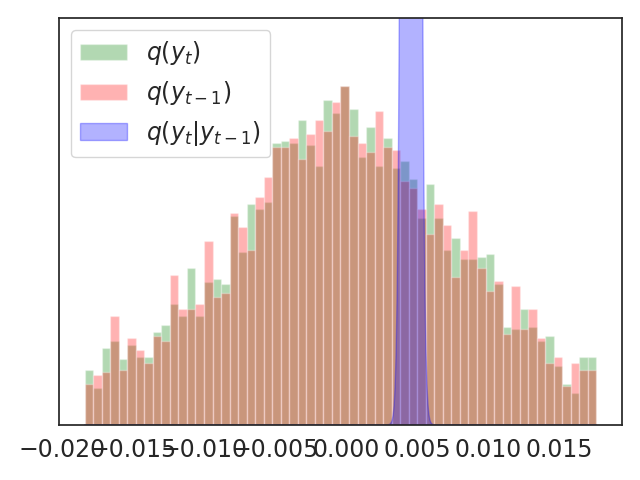}
    \end{minipage}
    \hfill
    \begin{minipage}{0.2\textwidth}
        \centering
        \includegraphics[width=\linewidth]{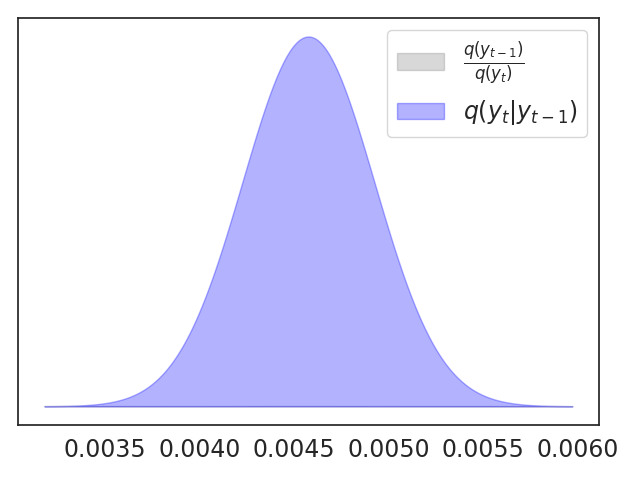}
    \end{minipage}
    \hfill
    \begin{minipage}{0.2\textwidth}
        \centering
        \includegraphics[width=\linewidth]{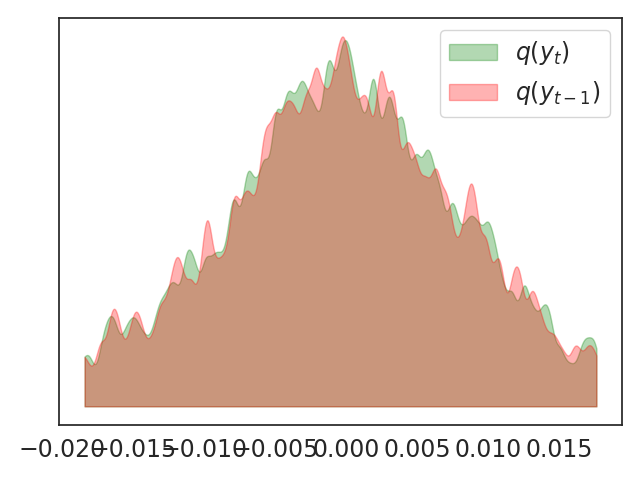}
    \end{minipage}
    \caption*{Frequency 1023 (high), $t=2$}

    \vspace{5pt} %

    \begin{minipage}{0.2\textwidth}
        \centering
        \includegraphics[width=\linewidth]{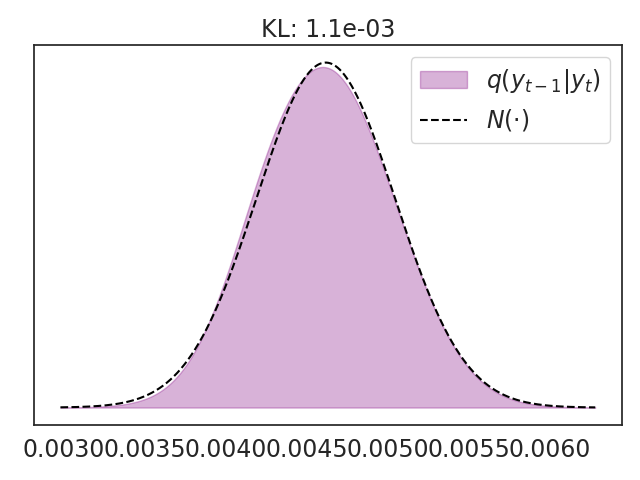}
    \end{minipage}
    \hfill
    \begin{minipage}{0.2\textwidth}
        \centering
        \includegraphics[width=\linewidth]{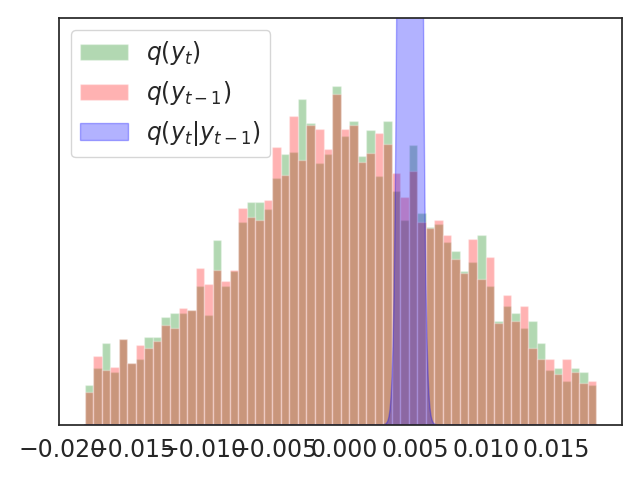}
    \end{minipage}
    \hfill
    \begin{minipage}{0.2\textwidth}
        \centering
        \includegraphics[width=\linewidth]{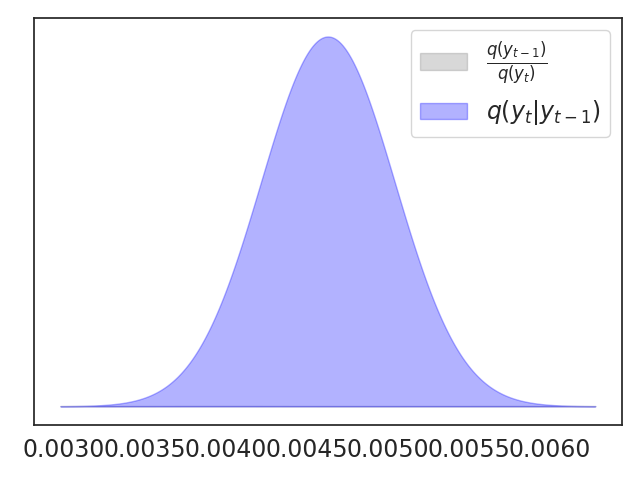}
    \end{minipage}
    \hfill
    \begin{minipage}{0.2\textwidth}
        \centering
        \includegraphics[width=\linewidth]{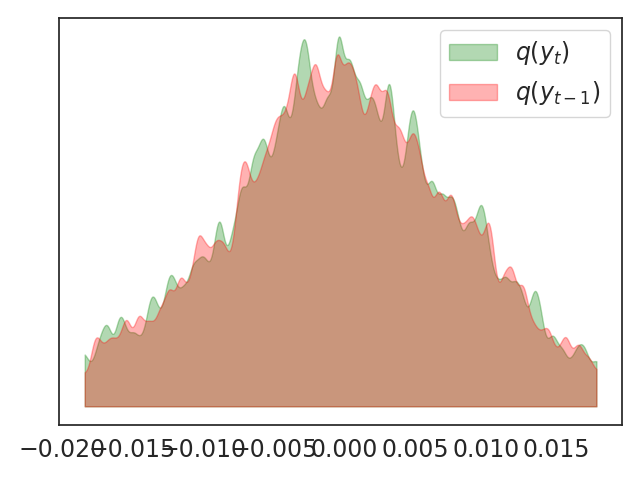}
    \end{minipage}
    \caption*{Frequency 1023 (high), $t=3$}

    \vspace{5pt} %

    \begin{minipage}{0.2\textwidth}
        \centering
        \includegraphics[width=\linewidth]{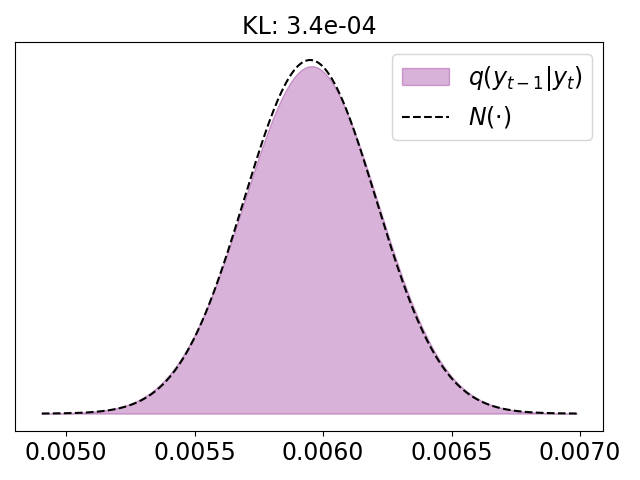}
    \end{minipage}
    \hfill
    \begin{minipage}{0.2\textwidth}
        \centering
        \includegraphics[width=\linewidth]{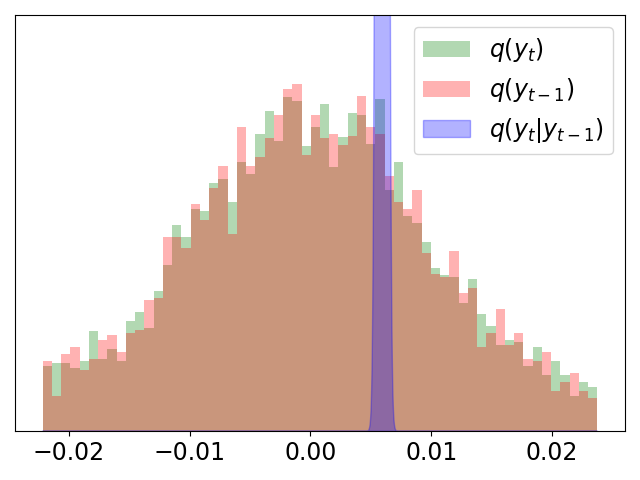}
    \end{minipage}
    \hfill
    \begin{minipage}{0.2\textwidth}
        \centering
        \includegraphics[width=\linewidth]{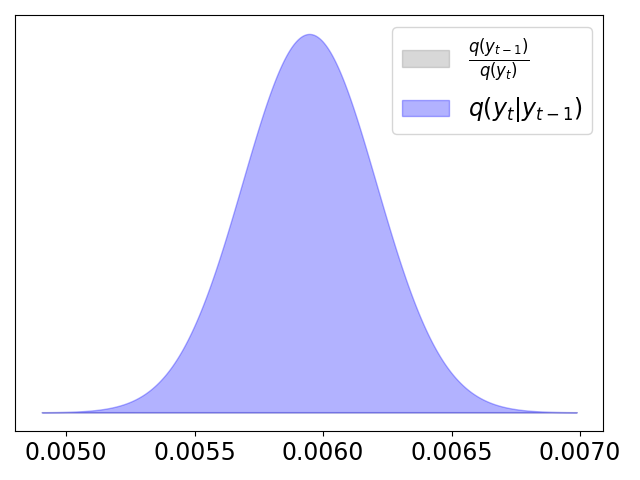}
    \end{minipage}
    \hfill
    \begin{minipage}{0.2\textwidth}
        \centering
        \includegraphics[width=\linewidth]{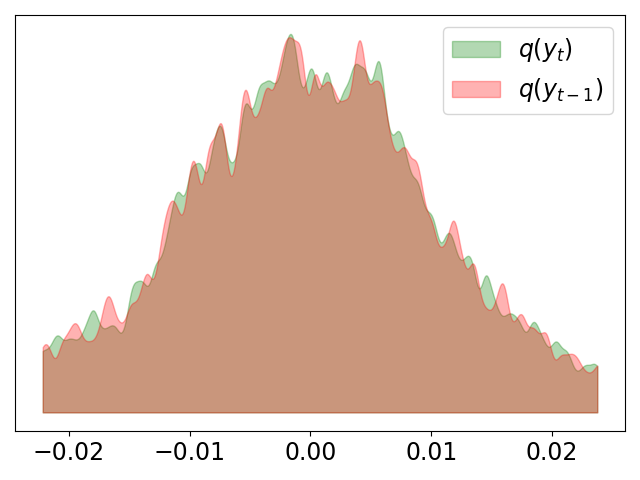}
    \end{minipage}
    \caption*{Frequency 1024 (high), $t=1$}

    \vspace{5pt} %

    \begin{minipage}{0.2\textwidth}
        \centering
        \includegraphics[width=\linewidth]{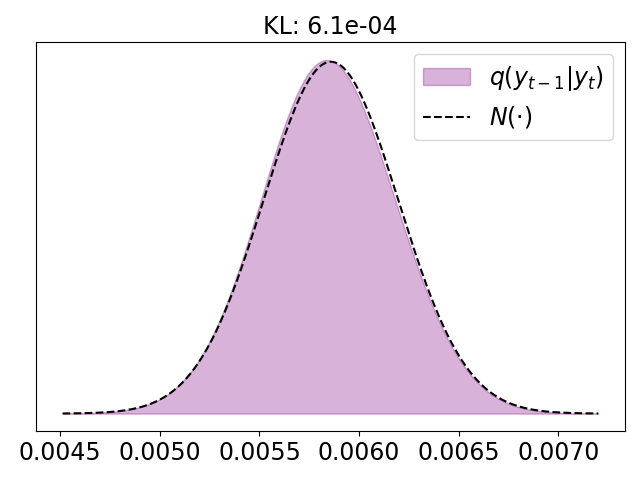}
    \end{minipage}
    \hfill
    \begin{minipage}{0.2\textwidth}
        \centering
        \includegraphics[width=\linewidth]{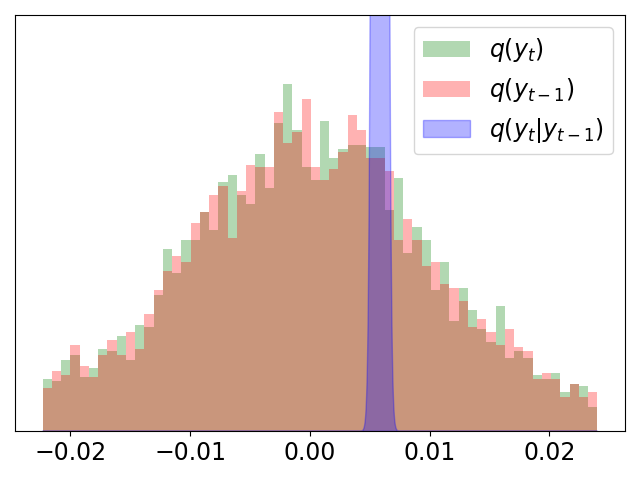}
    \end{minipage}
    \hfill
    \begin{minipage}{0.2\textwidth}
        \centering
        \includegraphics[width=\linewidth]{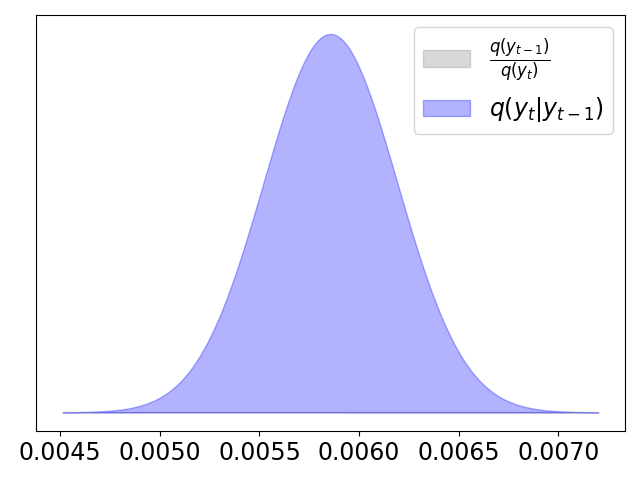}
    \end{minipage}
    \hfill
    \begin{minipage}{0.2\textwidth}
        \centering
        \includegraphics[width=\linewidth]{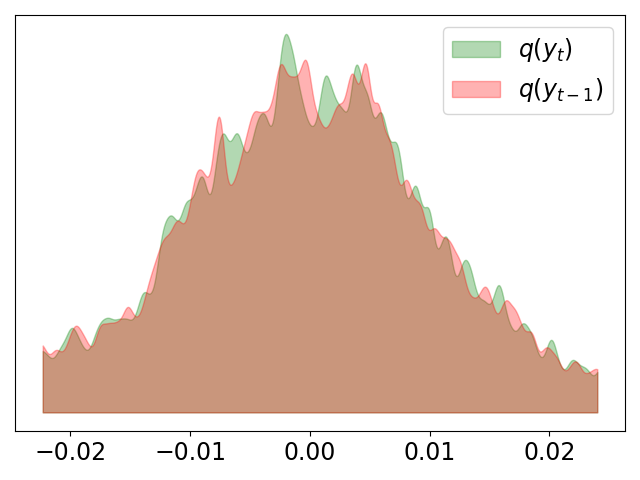}
    \end{minipage}
    \caption*{Frequency 1024 (high), $t=2$}

    \vspace{5pt} %

    \begin{minipage}{0.2\textwidth}
        \centering
        \includegraphics[width=\linewidth]{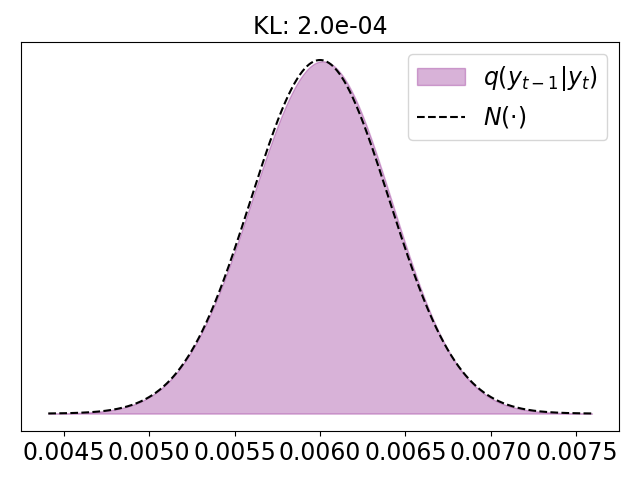}
    \end{minipage}
    \hfill
    \begin{minipage}{0.2\textwidth}
        \centering
        \includegraphics[width=\linewidth]{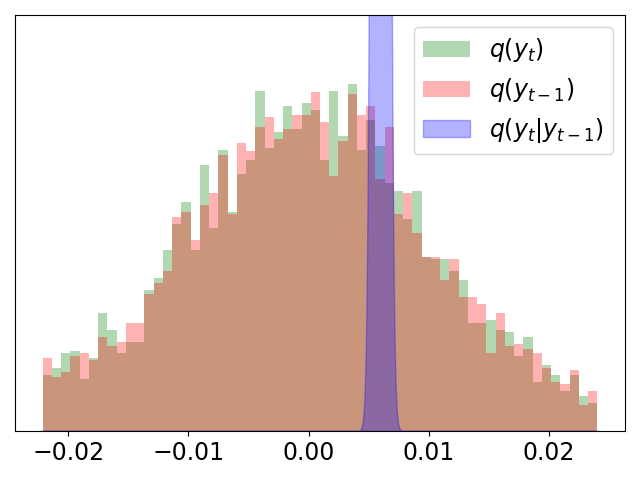}
    \end{minipage}
    \hfill
    \begin{minipage}{0.2\textwidth}
        \centering
        \includegraphics[width=\linewidth]{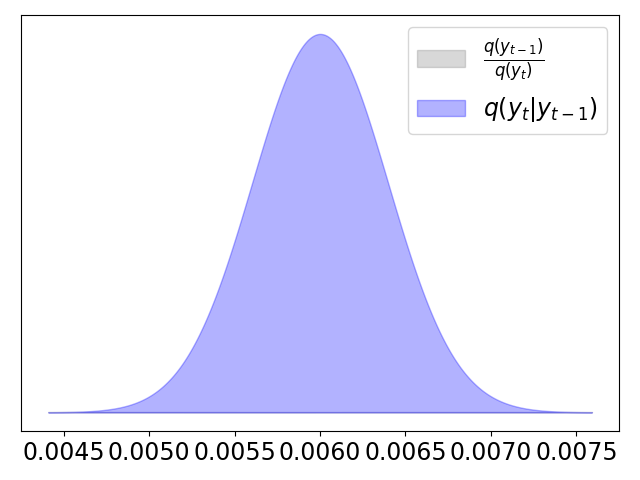}
    \end{minipage}
    \hfill
    \begin{minipage}{0.2\textwidth}
        \centering
        \includegraphics[width=\linewidth]{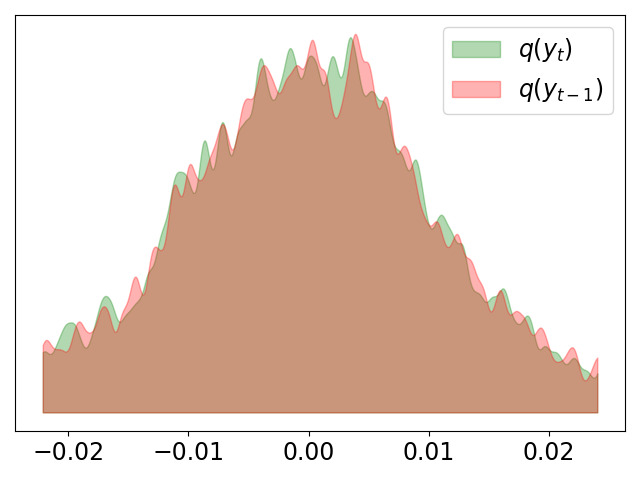}
    \end{minipage}
    \caption*{Frequency 1024 (high), $t=3$}

\caption{Analysis of violations of the Gaussian assumption in EqualSNR (3 of 3).}
\label{fig:gauss_app_end}
\end{figure}

\paragraph{Training a FlippedSNR diffusion model is challenging. }

As part of our experiments we tried to train with an inverted or \textit{flipped} SNR frequency profile as compared to DDPM (FlippedSNR), see \Cref{fig:flipped-snr}.
The idea of this schedule is to generate high frequencies before low frequencies in the reverse process, by having a noising process that corrupts the high frequencies gently at early timesteps, and corrupts low frequencies at a higher rate.
This noising process was accomplished by computing the Fourier variance of the data, and inverting the forward process by adding noise that \textit{flipped} the heatmap (see \Cref{fig:flipped-snr,fig:fwd_flippedsnr_app} for an illustration).
However, in spite of numerous attempts, we did not manage to learn the reverse process appropriately (see \Cref{fig:bwd_flippedsnr_app} for a visualisation of generated samples), and hence failed to approximate the data distribution.
In the following, we describe these attempts further.

We tried naive approaches such as multiple beta schedules, including cosine, linear, as well as more complicated variants which change the SNR linearly and with a cosine function, and different start and end values for the SNR.
In particular, we calculated the highest SNR values at $t=1$ in standard DDPM (about 40 dB for low frequencies) with a standard cosine schedule.
Similarly, lowest SNR values on high frequencies were around 5 dB also at $t=1$. 
This disparity corroborates that frequencies are not noised at the same rate.

We speculated that if we flipped the heatmap profile and increased the frequencies with lowest SNRs (in the flipped regime low frequencies would be very quickly corrupted), we could possibly train successfully.
We shifted the start SNR value to 60 dB (to improve the lowest SNR profile), and we kept the end SNR value the same at -40 dB.
We also tried shifting to -20 dB to maintain the same range.
We also trained with different number of discretization steps, i.e., $T\in\{1000,5000,10000\}$
However, all these approaches failed to train.
These experiments provide partial evidence that generating high before low frequencies may be challenging or not efficient, but further experiments are required to confirm this hypothesis.

\clearpage
\Cref{fig:c_matrix} presents the $\mathbf{C}$ matrix for the three standard imaging benchmarks used in \Cref{sec:experiments}.

\begin{figure}[h]
    \centering
    \begin{minipage}{0.33\textwidth}
        \centering
        \includegraphics[width=\textwidth]{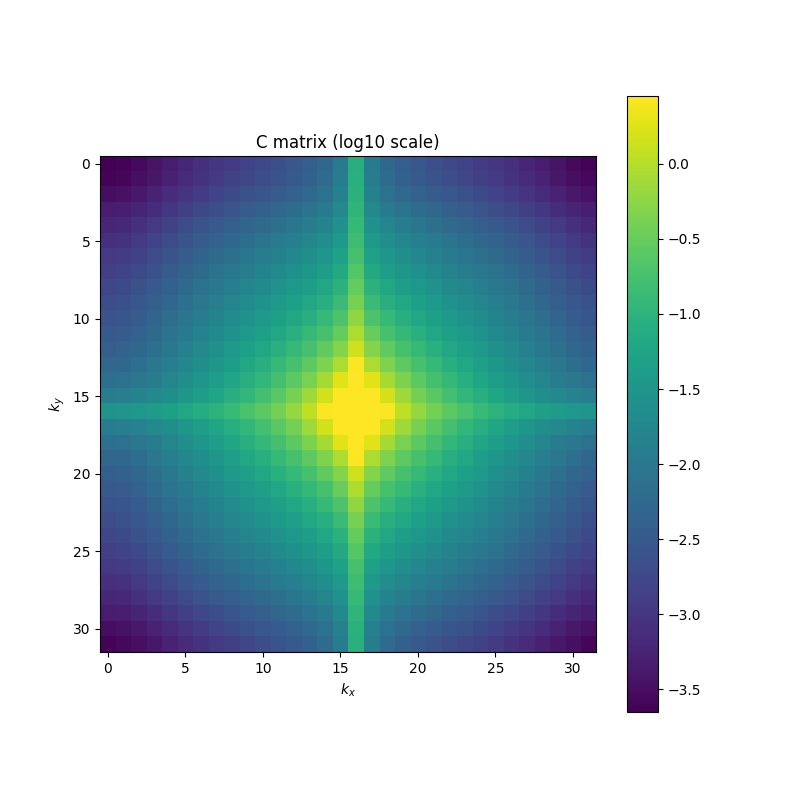}
    \end{minipage}
    \hfill
    \begin{minipage}{0.33\textwidth}
        \centering
        \includegraphics[width=\textwidth]{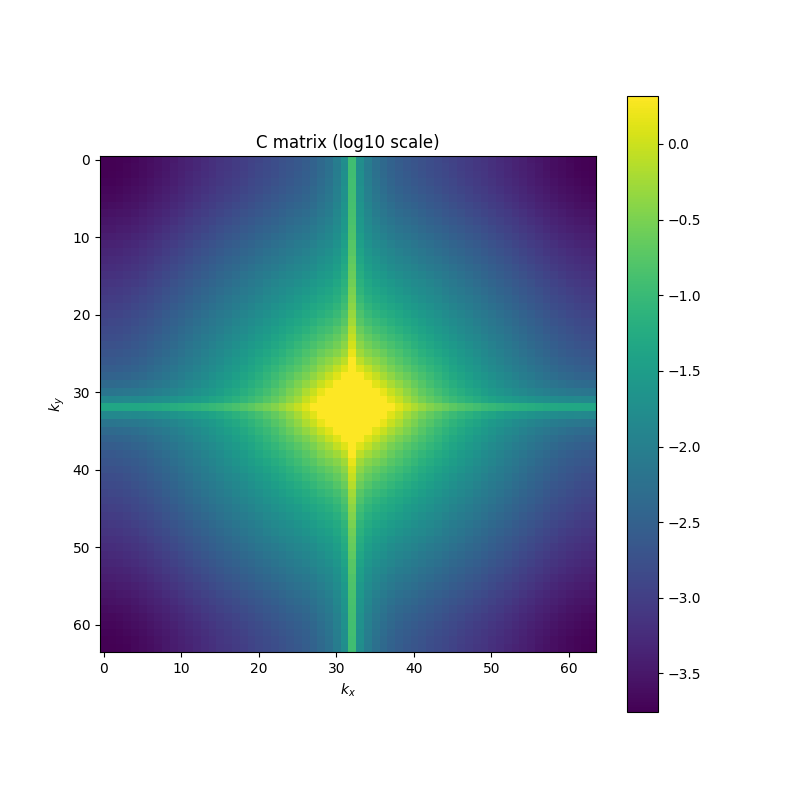}
    \end{minipage}
    \hfill
    \begin{minipage}{0.33\textwidth}
        \centering
        \includegraphics[width=\textwidth]{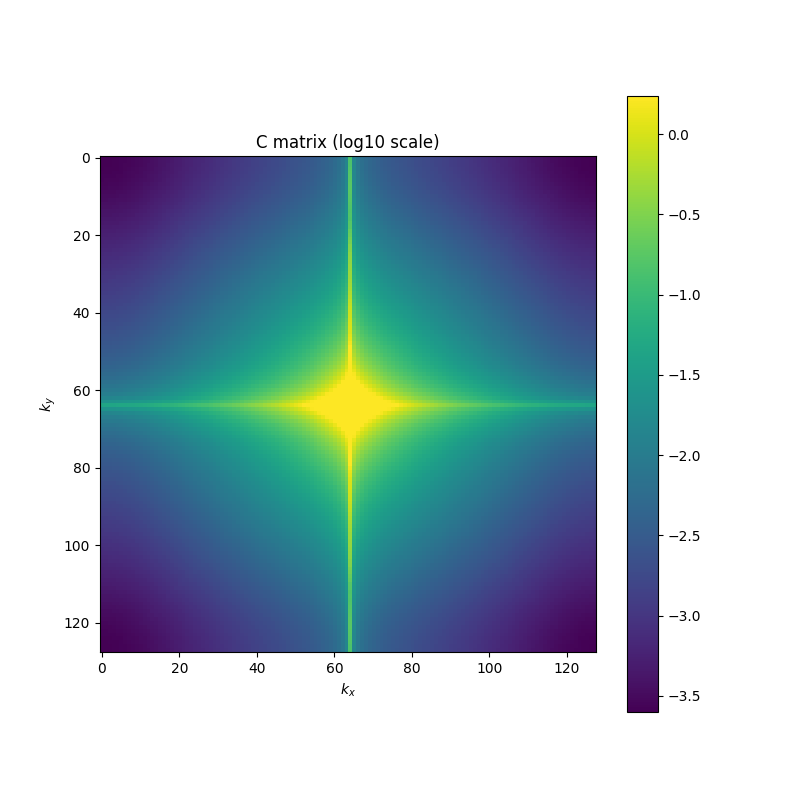}
    \end{minipage}
    \caption{The signal variance $\mathbf{C}$ in Fourier space (magnitude) on a log scale (with basis 10) for CIFAR10 [left], CelebA [centre] and LSUN Church [right].
    For visualisation purposes, the largest value plotted in bright yellow corresponds to a value larger or equal to the $.95$-quantile of $\mathbf{C}$.
    }
    \label{fig:c_matrix}
\end{figure}

\Cref{fig:flipped-snr-gif} illustrates the SNR trajectory of the FlippedSNR forward process, discussed in \Cref{app:Extended theoretical framework and proofs}, corresponding to the respective illustrations of DDPM and EqualSNR in \Cref{fig:fig1}.

\begin{figure}[h!]
    \centering
    \animategraphics[loop, autoplay, width=.6\textwidth]{2}{Figures/fig1_modality-images_cifar10_K_True/fig-}{01}{21}
    \caption{
    The alternate FlippedSNR forward process noises \textit{low-frequency} components before and faster than high-frequency components, flipping the SNR profile of DDPM. 
    This figure complements the corresponding figures for DDPM and EqualSNR in \Cref{fig:fig1}.
    The GIF is best viewed in Adobe Reader.
    }
    \label{fig:flipped-snr-gif}
\end{figure}

\clearpage

The goal of \Cref{fig:freq_var_app} is---augmenting the view of \Cref{fig:snr_heatmaps}---to characterise that the learned reverse process mirrors the forward process in diffusion models. 
This figure computes the per-frequency variances of samples generated via the respective push-forward distributions (in the forward process), and by sampling with DDIM or \Cref{algo:sampling} (in the reverse process), respectively. 
It hence provides a notion of per-frequency variability at a given diffusion timestep.
It can also be interpreted relatively to the base variability at $t=0$, which indicates when certain frequencies stabilise.

\begin{figure}[h!]
    \centering
    \includegraphics[width=1.\textwidth]{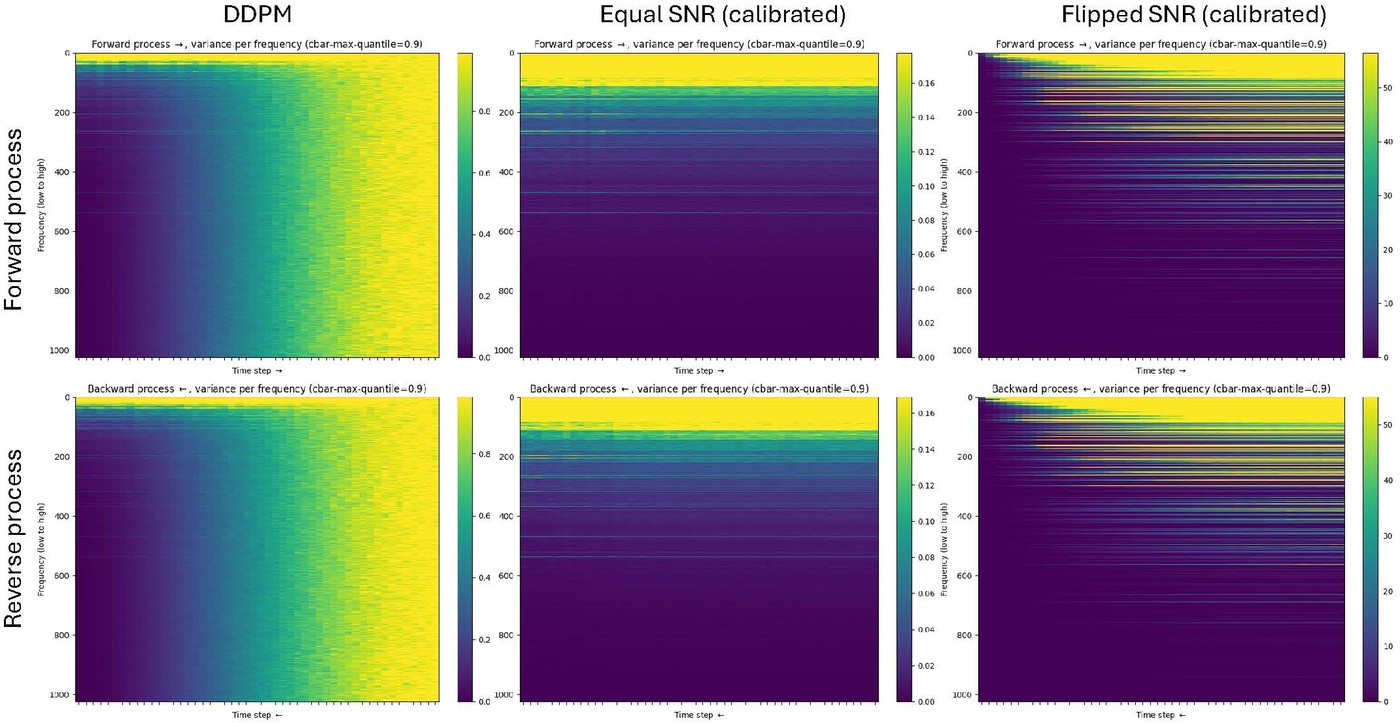}
    \caption{
    \textit{The reverse process learns to mirror the forward process.}
    We show the variances at time $t$ of the [top] forward and  [bottom] backward process in Fourier space, for [left] DDPM, [centre] \texttt{ESNR}, and [right] FlippedSNR. 
    }
    \label{fig:freq_var_app}
\end{figure}

In \Cref{fig:C_cifar10_ddpm} we illustrate the $\mathbf{C}$ matrix, i.e. the dimension-wise variance in pixel space (as opposed to in Euclidean space), illustrating the how the variance of the data spreads spatially across an image. 
In particular, we can see that the variance deviates from $1$, rendering DDPM not a variance-preserving forward process, in the sense that the variance changes throughout diffusion time.

\begin{figure}[h]
    \centering
    \includegraphics[width=0.3\linewidth]{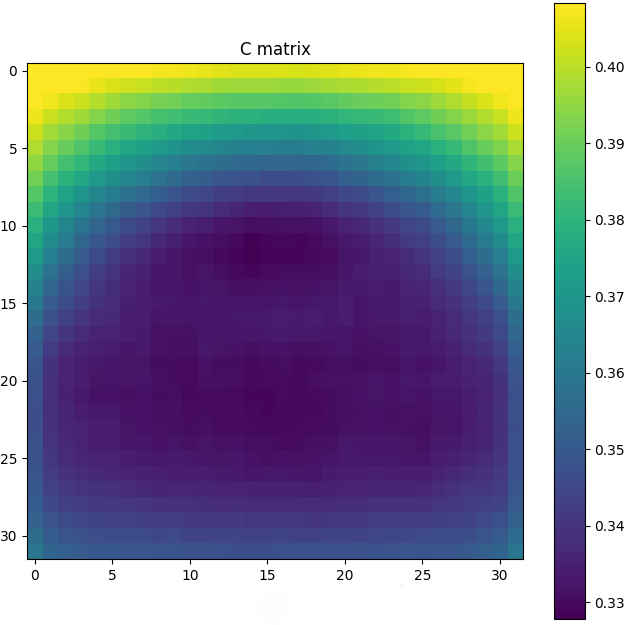}
    \caption{
    \textit{DDPM on images normalised to $[-1,1]$ is \textit{not} variance-preserving},  i.e. the variance changes throughout the forward process.
    To see this, we plot the coordinate-wise signal variances $\text{diag}(Cov(\mathbf{x}_0))$ in \textit{pixel space} (in contrast to \Cref{fig:c_matrix}), computed on the CIFAR10 dataset whose images have been normalised to the range $[-1,1]$.}
    \label{fig:C_cifar10_ddpm}
\end{figure}

\clearpage

In Figures \ref{fig:heatmap_celeba} and \ref{fig:heatmap_lsun} we present SNR heatmaps for the forward process on the higher-resolution datasets CelebA and LSUN Church datasets, corresponding to \Cref{fig:snr_heatmaps}.
In \Cref{fig:flipped-snr} we also present the SNR heatmap for the forward process of FlippedSNR on CIFAR10.

\begin{figure*}[h!]
    \centering
    \hfill
    \begin{minipage}{0.41\textwidth}
        \centering
        \includegraphics[width=\linewidth]{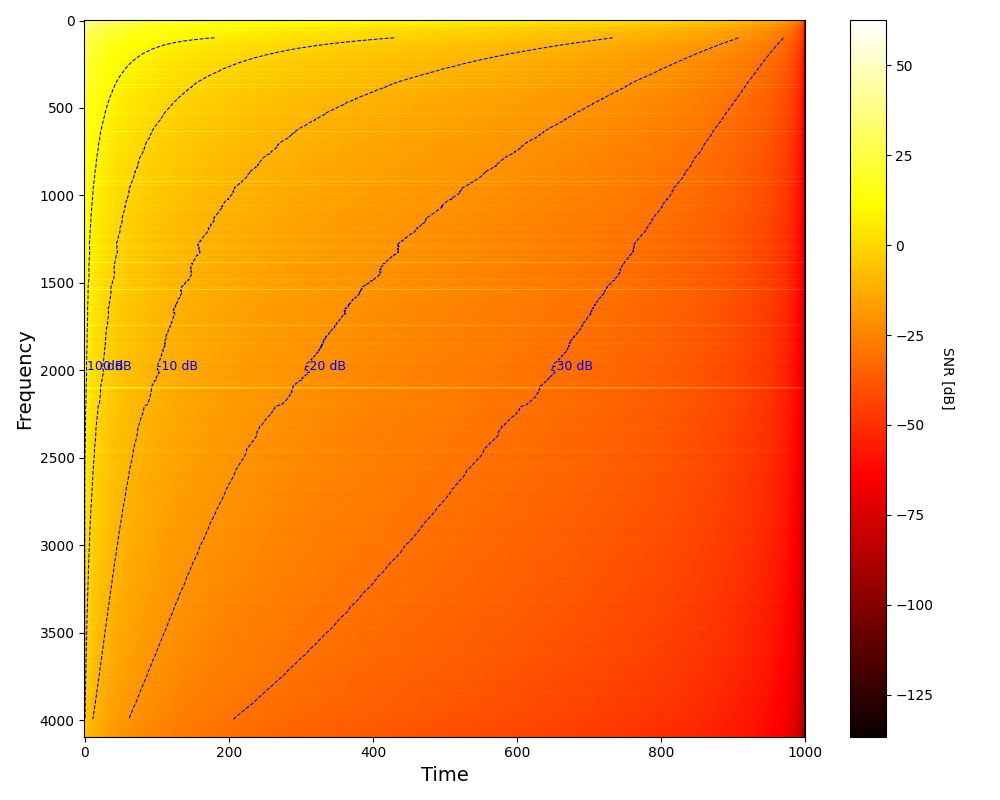}
    \end{minipage}
    \hfill
    \begin{minipage}{0.41\textwidth}
        \centering
        \includegraphics[width=\linewidth]{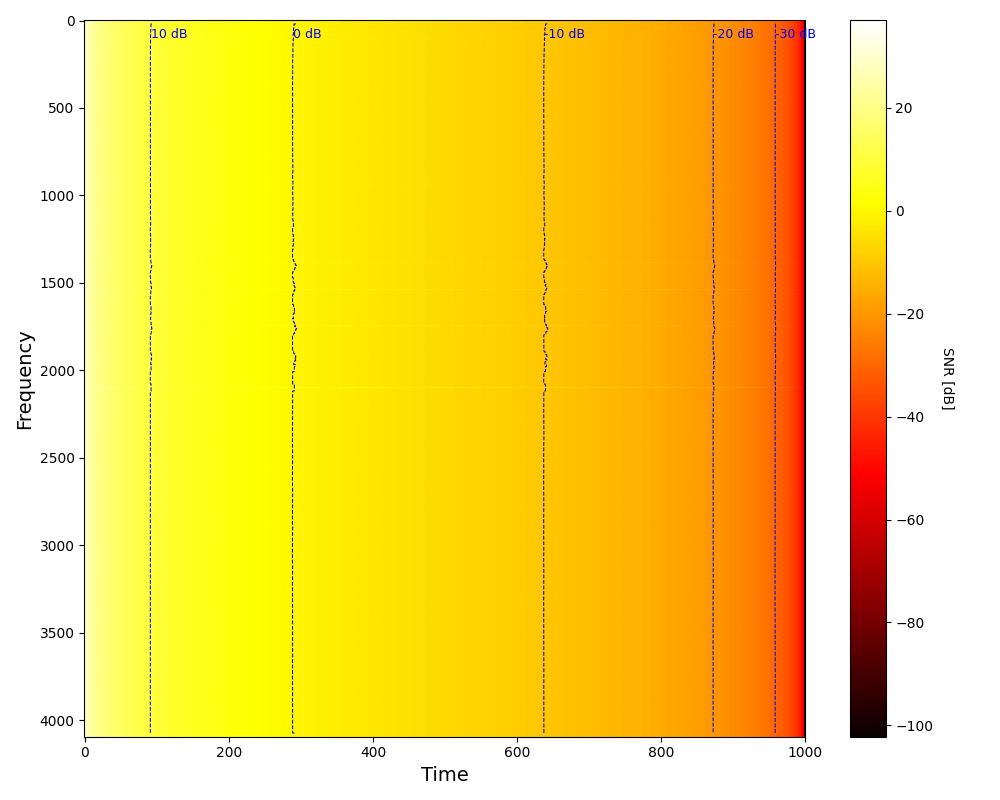}
    \end{minipage}
    \hfill
    \caption{SNR of the [left] DDPM and [right] EqualSNR forward processes on CelebA ($64 \times 64$). %
    }
    \label{fig:heatmap_celeba}
\end{figure*}

\begin{figure*}[h!]
    \centering
    \hfill
    \begin{minipage}{0.41\textwidth}
        \centering
        \includegraphics[width=\linewidth]{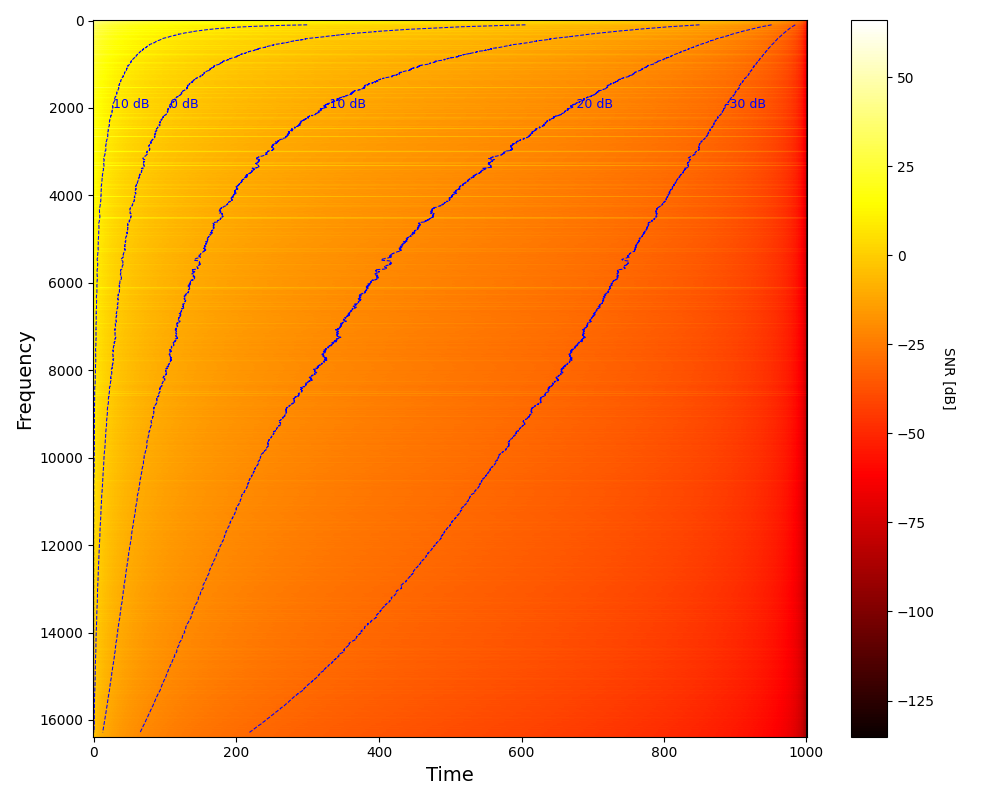}
        \label{fig:celeba_ddpm}
    \end{minipage}
    \hfill
    \begin{minipage}{0.42\textwidth}
        \centering
        \includegraphics[width=\linewidth]{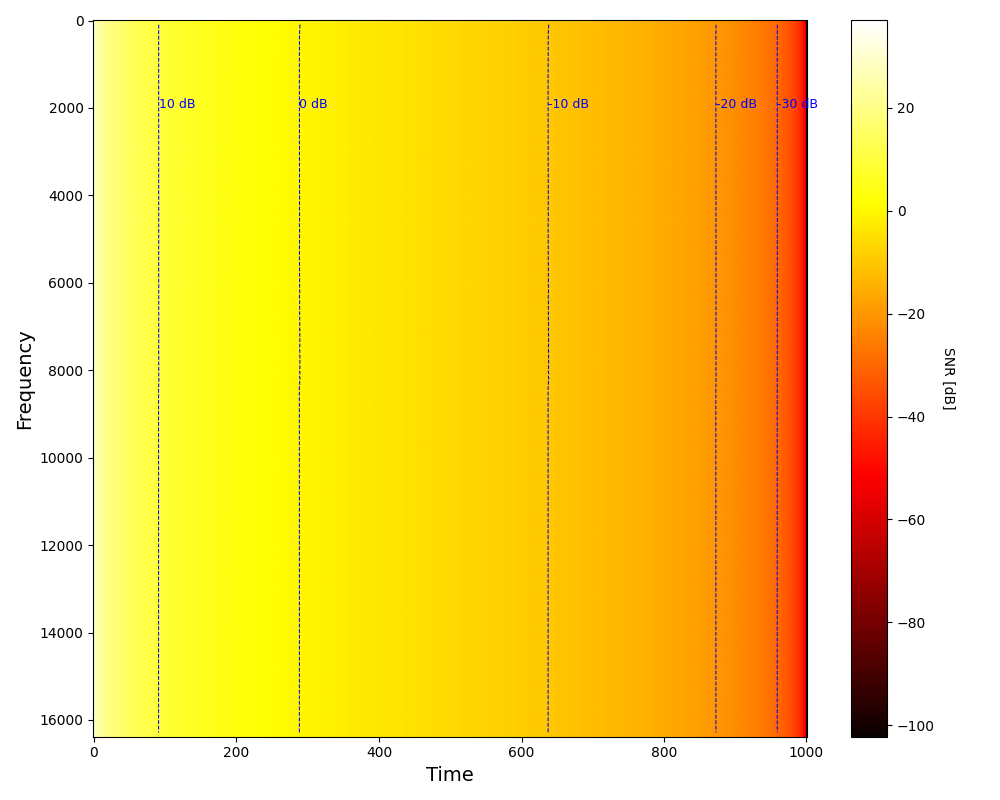}
        \label{fig:celeba_kddpm}
    \end{minipage}
    \hfill
    \caption{SNR of the [left] DDPM and [right] EqualSNR forward processes on LSUN Churches ($128 \times 128$). %
    }
    \label{fig:heatmap_lsun}
\end{figure*}

\begin{figure*}[h!]
    \centering
    \includegraphics[width=0.41\linewidth]{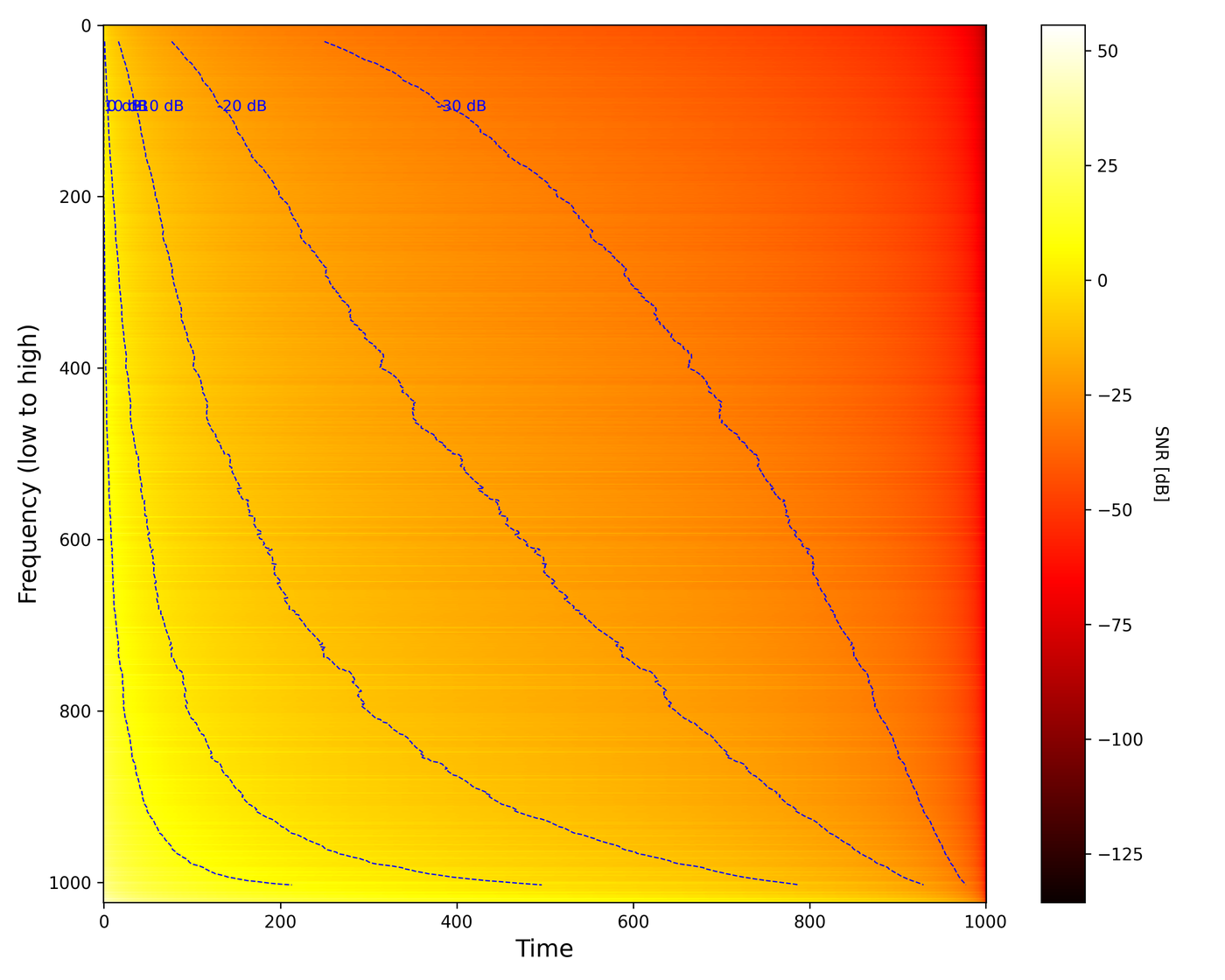}
    \caption{Flipped SNR forward process on CIFAR10 ($32 \times 32$).}
    \label{fig:flipped-snr}
\end{figure*}

\clearpage

Figures \ref{fig:fwd_ddpm_app} to \ref{fig:fwd_flippedsnr_app} present further examples illustrating the \textit{foward} process of DDPM, EqualSNR and FlippedSNR, augmenting \Cref{fig:fwd-bwd-high-low} in the main text.

\begin{figure}[ph!]
    \centering
    \includegraphics[width=.55\textwidth]{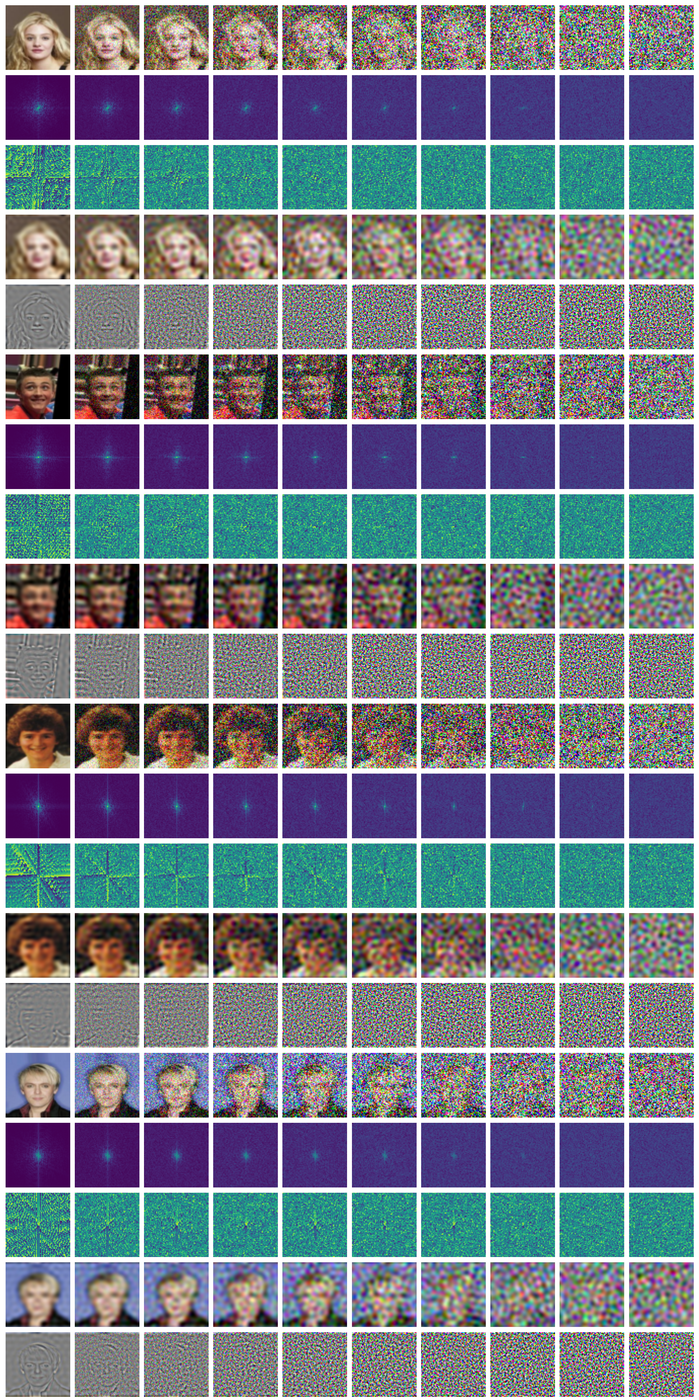}
    \caption{Forward process of DDPM, at uniformly-spaced timsteps between $t=0$ and $t=T$ (left to right).
    Each block of five rows visualises the noising process in pixel space (row 1) and in Fourier space (magnitude and phase, rows 2 and 3), and the low- and high-pass filtered images from row 1 (rows 4 and 5).}
    \label{fig:fwd_ddpm_app}
\end{figure}

\begin{figure}[ph]
    \centering
    \includegraphics[width=.6\textwidth]{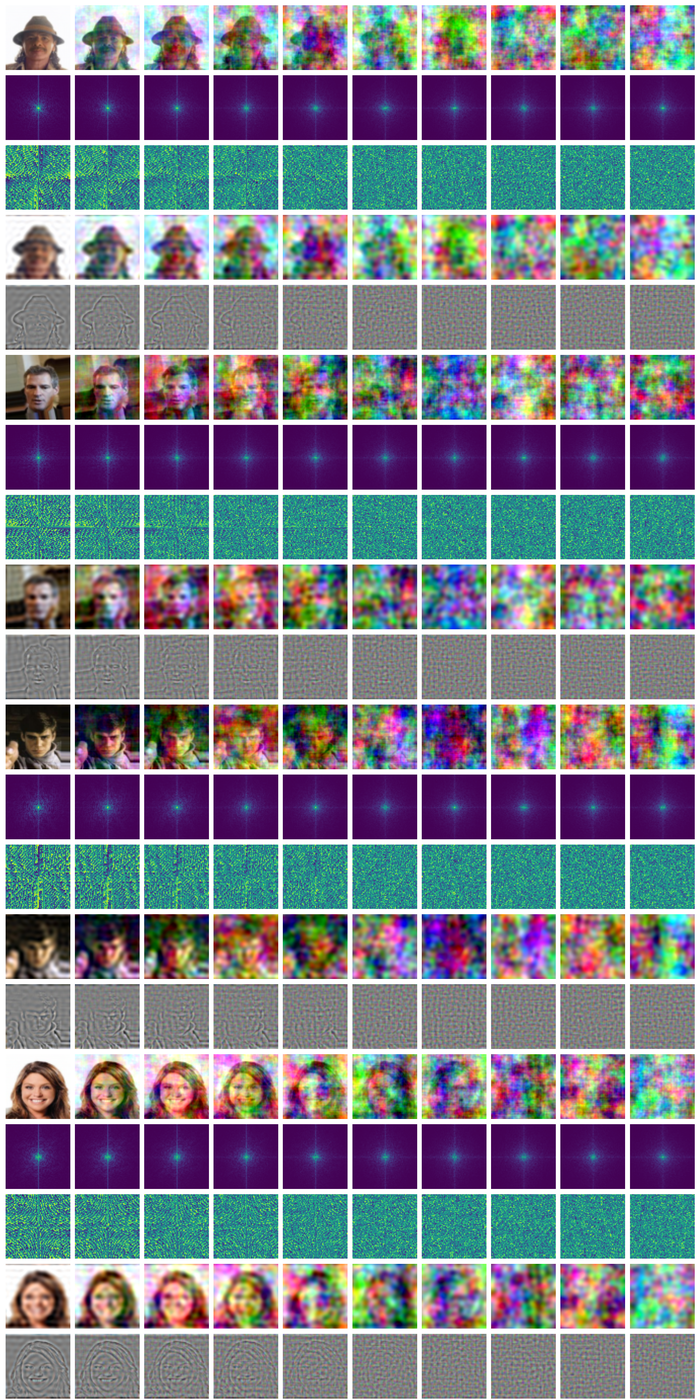}
    \caption{Forward process of EqualSNR, at uniformly-spaced timsteps between $t=0$ and $t=T$ (left to right).
    Each block of five rows visualises the noising process in pixel space (row 1) and in Fourier space (magnitude and phase, rows 2 and 3), and the low- and high-pass filtered images from row 1 (rows 4 and 5).}
    \label{fig:fwd_equalsnr_app}
\end{figure}

\begin{figure}[ph]
    \centering
    \includegraphics[width=.6\textwidth]{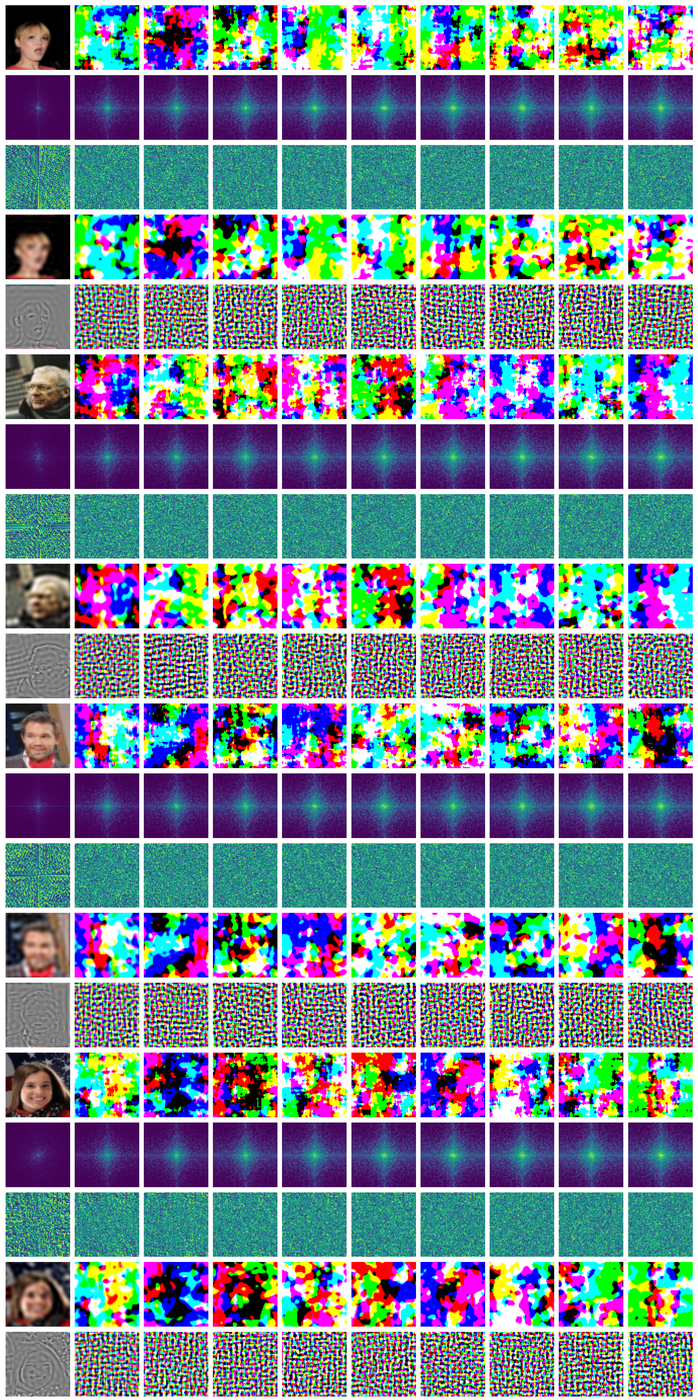}
    \caption{Forward process of FlippedSNR, at uniformly-spaced timsteps between $t=0$ and $t=T$ (left to right).
    Each block of five rows visualises the noising process in pixel space (row 1) and in Fourier space (magnitude and phase, rows 2 and 3), and the low- and high-pass filtered images from row 1 (rows 4 and 5).}
    \label{fig:fwd_flippedsnr_app}
\end{figure}

\clearpage

Figures \ref{fig:bwd_ddpm_app} to \ref{fig:bwd_flippedsnr_app} present further examples illustrating the \textit{reverse} process of DDPM, EqualSNR and FlippedSNR, augmenting \Cref{fig:fwd-bwd-high-low} in the main text.
Note that the FlippedSNR diffusion model could not be trained successfully, resulting in a deteriorated sample quality.

\begin{figure}[ph!]
    \centering
    \includegraphics[width=.55\textwidth]{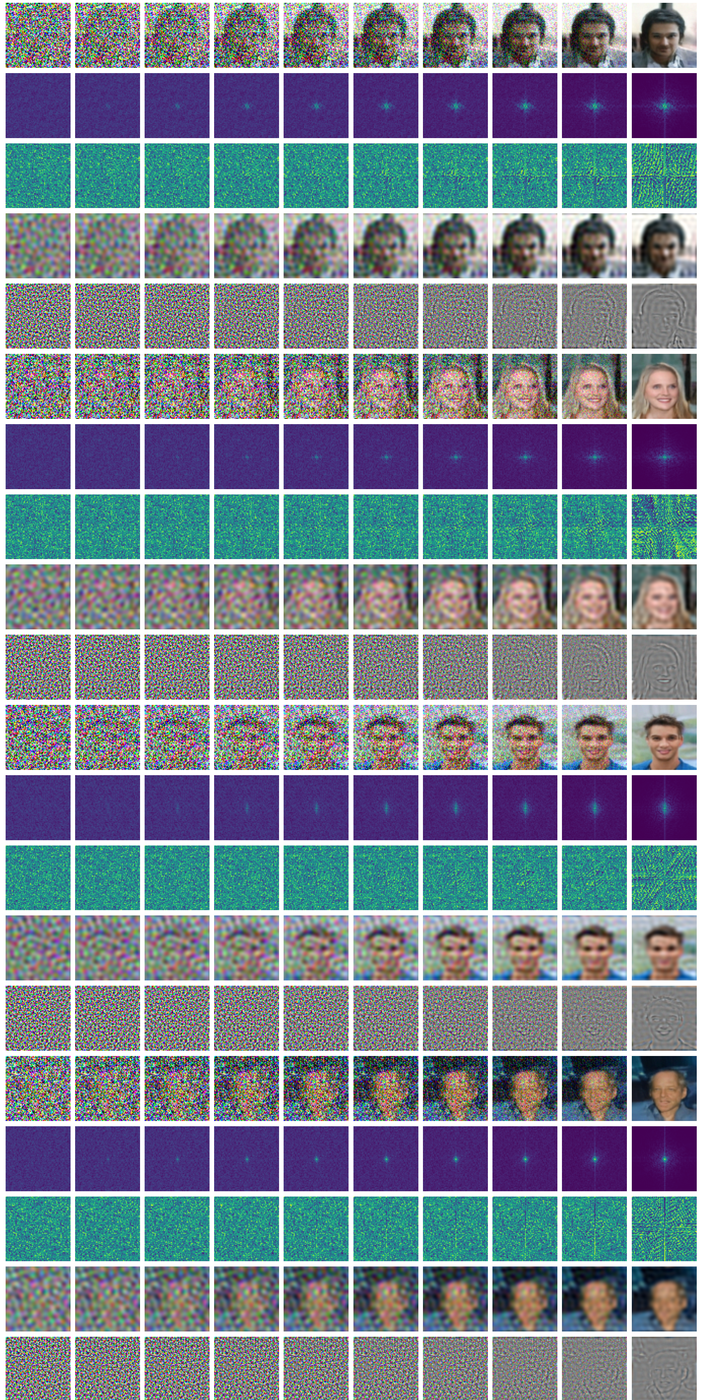}
    \caption{Reverse process of DDPM, at uniformly-spaced timsteps between $t=T$ and $t=0$ (left to right).
    Each block of five rows visualises the generative process in pixel space (row 1) and in Fourier space (magnitude and phase, rows 2 and 3), and the low- and high-pass filtered images from row 1 (rows 4 and 5). 
    }
    \label{fig:bwd_ddpm_app}
\end{figure}

\begin{figure}[ph]
    \centering
    \includegraphics[width=.6\textwidth]{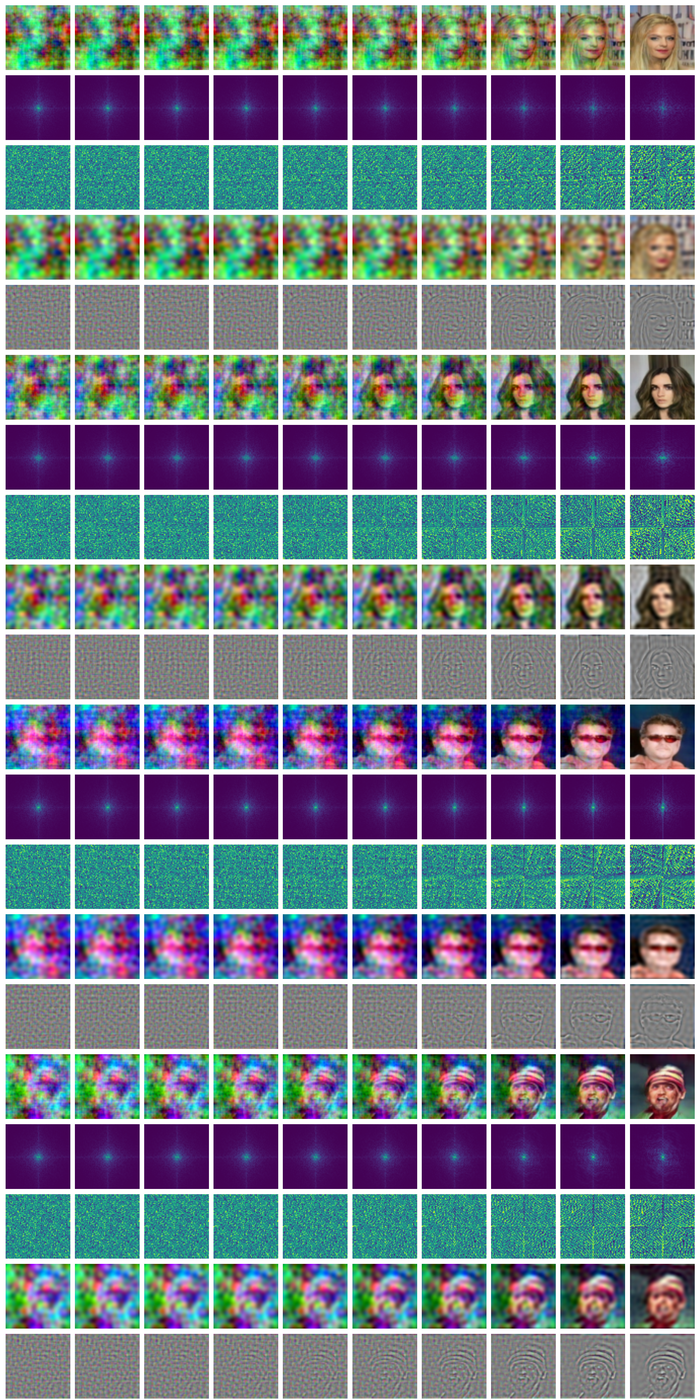}
    \caption{Reverse process of EqualSNR, at uniformly-spaced timsteps between $t=T$ and $t=0$ (left to right).
    Each block of five rows visualises the generative process in pixel space (row 1) and in Fourier space (magnitude and phase, rows 2 and 3), and the low- and high-pass filtered images from row 1 (rows 4 and 5). 
    }
    \label{fig:bwd_equalsnr_app}
\end{figure}

\begin{figure}[ph]
    \centering
    \includegraphics[width=.6\textwidth]{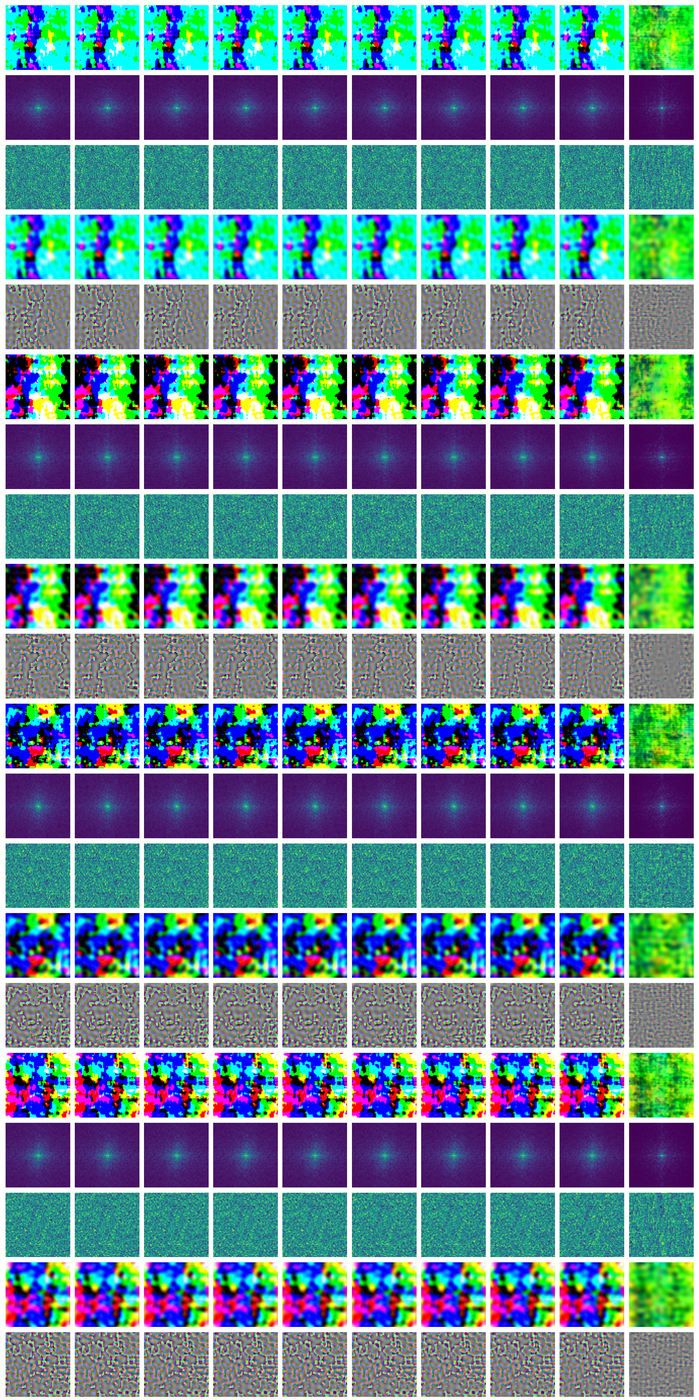}
    \caption{Reverse process of FlippedSNR, at uniformly-spaced timsteps between $t=T$ and $t=0$ (left to right).
    Each block of five rows visualises the generative process in pixel space (row 1) and in Fourier space (magnitude and phase, rows 2 and 3), and the low- and high-pass filtered images from row 1 (rows 4 and 5). 
    }
    \label{fig:bwd_flippedsnr_app}
\end{figure}

\clearpage

Figures \ref{fig:samples_cifar10} to \ref{fig:samples_lsun} present samples comparing a DDPM and EqualSNR diffusion model trained with the exact same hyperparameters and neural architecture, respectively.

\begin{figure}[ph]
    \centering
    \begin{minipage}{0.44\textwidth}
        \centering
        \includegraphics[width=\textwidth]{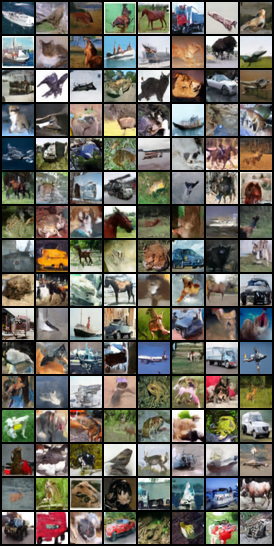}
        \caption*{DDPM}
    \end{minipage}
    \begin{minipage}{0.44\textwidth}
        \centering
        \includegraphics[width=\textwidth]{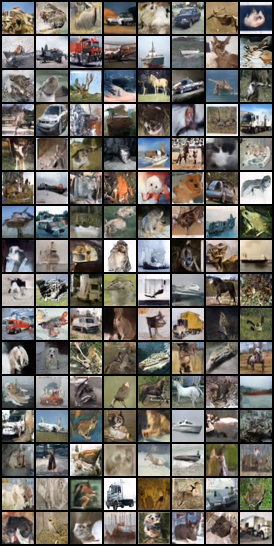}
        \caption*{EqualSNR}
    \end{minipage}
    \caption{Samples of a diffusion model trained on \texttt{CIFAR10} with a  [left] DDPM and [right] EqualSNR (calibrated) forward process, using $T=200$ steps at inference time.}
    \label{fig:samples_cifar10}
\end{figure}

\begin{figure}[ph]
    \centering
    \begin{minipage}{0.44\textwidth}
        \centering
        \includegraphics[width=\textwidth]{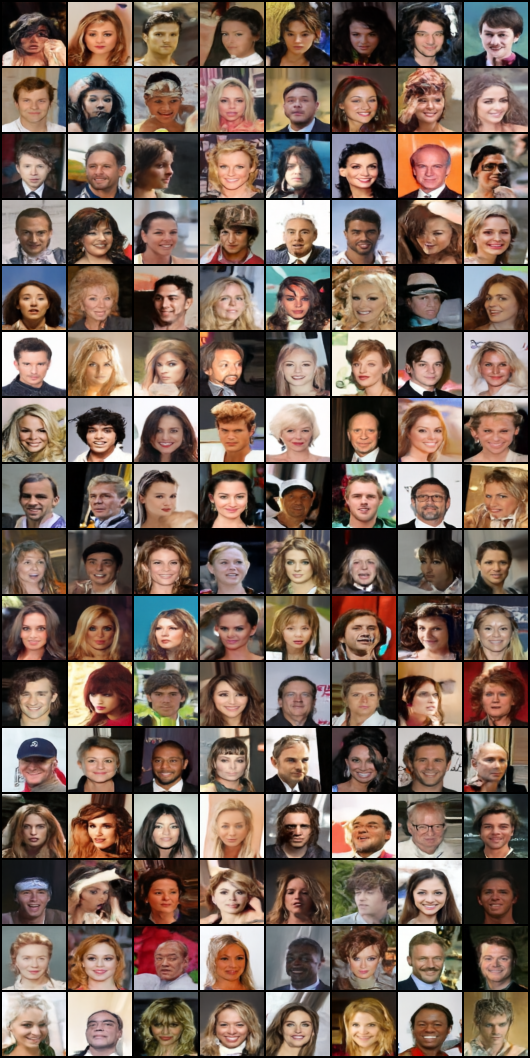}
        \caption*{DDPM}
    \end{minipage}
    \begin{minipage}{0.44\textwidth}
        \centering
        \includegraphics[width=\textwidth]{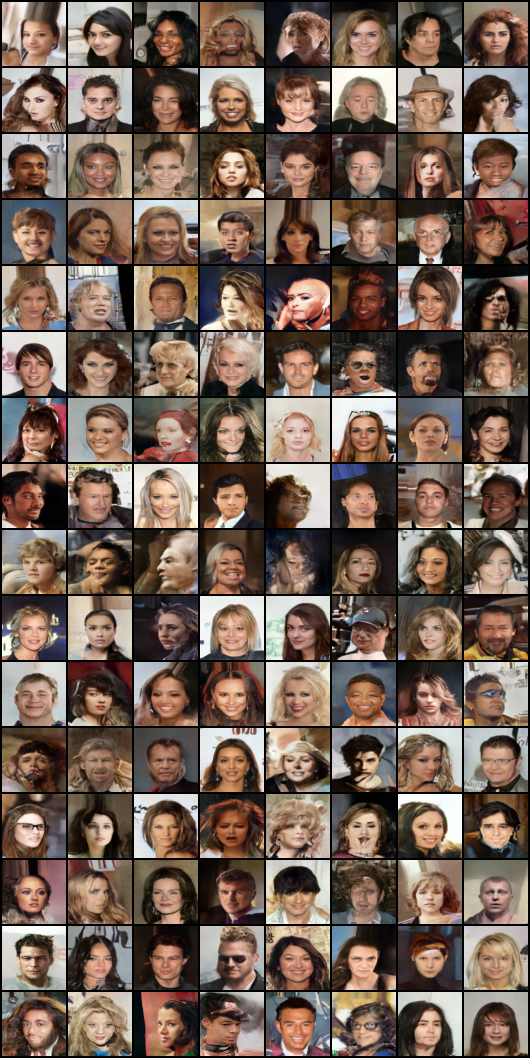}
        \caption*{EqualSNR}
    \end{minipage}
    \caption{Samples of a diffusion model trained on \texttt{CelebA} with a  [left] DDPM and [right] EqualSNR (calibrated) forward process, using $T=200$ steps at inference time.}
    \label{fig:samples_celeba}
\end{figure}

\begin{figure}[ph]
    \centering
    \begin{minipage}{0.44\textwidth}
        \centering
        \includegraphics[width=\textwidth]{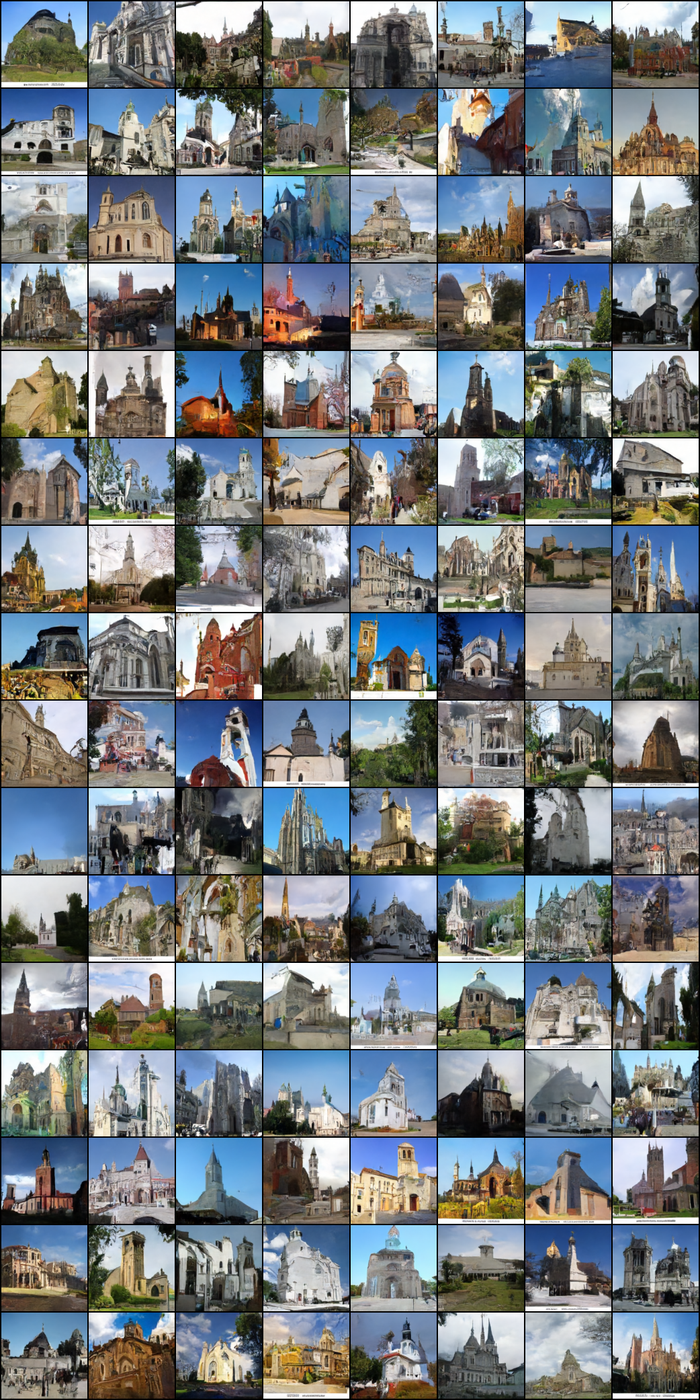}
        \caption*{DDPM}
    \end{minipage}
    \begin{minipage}{0.44\textwidth}
        \centering
        \includegraphics[width=\textwidth]{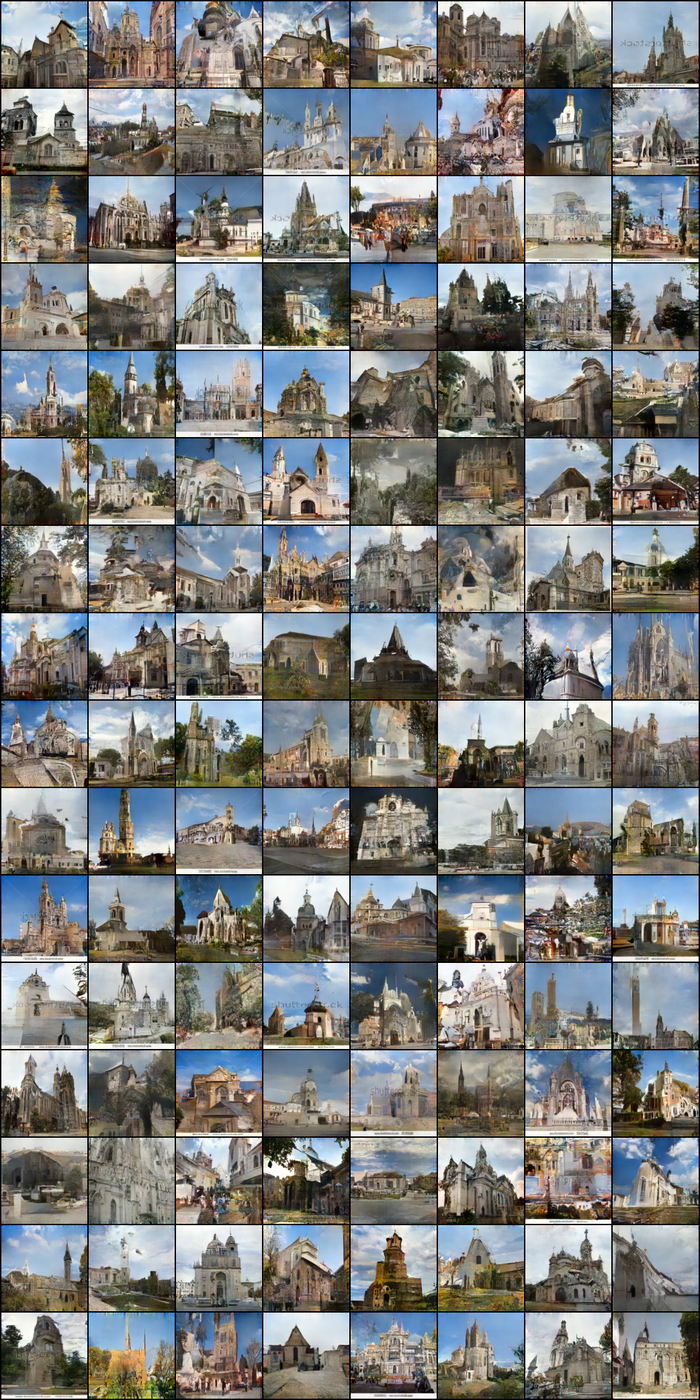}
        \caption*{EqualSNR}
    \end{minipage}
    \caption{Samples of a diffusion model trained on \texttt{LSUN Church} with a  [left] DDPM and [right] EqualSNR (calibrated) forward process, using $T=200$ steps at inference time.}
    \label{fig:samples_lsun}
\end{figure}

\clearpage

\end{document}